\documentclass[11pt]{article}

\usepackage[letterpaper, left=1.25in, right=1.25in, top=1.2in, bottom=1.2in]{geometry}

\usepackage[colorlinks=true, citecolor=blue, linkcolor=blue, urlcolor=blue]{hyperref}
\usepackage[dvipsnames]{xcolor}

\usepackage{natbib} 
\usepackage{caption}
\usepackage{subcaption}

\usepackage{amsmath, amssymb, mathtools, bbm}
\usepackage{amsfonts}
\usepackage{bm} 
\usepackage{physics}  
\usepackage{mathpazo}

\usepackage{amsthm}

\theoremstyle{plain}
\newtheorem{theorem}{Theorem}[section]
\newtheorem{lemma}[theorem]{Lemma}

\theoremstyle{definition}

\newtheorem{assumption}[theorem]{Assumption}

\theoremstyle{remark}

\usepackage{multirow} 
\usepackage{float}

\usepackage{tikz}
\usetikzlibrary{shapes, arrows.meta, positioning, calc}

\DeclareMathOperator*{\argmax}{arg\,max}

\let\Pr\undefined
\DeclareMathOperator{\Pr}{\mathbb{P}}
\newcommand{\indep}{\mathbin{\perp\!\!\!\!\!\:\perp}}

\usepackage{graphicx}
\usepackage{booktabs}
\usepackage{microtype}
\usepackage{setspace}
\usepackage{placeins}
\usepackage{comment}

\makeatletter
\newcommand{\assumplabel}[2]{%
   \protected@write \@auxout {}{\string\newlabel{#1}{{#2}{\thepage}{#2}{#1}{}}}%
   \hypertarget{#1}{#2}%
}
\makeatother

\usepackage{longtable}

\newcommand{\ind}{\mathbf{1}}

\DeclareMathOperator{\E}{\mathbb{E}}

\newcommand{\sur}{S}
\newcommand{\surlow}{s}
\newcommand{\dimsur}{d_{\mathcal{S}}}

\newcommand{\first}{F}

\newcommand{\geo}{G}
\newcommand{\geolow}{g}
\newcommand{\dimgeo}{d_{\mathcal{G}}}

\newcommand{\targ}{R}
\newcommand{\targlow}{r}
\newcommand{\dimtarg}{d_{\mathcal{R}}}

\newcommand{\coarsegrp}{C}
\newcommand{\coarsegrplow}{c}
\newcommand{\dimcoarsegrp}{d_\mathcal{C}}

\newcommand{\listmem}{L}

\newcommand{\emb}{E}
\newcommand{\emblow}{e}
\newcommand{\dimemb}{d_{\mathcal{E}}}

\newcommand{\bin}{B}
\newcommand{\binlow}{b}
\newcommand{\dimbin}{d_{\mathcal{B}}}

\newcommand{\outcome}{Y}
\newcommand{\outcomelow}{y}

\newcommand{\Prhat}
{\widehat{\Pr}}

\newcommand{\rhovec}{\bm{\rho}}
\newcommand{\rhovechat}{\widehat{\bm{\rho}}}
\newcommand{\pivec}{\bm{\phi}}
\newcommand{\pivechat}{\widehat{\bm{\phi}}}
\newcommand{\mmat}{\bm{M}}
\newcommand{\mmathat}{\widehat{\bm{M}}}

\newcommand{\Pimat}{\bm{\Pi}}
\newcommand{\betamat}{\bm{\beta}}

\newcommand{\Plisttarg}{\pi}
\newcommand{\Plisttargmat}{\bm{\Pi}}

\usepackage{algorithm}
\usepackage{algpseudocode}
\usepackage{float}

\newcommand{\blind}{0}
\newcommand{\tit}{$\ell$BISG: Probabilistic Race and Ethnicity Prediction Using Group-Specific Name Lists\thanks{This manuscript benefited from independent pre-journal peer review through rapidPeer. AI Disclosure: ChatGPT 5.6 was used to check for any logical errors or typos in proofs as well as editing of the manuscript. Claude Opus 4.6 through 5.5 were used to validate code and suggest changes to phrasing. Gemini models were used as described in the main text for generating surname lists and proposing group-specific prompts.}}
\title{{\bf\tit}}
\if0\blind

\author{Kyla Chasalow\thanks{PhD Candidate, Department of Statistics, Harvard University.} \and Noah Dasanaike\thanks{PhD Candidate, Department of Government, Harvard University.} \and Kosuke Imai\thanks{Edith and Benjamin Geisinger Professor, Department of Government and Department of Statistics, Harvard University.  1737 Cambridge Street, Institute for Quantitative Social Science, Cambridge MA 02138.  Email: \href{mailto:imai@harvard.edu}{imai@harvard.edu} URL: \href{https://imai.fas.harvard.edu}{https://imai.fas.harvard.edu}}}
\date{\today}

\begin{document}

\maketitle

\begin{abstract}
Statistically valid estimation of racial and ethnic disparities often requires inferring group membership probabilities from names and geographic locations. The standard method, Bayesian Improved Surname Geocoding (BISG), requires group frequencies for each name. Such data are available from the U.S. Census Bureau for common names and limited racial categories, but are unavailable for many ethnic groups and rarely exist outside the United States. We propose list-powered BISG ($\ell$BISG), in which we replace these frequencies with group-specific name lists and combine the lists with geographic prevalence data to obtain calibrated group membership probabilities. These lists can be compiled from different sources or generated by large language models (LLMs), despite potentially unknown biases. Under certain conditions on the name lists, $\ell$BISG can also be applied without geographic prevalence data. Empirical validation using U.S. voter files with self-reported race, the full-count 1900 U.S. Census, and the Lebanese voter registry demonstrates that LLM-generated lists yield accurate, well-calibrated probabilities and precise disparity estimates comparable to existing methods that require name–race frequency data. Thus, $\ell$BISG extends probabilistic race and ethnicity prediction to settings where it was previously not thought possible. 

\end{abstract}
\newpage

\section{Introduction}

Measuring racial and ethnic disparities across social, economic, and political domains requires information about individuals' race in administrative and other records \citep[e.g.,][]{adjayegbewonyo2014bisg, fraga2018turnout, greenwald2024regulatory}. However, race is often unavailable because it is withheld by law or omitted due to privacy concerns and the sensitivity of collecting such information \citep{fiscella2006geocoding}. As a result, researchers must infer individuals' racial or ethnic group membership from the information that is available, most commonly their name and place of residence.

Bayesian Improved Surname Geocoding (BISG) combines surname-based racial information, derived from the U.S. Census Bureau’s surname list, with the racial composition of an individual’s place of residence to estimate individual race probabilities. The method assumes that surname and geolocation are conditionally independent given unobserved race \citep{elliott2009using, imaikhanna2016, Imai_etal_2022}. BISG has been extended to incorporate first names, resulting in the Bayesian Improved First Name Surname Geocoding (BIFSG) method \citep{Voicu2018}.

BISG has become the standard approach for racial demographic inference across many domains. Its applications range from fair lending enforcement \citep{cfpb2014proxy, greenwald2024regulatory} to studies of health disparities \citep{adjayegbewonyo2014bisg} and racial turnout gaps in voting \citep{fraga2018turnout}, and it is supported by widely used software implementations \citep{wru}. BISG returns a probability for each individual belonging to each group, which is essential for downstream statistical analyses. Probabilistic outputs enable unbiased estimation of group disparities in settings where deterministic classifications can introduce bias \citep{McCartan02102025}.

However, BISG requires data that exist in few settings: a surname-by-race frequency table (and BIFSG additionally requires an analogous table for first names) and the geographic prevalence of each racial group. For these reasons, BISG is most commonly applied in the United States. The U.S. Census Bureau publishes the former tables only for ``frequently occurring" names -- names that appear at least 100 times in the Census returns -- and provides these tables for a limited set of racial categories \citep{census2016surnames}. Detailed geographic distributions of groups are similarly confined to these few racial categories. Consequently, individuals with uncommon names receive only an uninformative geography-based prior, and the group memberships for which BISG can be plausibly estimated are few. The required data are unavailable for more detailed ethnic groups and for most countries outside the United States.

We propose list-powered BISG ($\ell$BISG), which estimates calibrated race probabilities without requiring a name-by-race frequency table (Section~\ref{sec-keymethod}).  We consider two settings: a setting in which geographic prevalence is available for all groups of interest and a limited information setting in which geographic prevalence is either unavailable or available only for coarser groups (Section~\ref{sec-recovergeography}).  The proposed $\ell$BISG relies on relatively short group-specific name lists, which can be compiled from expert knowledge or generated synthetically using large language models (LLMs) (Section~\ref{sec-lists}). 

$\ell$BISG is based on two key ideas: (1) lists contain information about the name patterns associated with group membership, and (2) group geographic prevalence is associated with name and geographic prevalence of list membership. To formalize (1), we embed names in a vector space, in which morphologically and linguistically similar names are close to one another, and we fit prediction models of whether a name embedding belongs to a given list \citep{ye2017nationality, wang2024e5, Dasanaike2026zeroshot, DasanaikeImai2026}. We then use these predictions to form list-guided name clusters. 

The core identification strategy of $\ell$BISG is based on (2): we apply proximal inference \citep{KurokiPearl2014} to recover group probabilities within each name cluster using the geographic distributions of groups and name clusters. When geographic prevalence information is available, $\ell$BISG can use embeddings without name lists by relying more heavily on the spatial clustering of the embeddings to infer group membership.  However, name lists play an essential role in $\ell$BISG when geographic prevalence information is absent. Lastly, $\ell$BISG also provides diagnostics about the quality of name lists even when no ground truth evaluation set is available.

Like standard BISG \citep{Voicu2018}, $\ell$BISG naturally extends to incorporating group-specific first name lists, yielding what we call $\ell$BIFSG. Moreover, because $\ell$BISG produces calibrated group probabilities, it can also substitute for BISG in downstream analyses. For instance, $\ell$BISG can be combined with BIRDiE \citep{McCartan02102025} to obtain valid estimates of group disparities in an outcome. BIRDiE assumes that the outcome is independent of name conditional on race and geography, rather than the often less plausible assumption that it is independent of race conditional on name and geography.

We evaluate the empirical performance of $\ell$BISG using three datasets with individual-level ground truth (Section~\ref{sec-validation}): (1) the Florida and North Carolina voter files, which include self-reported race for registered voters; (2) the complete-count 1900 U.S. Census, which enumerates Chinese and Japanese as separate ethnic groups \citep{ipumsfullcount}; and (3) the 2018 Lebanese voter file, scraped from government records, which contains individuals' self-reported religious sect. Because these settings include ethnic groups and religious sects, we use the terms race, ethnicity, and group interchangeably throughout the paper. Across these settings, $\ell$BISG achieves prediction accuracy and probability calibration comparable to those of BISG, and otherwise performs well even where the requisite data for BISG are unattainable. We also show that this strong performance carries over to downstream analyses, with $\ell$BISG yielding group disparity estimates nearly identical to those obtained using BISG under the BIRDiE framework \citep{McCartan02102025}. 

By eliminating the need for name–race frequency tables and requiring only easily constructed name lists, $\ell$BISG substantially broadens the scope of race and ethnicity imputation.  $\ell$BISG is implemented through an open-source software package {\tt wru} (\url{ https://CRAN.R-project.org/package=wru}).

\subsection*{Related literature}

A growing literature seeks to overcome the limitations of the Census surname list that BISG requires by training machine learning models to infer race or ethnicity directly from names. Character-level and transformer-based models predict race from the sequence of letters in a name \citep{sood2018predicting, parasurama2021racebert}; interpretable models leverage patterns in name popularity \citep{jain2022importance}; and embedding-based approaches learn representations of names from social and communication data without using race labels \citep{ye2017nationality, ye2019secret}. More recently, LLMs have been used to infer ethnicity, religious sect, or caste from names without task-specific training, both in the United States and elsewhere \citep{Dasanaike2026zeroshot}.

These approaches typically assign each individual to a single group and are often proposed as alternatives to BISG. However, classification errors make na\"{\i}vely treating these predictions as ground truth problematic, leading to biased estimates in downstream disparity analyses \citep[][see also Section~\ref{sec-valueofprediction}]{chen2019fairness,xin2026proxy}. Rather than replacing BISG, $\ell$BISG leverages the outputs of these and other classifiers to recover calibrated probabilistic predictions, which can then be used for statistically valid downstream inference \citep{McCartan02102025}.

A number of existing methods extend BISG itself. As noted above, \citet{Voicu2018} incorporates first-name information through an additional Bayes factor (BIFSG), while \citet{Imai_etal_2022} develops a fully Bayesian version (fBISG) that accounts for potential measurement error in the Census surname data. \citet{rosenman2023names} augments the Census surname list with name dictionaries derived from Southern voter files, and \citet{DasanaikeImai2026} replaces the uninformative prior for uncommon surnames with embedding-based predictions (eBISG). Lastly, \citet{greengard2025calibrated} improves the calibration of BISG probabilities to relax BISG's usual surname-geography conditional independence assumption. Nevertheless, all of these methods rely on a name-race frequency table. In contrast, $\ell$BISG eliminates this requirement while remaining compatible with the same downstream analyses and naturally extending to incorporate first-name information (Sections~\ref{sec-firstnames}~and~\ref{sec-val-voterfiles}).

Curated name lists have long been used to identify and sample members of specific populations. Distinctive surname lists have been developed for American Jews \citep{himmelfarb1983surnames}, Asian Americans \citep{lauderdale2000asian}, and, with formal validation, South Asian and Chinese populations in administrative health data \citep{shah2010surname}. These approaches often use curated lists as deterministic filters, classifying individuals solely on the basis of whether their names appear on the list (see \cite{chasalow2026sampling} for an exception). 

In $\ell$BISG, we instead treat list membership as a noisy proxy for group membership. A given list omits the names of many group members and may include names that are shared with members of other groups. We use name embeddings to predict list membership for names that appear on no list, and thereby obtain group membership probabilities for unlisted names, whereas a deterministic filter would classify these as a non-member. Moreover, the further steps of $\ell$BISG are necessary because, as we confirm empirically, list membership probabilities are not generally well-calibrated group probabilities.

\section{Methods}

We now introduce our $\ell$BISG methodology. Suppose we have a random sample of $n$ individuals from a population. For each individual, we observe their geographic location $\geo_i  \in \mathcal{G}$, name $\sur_i \in \mathcal{S}$ (a first name, surname, or both), and outcome $\outcome_i \in \mathcal{Y}$, but not their membership $\targ_i \in \mathcal{R}$ in a racial, ethnic, or other group of interest. Let $\dimgeo:=|\mathcal{G}|$, $\dimsur:=|\mathcal{S}|$, and $\dimtarg:=|\mathcal{R}|$ denote the numbers of possible values of these variables in the population, and assume each is finite. Our goal is to estimate probabilities for the missing $\targ_i$ given each individual's name $S_i$ and geographic location $G_i$, and then use these probabilities to estimate quantities such as differences in the average outcome $\outcome$ across groups.

We first explain the value of probabilistic prediction over classification in Section~\ref{sec-valueofprediction} to motivate $\ell$BISG. In Section~\ref{sec-keymethod}, we assume that information about group geographic prevalence is available, while in Section \ref{sec-recovergeography}, we extend the method to the setting in which such information is unavailable or available only for coarser groups. In both cases, we assume only that for each group $\targlow$, we have a list of names that are distinctive, common, or typical of that group. We let $\listmem_\targlow = \listmem_\targlow(\sur)$ be a binary indicator for group $\targlow$ which is $1$ if name $\sur$ is in that group's list and $0$ otherwise. While Section~\ref{sec-keymethod} makes no assumptions about list quality (though better lists will lead to more precise final probabilities), Section~\ref{sec-recovergeography} imposes more list quality requirements. In Section~\ref{sec-lists}, we further discuss desirable properties of name lists as well as how one may obtain them and how to compute list quality diagnostics.

\subsection{The value of probabilistic prediction in downstream analysis}\label{sec-valueofprediction}

We begin by briefly explaining why probabilistic prediction of an individual’s race is essential for valid statistical estimation of racial disparities. A common approach to the missing group membership problem is to assign each individual to the group with the highest predicted probability based on their name and geographic location. However, this discretization ignores classification error and can bias downstream analyses. To illustrate, consider the case in which both $\targ$ and $\outcome$ are binary. Suppose each individual is assigned to the racial group with the highest probability, i.e.,
$$
\targ^* = \argmax_{\targlow \in \{0,1\}} \Pr(\targ=\targlow \mid \sur,\geo).
$$
We refer to this as the {\it oracle classifier} because it assumes the true conditional probabilities are known, thereby eliminating estimation error and model misspecification. 

Now, suppose the outcome is independent of name conditional on race and geography, i.e., $\outcome \indep \sur \mid \targ, \geo$, which is a key identification assumption for downstream analyses of racial disparities \citep[see][]{McCartan02102025}.
Appendix~\ref{app-threshbias} shows that if we estimate the racial disparity in geography $\geo=\geolow$ using the oracle classification $\targ^*$, the resulting bias is given by the following product of racial disparity and misclassification rate,
{\small \begin{align*} 
&[\Pr(\outcome=\outcomelow\mid \targ^*=1,\geolow)-\Pr(\outcome=\outcomelow\mid \targ^*=0,\geolow)]-[\Pr(\outcome=\outcomelow\mid \targ=1,\geolow)-\Pr(\outcome=\outcomelow\mid \targ=0,\geolow)]\\
=&-[\Pr(\outcome=\outcomelow\mid \targ=1,\geolow)-\Pr(\outcome=\outcomelow\mid \targ=0,\geolow)]
\times [\Pr(\targ=1\mid \sur\in\mathcal{C}_{0g},\geolow)+\Pr(\targ=0\mid \sur\in\mathcal{C}_{1g},\geolow)],
\end{align*}}where $\mathcal{C}_{rg}$ denotes the set of names classified as $\targ^*=r$ in geography $\geolow$. Related results appear in \cite{chen2019fairness} and \cite{xin2026proxy}.

The bias in each geography is therefore proportional to the total misclassification rate and has the opposite sign of the true racial disparity. Consequently, hard classification systematically attenuates the estimated disparity toward zero: if the true racial disparity is positive, it is underestimated, whereas if it is negative, it is overestimated. Moreover, the bias is zero if and only if the classification is perfect or the true racial disparity is zero.

Empirically, we find that the misclassification rate can be sizable.  For a brief illustration, we use the Florida voter file from Section~\ref{sec-val-voterfiles}, which contains self-reported race. We compute standard BISG probabilities $\Pr(\targ\mid\sur,\geo)$ for every voter using the Census surname table and the block-level race shares. We estimate the racial composition of the electorate by assigning each registrant to their single most likely race and summing across individuals. 

Compared to the self-reported population composition, the hard classification over-predicts the White share by $5.8$ percentage points (ppts) and under-predicts the Black share by $2.2$ ppts. Using the average estimated probability instead of the hard classifications reduces these errors to $-1.4$ and $-0.8$ ppts, respectively, and also causes the mean absolute error of the estimated racial composition, averaged over all five categories, to drop from $2.3$ to $1.0$ ppts.

\subsection{The List-powered BISG Method}\label{sec-keymethod}

In this subsection, we assume that information about the geographic distribution of the group variable $\targ$ is available, meaning that $\Pr(\targ=\targlow\mid \geo=\geolow)$ is known from census data or has been estimated from surveys for each geography $\geolow$. For now, we make no assumptions about list quality. The lists may even have some names in common, though more overlap will make groups less distinguishable. In some cases, one group may represent a residual or ``other'' group not of interest for analysis.  Although it helps to have a list for this group, it is not strictly necessary. That is, we may have $\dimtarg-1$ lists.

\subsubsection{Using name embeddings for proxy prediction}\label{sec-nameembed}

Name lists contain information about the underlying patterns and structures associated with names from a particular group. For example, lists could suggest that surnames ending in ``dez'' are often Hispanic, whereas names starting with ``zh'' are more likely to be Chinese. Regularities such as these imply that names not appearing on a group-specific list, but nonetheless sharing similar structural or linguistic characteristics, may also be associated with that group.

Dense text embeddings provide a natural way to formalize these intuitions both mathematically and at scale. Representing names as vectors in a high-dimensional space, they place structurally, linguistically, or otherwise similar names closer together based on patterns learned from the corpora used to train the embedding model \citep{ye2017nationality, ye2019secret, wang2024e5}. As a result, the embeddings derived from a group-specific name list encode information about the similarity among names both on and off the list, allowing us to identify names that are likely to belong to the same group even if they do not explicitly appear in the original list.

Suppose that for each name $\sur$, we obtain a $\dimemb$-dimensional text embedding vector $\emb=\emb(\sur)$ that captures its linguistic and structural features. \citet{DasanaikeImai2026} demonstrate the value of an embedding-based BISG approach (eBISG), which uses Census name tables to learn $\Pr(\targ=\targlow\mid \emb,\geo)$ and generalizes the resulting model to rare names whose frequencies are unavailable.

In our setting, however, no such name–group frequency data are available. Instead, we formulate a new proxy prediction task. For each group-specific name list, represented by the indicator $\listmem_{\targlow}$, we model the probability that a name with embedding $\emb$ belongs to the list, using standard supervised machine learning methods. We call the resulting probability $\Pr(\listmem_\targlow=1\mid \emb,\geo)$ the \emph{list score}. Though $\listmem_\targ$ is a binary function of surname, the list score is generally greater than $0$ and less than $1$ because lists are neither exhaustive nor perfectly separated in embedding space. 

Unfortunately, this proxy prediction task does not directly recover the target probabilities $\Pr(\targ=\targlow\mid \emb,\geo)$. Although $\Pr(\listmem_{\targlow}=1\mid \emb,\geo)$ is generally positively associated with $\Pr(\targ=\targlow\mid \emb,\geo)$, the two quantities are not equal. The list score depends on which names are included on the list, and a given list may be skewed by including names that the group shares with our groups or by excluding names that are common to the group. The list scores also do not take into account base rates. 

For example, the LLM-generated list of Asian surnames in Section~\ref{sec-val-voterfiles} includes ``Lee," and names with embeddings close to that of ``Lee" therefore receive a high Asian list score. However, most Florida voters with the surname Lee are white or Black. The list prediction task does not account for the fact that the Asian group is rare. In Appendix~\ref{app-rawf}, we show that the Florida list scores are not well calibrated, which means they are likely to yield biased downstream estimates. This failure of the simple list prediction approach suggests that achieving good calibration will require incorporating additional information.   

\subsubsection{Identifying the target probabilities from proxy prediction}\label{sec-meth-getRprob}

We now introduce $\ell$BISG, which uses list scores but bases identification on geographic prevalence patterns that reflect the relationship between names and groups.
The standard identifying assumption underlying BISG is that, conditional on group membership, name and geography are independent \citep{elliott2009using}.  
\begin{assumption}[BISG] \label{a-bisg}
$$
\sur \indep \geo \mid \targ.
$$
\end{assumption}

Because the embedding is a deterministic function of the name, i.e., $\emb=\emb(\sur)$, this assumption implies the identifying assumption used by eBISG \citep{DasanaikeImai2026}, which requires embedding and geography to be conditionally independent given group membership.
\begin{assumption}[eBISG] \label{a-ebisg}
$$
\emb \indep \geo \mid \targ.
$$
\end{assumption}

Under Assumption~\ref{a-ebisg}, Bayes Rule implies
\begin{align}
        \Pr(\targ = \targlow \mid \emb, \geo)%
        & \propto \label{eq-keybayesrule}\Pr(\emb\mid \targ=\targlow)\Pr(\targ=\targlow\mid \geo).
\end{align}
Although we assume that $\Pr(\targ=r\mid\geo)$ is identified from an external source such as Census data, no analogous data are available to estimate the embedding distribution $\Pr(\emb\mid\targ=\targlow)$.

The proposed $\ell$BISG method is based on the observation that under Assumption~\ref{a-ebisg}, the following equality holds for each $\targlow$ and $g$,
\begin{equation}\label{eq-keysystem1}
    \Pr(\emb = \emblow \mid \geo=\geolow) = \sum_{\targlow=1}^{\dimtarg} \Pr(\emb=\emblow \mid \targ=\targlow )\Pr(\targ = \targlow \mid \geo=\geolow). 
\end{equation}
Thus, for each value of $e$, we have a system of $\dimgeo$ linear equations with $\dimtarg$ unknowns.  Under the following rank condition, we can uniquely identify $\Pr(\emb=\emblow \mid \targ=\targlow)$ \citep{KurokiPearl2014}.
    
\begin{assumption}[Geographic prevalence rank condition] \label{a-georank}
    Let $\mmat$ be the $\dimgeo \times \dimtarg$ matrix with $\mmat_{g,r}=\Pr(\targ=\targlow\mid \geo=\geolow)$. Assume $\mmat$ is full column rank, for which a necessary condition is $\dimgeo \geq \dimtarg$.
\end{assumption} 
    
Concretely, for each $\emblow$, let $\rhovec_\emblow$ be the vector of $\Pr(\emb=\emblow\mid G)$ and let $\pivec_\emblow$ be the vector of $\Pr(\emb=\emblow\mid \targ)$. Then, the target probabilities can be recovered by regressing $\rhovec_\emblow$ on $\mmat$:
\begin{equation}\label{eq-phireg}
    \pivec_\emblow = (\mmat^\top \mmat)^{-1}\mmat^\top\rhovec_{\emblow}.
\end{equation}
Thus, the identification underlying the proposed $\ell$BISG method leverages the association between geographic group patterns and name prevalence patterns. 

In practice, given $n$ observations of $(\emb_i,\geo_i)$, we can also view identification equality of Equation~\eqref{eq-phireg} as an OLS regression of $\Pr(\emb_i\mid \geo_i = \geolow)$ on the length $\dimtarg$ vector of predictors $\Pr(R_i \mid \geo_i = \geolow)$. Let $\bm{W}$ be the $\dimgeo\times \dimgeo$ diagonal matrix with the number of observations from each geography $\geolow$, i.e., $n_{\geolow}$, on the diagonal. Then, this gives the weighted regression:
\begin{equation}
    \pivec_\emblow = (\mmat^\top\bm{W} \mmat)^{-1}\mmat^\top\bm{W}\rhovec_{\emblow}.
\end{equation}
The weighted regression is preferable because it accounts for differences in estimation uncertainty of $\rhovec_{\emblow}$ due to different numbers of observations across geographies.

\subsubsection{Leveraging name lists for dimension reduction}\label{sec-dimreduction}

While the previous result establishes identification, directly estimating the conditional distribution $\Pr(\emb=\emblow\mid\geo=\geolow)$ is challenging because the embedding vector $\emb$ is high-dimensional. 
Our practical solution is to replace the high-dimensional embedding $\emb$ with an informative lower-dimensional representation $\bin=\bin(\emb)$ that we now assume satisfies the following conditional independence assumption:
\begin{assumption}[$\ell$BISG]\label{a-lbisg}
$$
\bin \indep \geo \mid \targ.
$$
\end{assumption}
Intuitively, $\bin$ is obtained by clustering or otherwise reducing the dimension of $\emb$ while retaining information relevant for predicting the target group. Because $\bin$ is a deterministic many-to-one function of $\emb$, the eBISG assumption (Assumption~\ref{a-ebisg}) implies Assumption~\ref{a-lbisg}, but the reverse is not true. Assumption~\ref{a-lbisg} is therefore weaker than the eBISG assumption and the classic BISG assumption (Assumption~\ref{a-bisg}). 

Suppose, for example, that Hispanic residents of one county more often have surnames of Mexican origin and that Hispanic residents of another county more often have surnames of Cuban origin. Among Hispanic individuals, the surnames and their embeddings vary across geography, violating Assumptions~\ref{a-bisg}~and~\ref{a-ebisg}. However, if these Mexican and Cuban surnames belong to the same bin $\binlow$, then Assumption~\ref{a-lbisg} holds so long as among Hispanic individuals in those two counties, the share in that bin is the same in each county.

Another way to think about this distinction between these assumptions is that the original BISG assumption may be too high-resolution, especially for distinctive names and when first names are included. For example, among white individuals, the probability that ``Kyla Chasalow'' lives in Florida and the probability that she lives in Massachusetts are different -- technically a violation of Assumption~\ref{a-bisg}. However, Assumption~\ref{a-lbisg} only requires that among white people, being in some cluster of names of the same general type as ``Kyla Chasalow'' is not associated with geographic location.

Under Assumption~\ref{a-lbisg}, equation~\eqref{eq-keysystem1} becomes,
\begin{equation}\label{eq-keysystem2}
    \Pr(\bin=\binlow\mid\geo=\geolow) =  \sum_{\targlow=1}^{\dimtarg} \Pr(\bin=\binlow\mid\targ=\targlow) \Pr(\targ=\targlow\mid \geo=\geolow).
\end{equation}
When $\bin$ is a low-dimensional categorical variable, the left-hand side, $\Pr(\bin=\binlow\mid\geo=\geolow)$, can be estimated directly from the data. Hence, this system consists of $\dimbin\times\dimgeo$ equations with $(\dimbin-1)\times\dimtarg$ unknowns, while the rank condition for identification remains unchanged (Assumption~\ref{a-georank}).  Solving the resulting system identifies $\Pr(\bin=\binlow\mid\targ=\targlow)$ for all $\binlow$ and $\targlow$. Bayes' rule then yields the desired target posterior probabilities,
\begin{equation}\label{eq-finallbisgprob}
\Pr(\targ=\targlow\mid\bin=\binlow,\geo=\geolow)
\propto
\Pr(\bin=\binlow\mid\targ=\targlow)
\Pr(\targ=\targlow\mid\geo=\geolow),
\end{equation}
where the proportionality constant implies these must be normalized so that the posterior probabilities of group membership sum to one over $\targlow$. 

While any fixed choice of $B$ yields consistent estimates of probabilities of group membership under Assumptions~\ref{a-georank}~and~\ref{a-lbisg}, the efficiency of the resulting estimator depends on how informative $B$ is about the target variable $R$. In particular, greater correlation between $\Pr(\bin=\binlow\mid\geo=\geolow)$ and $\Pr(\targ=\targlow\mid\geo=\geolow)$ leads to more precise $\Pr(\bin=\binlow\mid\targ=\targlow)$, more precise final $\ell$BISG probabilities, and hence more efficient estimation.   

To construct such a variable $B$, we leverage the name lists and the proxy list membership prediction task described in Section~\ref{sec-nameembed}. Specifically, let $f_{\targlow,\emblow} = \Pr(\listmem_{\targlow}=1\mid\emb=\emblow)$ denote a list score, and define $\bm{f}_{\emblow}=(f_{\emblow, 1},\ldots,f_{\emblow, \dimtarg})^\top$ as the vector of list scores for a name with embedding $\emb$. We then apply the $K$-means clustering algorithm to the log-odds transformation of $\bm{f}_{\emblow}$, partitioning these vectors into $K\geq\dimtarg$ clusters with similar patterns of list scores. When the name lists are informative, $B$ retains substantial predictive information about the target variable. In this case, equation~\eqref{eq-keysystem2} yields a system of $\dimgeo \times K$ equations with $\dimtarg \times (K-1)$ unknowns. 

The intuition for why it can be helpful to use more than $\dimtarg$ clusters is that different subgroups of group $\targlow$ may form distinct sub-clusters. For example, we observe this among Asian surnames in the Florida data example (Appendix~\ref{app-embed}). In summary, (1) each name receives group-specific list scores, (2) the bins partition names with similar list scores into clusters, and (3) the weighted regression defined in Section~\ref{sec-meth-getRprob} recovers the distribution over clusters within each group using the relationship between cluster geographic prevalence and group geographic prevalence. 

\begin{algorithm}[!t]
\caption{{$\ell$BISG}}
\label{alg-lbisg}
\begin{algorithmic}[1]
\Require
\begin{tabular}[t]{@{}l@{}}
Target dataset of $n$ observations $\{\sur_i,\geo_i\}_{i=1}^n$,\\
Geographic prevalence estimates $\Prhat(\targ=\targlow\mid \geo=\geolow)$,\\
Name lists $\listmem_{\targlow}(\sur)$ for every group $\targlow=1,\ldots,\dimtarg$,\\
Pre-trained embedding model $\emb(\sur)$,\\
Number of folds $J \geq 2$,\\
Number of clusters $K\geq \dimtarg$.
\end{tabular}
\Statex
\State Obtain name embeddings for each observation $i$: $\emb_i=\emb(\sur_i)$
\State Randomly partition the distinct names into $J$ folds $\mathcal{I}_1,\ldots,\mathcal{I}_J$ containing approximately equal numbers of people
\For{$j=1,\ldots,J$}
\State Fit a predictive model separately for membership in each list $\listmem_{\targlow}$ using all
\Statex\hspace{\algorithmicindent}observations other than those in $\mathcal{I}_{j}$ and estimate $\hat{f}_{\emblow, \targlow}=\Prhat(\listmem_{\targlow}=1\mid \emb=\emblow)$
\State Compute the predicted probabilities for each observation $i \in \mathcal{I}_{j}$:
\Statex\hspace{\algorithmicindent}\hspace{\algorithmicindent}$\hat{\bm{f}}_{E_i}=(\hat{f}_{E_i,1},\ldots, \hat{f}_{E_i,\dimtarg})^\top$
\EndFor
\State Construct $K$ clusters through $K$-means on the log-odds scores $\bm{z}_i$, and assign each observation $i$ to a cluster with cluster membership denoted by $\bin$:
\Statex\hspace{\algorithmicindent}$z_{i\targlow}=\log\{\tilde{f}_{i\targlow}/(1-\tilde{f}_{i\targlow})\}$, \; $\tilde{f}_{i\targlow}=\min\{\max\{\hat{f}_{\emb_i,\targlow},\,\epsilon\},\,1-\epsilon\}$, \; $\epsilon=10^{-6}$
\Statex\hspace{\algorithmicindent}$\bin_i=\bin(\emb_i)\in\{1,\dots,K\}$
\State Calculate empirical cluster proportions within each geography:
\Statex \hspace{\algorithmicindent} \hspace{\algorithmicindent}$\hat{\bm{\rho}}_{b,g}=\Prhat(\bin=\binlow \mid \geo=\geolow)$ with $\hat{\bm{\rho}}_b=(\hat{\bm{\rho}}_{b,1},...,\hat{\bm{\rho}}_{b,\dimgeo})^\top$
\State Obtain probability estimates $\Prhat(\bin=\binlow \mid \targ=\targlow)$ as:
\Statex \hspace{\algorithmicindent}\hspace{\algorithmicindent}$\pivechat_\binlow
=
(\mmathat^\top\bm{W}\mmathat)^{-1}\mmathat^\top\bm{W}\rhovechat_{\binlow}$
\Statex where $\mmathat_{\geolow,\targlow}=\Prhat(\targ=\targlow \mid \geo=\geolow)$, and $\bm{W}=\text{diag}(n_1,...,n_{\dimgeo})$. If there are negative estimates in $\pivechat_\binlow$, set them to zero and re-normalize the remaining estimates over $\binlow$
\State For every $i$, compute the target posterior probabilities:
\Statex \hspace{\algorithmicindent}\hspace{\algorithmicindent}$\Prhat(\targ_i=\targlow \mid \bin_i,\geo_i)
\propto
\Prhat(\bin=\bin_i\mid\targ=\targlow)\,
\Prhat(\targ=\targlow \mid \geo=\geo_i)$
\end{algorithmic}
\end{algorithm}

We formally present the $\ell$BISG method in Algorithm~\ref{alg-lbisg}. When fitting the list membership prediction models in Step~4, we randomly split names into folds and use out-of-sample list scores for each name to reduce overfitting on training names. In our applications, we fit these models using neural networks, for which we choose the number of training epochs by cross-validation within the training data before fitting each model to all the training data.\footnote{In practice, we set negative estimates of $\Pr(\bin=\binlow\mid\targ=\targlow)$ to $10^{-9}$ rather than zero (preventing $\Prhat(\bin=\binlow\mid\geo=\geolow)$ in Algorithm~\ref{alg-selectK} from ever being exactly zero), add $10^{-9}$ to the diagonal of $\mmathat^\top\bm{W}\mmathat$, and drop any cluster without any observations.} 

\subsubsection{Choosing the number of clusters}

We choose the number of clusters $K$ to make $\bin$ as informative about $\targ$ as possible without overfitting: too few clusters leave $\bin$ unable to isolate less common groups, whereas too many clusters introduce sparsity, such that $\Pr(\bin\mid\targ)$ cannot be estimated reliably. Unfortunately, because clustering is not supervised by $\targ$, we cannot select $K$ by held-out accuracy in the usual way.  Instead, we select $K$ by how much geography improves the prediction of clusters in held-out geographic units relative to the overall cluster base rate. This relies on the fact that, under the $\ell$BISG assumptions of Section~\ref{sec-dimreduction},  $\Pr(\bin \mid \geo)$ and $\Pr(\targ \mid \geo)$ co-vary across geography (see equation~\eqref{eq-keysystem2}). 

\begin{algorithm}[!t]
\caption{Choosing the number of clusters $K$}
\label{alg-selectK}
\begin{algorithmic}[1]
\Require
\begin{tabular}[t]{@{}l@{}}
List-score vectors $\hat{\bm{f}}_{E_i}$ and geographies $\geo_i$ from Step~5 of Algorithm~\ref{alg-lbisg},\\
Geographic prior $\Prhat(\targ=\targlow\mid \geo=\geolow)$ for each $\targlow$ and $\geolow$,\\
Candidate set $\mathcal{K}$ for the number of clusters
\end{tabular}
\Statex
\State Randomly partition the geographic units into the training $\mathcal{Z}_{\text{train}}$ and validation $\mathcal{Z}_{\text{val}}$ sets
\For{$K \in \mathcal{K}$}
\State Run Step~7 of Algorithm~\ref{alg-lbisg} using all observations in $\mathcal{Z}_{\text{train}}$, and assign each \Statex \hspace{\algorithmicindent}observation $i \in \mathcal{Z}_{\text{val}}$ to the cluster whose centroid is closest to its log-odds scores $\bm{z}_i$,\Statex \hspace{\algorithmicindent}denoted by $\bin_i$
\State Run Steps~8--9 of Algorithm~\ref{alg-lbisg} using all observations in $\mathcal{Z}_{\text{train}}$ and obtain $\pivechat_\binlow$ for each \Statex \hspace{\algorithmicindent}$b=1,\dots,K$
\State For each geography $g$ in $\mathcal{Z}_{\text{val}}$, compute 
\Statex \hspace{\algorithmicindent}\hspace{\algorithmicindent}\hspace{\algorithmicindent}$\Prhat(\bin = \binlow \mid \geo=\geolow) :=\widehat{\bm M}_g \pivechat_\binlow$
\Statex \hspace{\algorithmicindent}where $\widehat{\bm M}_g$ is the vector of $\widehat{\Pr}(\targ \mid \geo=\geolow)$
\State Estimate $\Prhat(\bin=\binlow)$ as the share of observations in $\mathcal{Z}_{\text{train}}$ that are in cluster $\binlow$
\State Compute the held-out log-likelihood ratio of cluster membership with and without
\Statex \hspace{\algorithmicindent}geography
\begin{equation}
Q(K):=\frac{1}{|\mathcal{Z}_{\text{val}}|} \sum_{i\in \mathcal{Z}_{\text{val}}}\left[\log \Prhat(\bin=\bin_i\mid\geo=\geo_i)-\log\Prhat(\bin=\bin_i)\right] \label{e-georatio}
\end{equation}
\EndFor
\State \Return $K^{*}=\argmax_{K\in\mathcal{K}} Q(K)$
\end{algorithmic}
\end{algorithm}

Algorithm~\ref{alg-selectK} summarizes our selection procedure.
We use grouped sample splitting, partitioning the geographic units into the training and validation sets.
We use the observations in the training set to construct the clusters following Step~7 of Algorithm~\ref{alg-lbisg} and then estimate $\Pr(\bin\mid\targ)$ through the regression in Steps~8~and~9 of the algorithm. 
As formalized by $Q(K)$ in equation~\eqref{e-georatio} of Algorithm~\ref{alg-selectK}, our criterion for selecting $K$ is the difference between the multinomial log-likelihood based on the estimated conditional cluster membership probabilities, $\Prhat(\bin=\bin_i\mid\geo=\geo_i)$ computed as regression fitted values in Step~5 of Algorithm~\ref{alg-selectK}, and the corresponding log-likelihood based on the estimated marginal probabilities, $\Prhat(\bin=\bin_i)$ computed in Step~6. 

Maximizing the conditional log-likelihood alone is not suitable because prediction is generally easier with fewer clusters, causing the criterion to favor small values of $K$.
The marginal probability also captures this tendency, so comparing the two isolates the added predictive value of geography coming from $R$'s relationship to both $G$ and $B$.
In other words, because $\Pr(\bin\mid\geo)$ depends on $\Pr(\targ\mid\geo)$ and $\Pr(\bin\mid \targ)$, the difference in the two log-likelihoods tends to be larger when the cluster variable $\bin$ contains more information about $\targ$.
If $K$ is too large, however, the clusters become too sparse to estimate $\Pr(\bin\mid\targ)$ reliably relative to $\Pr(\bin)$. Selecting the value of $K$ that maximizes $Q(K)$ therefore balances informativeness against estimation instability.\footnote{If $Q(K)$ is largest at the top of the candidate set, we extend the set upward by a factor of $1.5$ until that is no longer the case.} Appendix~\ref{app-selectM} demonstrates that this selection procedure performs well in the Florida voter-file analysis.

\subsubsection{Incorporating multiple name lists for each group}\label{sec-firstnames}

The $\ell$BISG framework naturally accommodates multiple name lists (e.g., first-name and surname lists) for each group. Let $\emb_\first$ and $\emb_\sur$ denote the first-name and surname embeddings, respectively. Applying $\ell$BISG separately to each field yields the calibrated probabilities $\Pr(\targ \mid \bin_\first, \geo)$ and $\Pr(\targ \mid \bin_\sur, \geo)$, where $\bin_\first$ and $\bin_\sur$ denote the first-name and surname clusters.

The standard assumption underlying the incorporation of first names into BISG (known as BIFSG) is that first and surnames are conditionally independent given race and geography \citep{Voicu2018}. For $\ell$BISG$, $ we instead only require the following weaker assumption:
\begin{assumption}[$\ell$BIFSG]
$$
\bin_\first \indep \bin_\sur \mid \targ,\geo.
$$
\end{assumption}
Because the clusters are deterministic functions of the embeddings, this assumption is weaker than conditional independence of the names themselves. Under the assumption,
$$
\Pr(\targ \mid \bin_\first, \bin_\sur, \geo) \propto
\frac{\Pr(\targ \mid \bin_\first, \geo)
\Pr(\targ \mid \bin_\sur, \geo)}
{\Pr(\targ \mid \geo)},
$$
where the proportionality constant normalizes the probabilities to sum to one across all groups.

\subsubsection{Discussion}\label{sec-theory-discussion}

We emphasize that throughout, $\ell$BISG imposes no assumptions on the quality of the list or the ranking.  Indeed, when compared to the standard BISG method, the only additional assumption required for $\ell$BISG is Assumption~\ref{a-georank}, which is typically satisfied if there is some geographic variation in group prevalence. In fact, we also weaken the typical BISG Assumption~\ref{a-bisg} to Assumption~\ref{a-lbisg}. The caveat is that only good lists are able to improve upon the geographic baseline $\Pr(\targ=1\mid G)$.

In the $\ell$BISG algorithm, one could even cluster directly on embeddings $E$ without using name lists -- $\ell$BISG might then stand for ``low-dimensional embedding powered BISG.''  However, such an unsupervised approach relies upon the embeddings being primarily an encoding of group membership so that they are spatially separated in $\mathbb{R}^{\dimemb}$ into clusters that correspond to the groups of interest. By comparison, the proposed semi-supervised approach first projects the embeddings to a lower-dimensional metric specifically constructed for the group prediction task. Neither approach produces clusters with known labels, but the list-based clusters tend to construct more informative clusters. As shown in Section~\ref{sec-recovergeography}, the name lists are required when the geographic prevalence of different groups must be recovered.

Empirically, we find that the lists sometimes carry important information that can improve on clustering alone. For example, in the Florida voter file application presented in Section~\ref{sec-val-voterfiles}, clustering on the list scores rather than on the raw embeddings raises the average precision for the rare Asian group from $0.15$ to $0.25$ at $K=500$. The two approaches perform similarly for the other groups and at the larger number of clusters that Algorithm~\ref{alg-selectK} selects in the full sample (see Appendix~\ref{app:cluster-embeddings}). The same result extends to the Lebanese voter file in Section~\ref{sec-val-lebanon}, where clustering on the list scores rather than the raw embeddings raises the average precision for the Armenian Orthodox, the smallest sect, from $0.64$ to $0.69$ at $K=100$. These results suggest that lists may be especially useful for smaller groups that fully unsupervised clustering struggles to identify.

The clustering step of $\ell$BISG can also be viewed as using the embeddings to update and extend our original lists. For example, if Hernandez originally appeared on the Hispanic list and Fernandez did not, their closeness in embedding space hopefully will mean that Fernandez is also included in a Hispanic cluster. In this way, we generalize initial, possibly quite small, lists to the entire embedding space.

In some settings, researchers may have lists for only a few groups of interest rather than a full partition of the population. Because the subgroup probabilities sum to one, it is sufficient to estimate them for the $\dimtarg-1$ named groups, leaving a residual class as the complement. That said, forming an explicit residual list for that class, for example by gathering names distinctive of the overall population but not of any group of interest, can be worthwhile.  In Section~\ref{sec-recovergeography}, we show that the provision of such a name list becomes necessary once geographic prevalence must be recovered rather than assumed known.

The $\ell$BISG method naturally accommodates overlapping name lists, in which the same name appears on multiple lists. Specifically, each probability $\Pr(\listmem_{\targlow}\mid\emb=\emblow)$ is estimated as a separate prediction task, which treats each name separately for each target group. Consequently, a name appearing on multiple lists simply receives a high list score for each corresponding task, and this may be reflected in the clustering, which is performed on a function of the full vector of list scores, $\hat{\bm{f}}_{\emblow}$. Although overlap makes the target groups less distinguishable, it does not otherwise affect the validity of the method.

Finally, as discussed in Section~\ref{sec-dimreduction}, the assumption of $\ell$BISG (Assumption~\ref{a-lbisg}) is weaker than that of standard BISG (Assumption~\ref{a-bisg}) or eBISG (Assumption~\ref{a-ebisg}). In Appendix~\ref{app-violation}, we measure the extent to which these assumptions are violated in the Florida and North Carolina voter files analyzed in Section~\ref{sec-val-voterfiles}, and examine how the performance of $\ell$BISG and its downstream analysis changes with the size of the assumption violation. We find that the violation is small, with differences between counties explaining at most $0.61\%$ of the variation in the embeddings within a group beyond what random allocation would produce. As expected, we find that the assumption violation is smaller for the clusters than for the embeddings. We also find that the performance of $\ell$BISG and its downstream analysis declines as the extent of the assumption violation increases for most groups.

\subsection{Recovering Geographic Prevalence}
\label{sec-recovergeography}

The $\ell$BISG method developed in Section~\ref{sec-keymethod} assumes access to group geographic prevalence rates $\Pr(\targ=\targlow\mid \geo)$. In this section, we propose methods for recovering geographic prevalence when it is either entirely unavailable (Section~\ref{sec-recovercoarsegeo}) or only available for coarse groups relative to finer subgroups of interest (Section~\ref{sec-subgroupextension}). Compared to the previous section, in these settings, the name lists are required because we recover group geographic prevalence as a function of the rates at which people in a given geographic unit have surnames on each list. Once we recover geographic prevalence estimates, then we can apply the $\ell$BISG method described in Section~\ref{sec-keymethod}. 

\subsubsection{Proposed solution given assumptions about list quality}\label{sec-recovercoarsegeo}

The lack of geographical prevalence data means that we need additional assumptions to make progress.  Unlike the preceding application of $\ell$BISG, the methods we introduce below assume (1) that lists lack any geographic selection bias, and (2) that there is perfect separation across group name lists such that no member of one group may have a name on the list of another.  Formally, we require\footnote{An equivalent formulation of this assumption is: $\listmem_{\targlow} \indep \geo \mid \targ=\targlow$ for all $r$, and $\Pr(\listmem_{\targlow'}=1\mid \geo=\geolow,\targ=\targlow)=0$ for all $\targlow\neq \targlow'$.}

\begin{assumption}[List quality requirements]\label{a-listqual}$ $
    \begin{enumerate}
        \item[(a)] \textbf{No geographic bias:} $\listmem_{\targlow'} \indep \geo \mid \targ=\targlow$ for all $r,r'$
        \item[(b)] \textbf{List exclusivity:} $\Pr(\listmem_{\targlow'}=1\mid \targ=\targlow)=0$ for all $\targlow\neq \targlow'$
    \end{enumerate}
\end{assumption}

Assumption~\ref{a-listqual}(a) requires that, conditional on group membership, list membership is independent of geography.
This assumption is implied by the standard BISG assumption (Assumption~\ref{a-bisg}) provided that list membership is determined solely by name and is not associated with geographic units within each group. The assumption may be violated, however, if only the weaker eBISG or $\ell$BISG assumptions hold (Assumptions~\ref{a-ebisg}~and~\ref{a-lbisg}) so that at the resolution of the original raw names, some geographic specificity also affects the way the list is compiled. For example, suppose a list is compiled from a geographically unrepresentative source, such as the membership rolls of a regional organization. In that case, among individuals in the same group, list membership may be associated with geography (see Section~\ref{sec-liststrategies} for a discussion of list construction and evaluation).

Assumption~\ref{a-listqual}(b) means that the method introduced in this section is applicable only to subgroups whose names on thier lists are distinct from one another. The within-group and across-group list membership rates $\Pr(\listmem_{\targlow'}=1\mid \targ=\targlow)$ reflect how well the lists separate the groups. When $r\neq r'$, a positive cross-group-rate indicates that there are members of one group $\targlow$ whose names appear on the list for another group $\targlow'$. This can happen for two reasons: (1) poor list quality and (2) inherent lack of group separation by names. For example, if trying to separate English and Irish groups, positive cross-group-rates could be caused by using too-expansive lists not designed for the purpose of separating these groups. At the same time, historical and linguistic ties between the two groups may also make these groups inherently difficult to separate.

The list exclusivity assumption allows incomplete coverage (i.e., $\Pr(\listmem_{\targlow}=1\mid \targ=\targlow)<1$). It also remains possible for there to be members of groups $\targlow'$ and $\targlow$ that share the same name, but those names cannot appear on either group's list of distinctive names.  A necessary, but not sufficient, condition for satisfying this perfect separation assumption is that the name lists themselves do not overlap. 

Together, the two restrictions in Assumption~\ref{a-listqual} imply, by the Law of Total Probability,  
\begin{equation}\label{eq-geo-recover2}
\Pr(\listmem_{\targlow} = 1 \mid \geo) = \Pr(\listmem_{\targlow}=1\mid\targ=\targlow)\Pr(\targ = \targlow \mid \geo).
\end{equation}
We do also know the additional restriction that
\begin{equation}
   \sum_{\targlow=1}^{\dimtarg} \Pr(\targ = \targlow \mid \geo) = 1,\label{eq-sumto1}
\end{equation}
To make equation \eqref{eq-sumto1} more useful, we will now, unlike before, require that $\dimtarg$ groups have their own lists (i.e., $\dimtarg-1$ lists is not sufficient).  In other words, the set of groups considered must be exhaustive so we are not inferring one $\Pr(\targ=\targlow\mid \geo)$ after recovering the rest.

Let $\betamat$ be the $\dimgeo\times\dimtarg$ matrix of list rates $\Pr(\listmem_{\targlow}=1\mid\geo)$ for each geography. We assume it is full rank, which requires there to be some variation across geographies in list prevalence.
\begin{assumption}[List prevalence rank condition]\label{a-betarank} Assume $\betamat$ is full column rank.
\end{assumption}

A necessary, but not sufficient, condition is $\dimgeo \geq \dimtarg$. Under this assumption, we show in Appendix~\ref{app-geography-recovery} that equations~\eqref{eq-geo-recover2}~and~\eqref{eq-sumto1} uniquely identify the geographic rates for each group $\targlow$ as,
\begin{equation}\label{eq-georecoverysolution}
 \Pr(\targ = \targlow\mid \geo=\geolow) = \Pr(\listmem_{\targlow} = 1 \mid \geo=g)((\betamat^\top\betamat)^{-1}\betamat^\top \bm{1}_{\dimgeo})_{\targlow}.
\end{equation}

To interpret this formula, imagine directly using list-prevalence as an estimate of geographic prevalence under the exclusivity assumption that anyone whose name is on a group's list is from that group. But, this approach is likely to be biased because the names of some people are not on any list.  Equation~\eqref{eq-georecoverysolution} corrects this undercoverage by up-weighting the list prevalence. Under Assumption~\ref{a-listqual}, $\betamat^\top \bm{1}_{\dimgeo}$ equals the vector of overall list coverages (i.e., the fraction of people in each geography whose names are on some list) for each geographic area. Under perfect coverage ($\betamat^\top \bm{1}_{\dimgeo}=\bm{1}_{\dimgeo}$), equation~\eqref{eq-georecoverysolution} reduces to $ \Pr(\targ = \targlow\mid \geo=\geolow) = \Pr(\listmem_{\targlow} = 1 \mid \geo=g)$. 

We acknowledge that Assumption~\ref{a-listqual} may not hold exactly in practice. In Appendix~\ref{app-geography-recovery}, we develop sensitivity analysis tools for probing how empirical results change under small violations of Assumption~\ref{a-listqual}(b). Specifically, we show how to adjust the denominator of equation~\eqref{eq-georecoverysolution} if the cross-probabilities instead take a small constant $\lambda>0$. However, this sensitivity analysis still does not capture what would happen for greater heterogeneous violations of list exclusivity. Hence, we recommend using it only when the expected violations are fairly small and homogeneous across groups. 

We caution that performance can be poor under heterogeneous non-separation in name lists between groups. To make progress in such cases, researchers will need to gather geographic information about groups and apply the $\ell$BISG method described in Section~\ref{sec-keymethod}. Alternatively, they might consider pooling strongly overlapping groups (e.g., English and Irish) and focusing on recovery of geographic prevalence for the remaining more distinct groups.

\subsubsection{Extension to the subgroup problem}\label{sec-subgroupextension}

We next consider another common but challenging setting, in which geographic prevalence data are not available for groups of interest, but such information does exist for coarser groups. For example, the U.S. Census provides geographic prevalence for the broad Asian category but not for finer subgroups such as Chinese, Korean, Japanese, or South Asian. This again creates a barrier to directly applying the $\ell$BISG method of Section~\ref{sec-keymethod} with higher-resolution name lists, but we can leverage the coarser information in a way that generalizes the method and assumptions described above. 

First, we define both coarse group and subgroup. Let $\coarsegrp\in \mathcal{C}$ denote a coarsened group variable such that each original group $\targ$ belongs to exactly one coarse group $\coarsegrplow = \coarsegrplow(\targlow)$. Let $\dimtarg(\coarsegrplow)$ denote the number of groups within coarse group $\coarsegrplow$ with $\dimtarg=\sum_{\coarsegrplow}\dimtarg(\coarsegrplow)$ the total number of groups. The challenge is again that geographic prevalence rates at the resolution of interest $\Pr(\targ=\targlow\mid \geo)$ are unavailable, but now, we do have the geographic baseline for each coarse group $\Pr(\coarsegrp=\coarsegrplow\mid \geo)$.

We continue to assume that name lists are available for the original higher-resolution groups with list indicator $\listmem_{\targlow}$.  In addition, we also require that these groups exhaustively span each coarser group.  This means that $\dimtarg(\coarsegrplow)-1$ groups that leave a residual category group unrepresented by a list is not sufficient. This assumption ensures that the unknown subgroup rates sum to the coarse group baseline. If, in practice, there are not $\dimtarg(\coarsegrplow)$ lists, but researchers do have access to an overall list for the coarse group, then they can add a residual class whose list is the coarse group's name list minus every name claimed by any of the subgroup lists. The residual name list then enters the embedding step as an ordinary list, contributing its own score to $\hat{\bm{f}}_{\emblow}$. 

Our solution to the subgroup prevalence recovery problem relies on the same assumptions as earlier. Under Assumption~\ref{a-listqual} and the requirement that the groups exhaustively span each coarse group, we again have equation~\eqref{eq-geo-recover2} and now replace equation~\eqref{eq-sumto1} with
\begin{align}
 \Pr(\coarsegrp = \coarsegrplow \mid \geo)&= \sum_{\targlow :\coarsegrplow(\targlow) = \coarsegrplow} \Pr(\targ = \targlow \mid \geo).\label{eq-subsumtogroup}
\end{align}
We show in Appendix~\ref{app-geography-recovery} that this implies,
\begin{equation}\label{eq-subgroup-geo-recov}
    \Pr(\targ = \targlow\mid \geo=\geolow) = \Pr(\listmem_{\targlow} = 1 \mid \geo=g)((\betamat_\coarsegrplow^\top\betamat_\coarsegrplow)^{-1}\betamat_\coarsegrplow^\top \mmat_\coarsegrplow )_{\targlow},
\end{equation}
where $\betamat_\coarsegrplow$ is the $\dimgeo\times\dimtarg(\coarsegrplow)$ matrix of list rates $\Pr(\listmem_{\targlow}=1\mid\geo)$ for coarse group $\coarsegrplow$ in each geography, and where $\mmat_\coarsegrplow$ is the length $\dimgeo$ vector of coarse group prevalence rates $\Pr(\coarsegrp=\coarsegrplow\mid \geo=\geolow)$. This formula takes the group list prevalence in geography $\geolow$ and adjusts it by a term which captures the average relationship between the subgroup list prevalence and the coarse group prevalence across geographies. In practice, we estimate this adjustment by non-negative least squares with weighting of each geography by its expected number of coarse group members, followed by a normalization step to ensure final estimates sum to exactly $\Pr(\coarsegrp=\coarsegrplow\mid \geo)$ (see Appendix~\ref{app-nonnegative-ls} for details). 

In Appendix~\ref{app-geography-recovery}, we again provide a sensitivity analysis extension. We show how to adjust the denominator of equation~\eqref{eq-subgroup-geo-recov} if the cross-probabilities uniformly take a small value $\lambda>0$ within each coarse group and uniformly take a value $\lambda \geq \gamma\geq 0$ across coarse groups.

\subsubsection{Discussion}

The case where geographic information is missing represents a common but difficult problem. In such situations, high-quality name lists become not just a useful tool for better clustering of names but essential for being able to identify meaningful probabilities at all. The key intuition is that high-quality lists that are not geographically skewed contain some information about the geographic distribution of their group members. For well-separated groups with high-quality lists, the above method allows us to recover $\Pr(\targ=\targlow\mid \geo)$ even when there is no prior geographic information at all. When subgroups are of interest, and yet only the geographic prevalence of a coarse group is available, high-quality lists allow us to recover subgroup geographic prevalence as well. Once recovered, these estimates of geographic prevalence can be used by applying the $\ell$BISG method of Section~\ref{sec-keymethod} to again obtain the target posterior probabilities.

There may be further ways to weaken some of the assumptions made in this section, especially if prior information is available about which groups are likely to be poorly separated or at what resolution list geographic selection mechanisms occur. For example, 
one way to weaken Assumption~\ref{a-listqual} is to apply it only within a coarser geographic unit (e.g., within each county when the original geographic unit is census tract). In that case, all subsequent steps should be applied separately within each region. However, even if this conditioning means there is no further relationship between list membership and a higher-resolution geographic unit, list quality may be poor in regions whose names were originally under-represented, leading to potentially low values of $\Pr(\listmem_{\targlow}=1\mid \geo)$ or $\Pr(\listmem_{\coarsegrplow}=1\mid \geo)$ and unstable results for those regions. It may then be necessary to try to build separate name lists for each region.

\subsection{Obtaining and Evaluating Name Lists}\label{sec-lists}

In this section, we briefly discuss how to obtain and evaluate group-specific name lists. Section~\ref{sec-recovergeography} introduced some quality requirements on lists. But, more broadly, for any $\ell$BISG application, better list quality will lead to more informative clusters $B$ and hence final probabilities that differ more from predictions made using geography alone. Here, we formalize some notions of list quality. 

\subsubsection{Generating name lists}\label{sec-liststrategies}

There are different ways to build name lists. First, lists can be curated from domain knowledge or existing records. For example, they may be constructed from prior surveys, censuses, or genealogical records; membership rolls of group-affiliated organizations; obituaries associated with group-specific funeral homes \citep{chasalow2026sampling}; curated Wikipedia pages \citep{sood2018predicting}; or the activities of identified users on social media \citep{wood2018predicting}.

A second approach, which makes $\ell$BISG applicable even when no frequency table or curated record exists, is to prompt an LLM to generate names that are typical of a target group, thus leveraging the world knowledge encoded during pre-training \citep{Dasanaike2026zeroshot}. This approach is broadly applicable across many types of groups and targets. 
To improve separation between groups in LLM generated lists, one can apply a \emph{within-group exclusivity filter} by removing any name that appears on multiple group lists. This is necessary, though not sufficient, for the geographic prevalence recovery method described in Section~\ref{sec-recovergeography}. In addition, one can prompt an LLM to generate names that are distinctively associated with the group of interest. The generated list can then be refined by asking a possibly different model to assess the distinctiveness of each name and retaining only those that are judged to be truly distinctive. 

More generally, an LLM agent with access to web search can combine the two approaches by constructing a list from different sources such as those described above, or by generating names and validating each against them. Agents can also be used to clean or supplement lists generated by other models. Agents are particularly useful for probing the list exclusivity required by the geographic prevalence recovery method of Section~\ref{sec-recovergeography}. For example, we can ask agents which names on one list might also belong to members of another group.

\subsubsection{Desirable properties of name lists}\label{sec-listdesirable}

What makes a good name list? Ideally, a name list contains names that are both \emph{distinctive} and \emph{common} within the group of interest. Distinctive names have high positive predictive value for the target group (i.e., $\Pr(\targ=\targlow \mid \sur=\surlow)$ is high), whereas common names are held by a large share of the target group (i.e., $\Pr(\sur=\surlow \mid \targ=\targlow)$ is high). 

In addition, a name list should ideally cover different subgroups within the group of interest. For example, a list of Asian names that entirely excludes Japanese names may have distinctive and common names for other subgroups, like Chinese, but may not capture the full diversity of the group. If the excluded subgroup is concentrated in particular geographic areas, the list also violates the no-geographic-selection assumption required for the geographic prevalence recovery method discussed in Section~\ref{sec-recovercoarsegeo}.

Returning to the different types of lists described above, curated lists can contain highly distinctive names, but they may or may not have been designed to include common yet less distinctive names. Their coverage also depends on the source. Sources tied to a particular place may violate the no-geographic-selection assumption (Assumption~\ref{a-listqual}) that the geographic prevalence recovery method of Section~\ref{sec-recovergeography} relies on. On the other hand, the properties of LLM generated lists are more of an open question. Because they are trained to return probable responses, it is possible they could under-represent names that are highly distinctive but less common.

To illustrate and explore LLM list generation, we generate a name list for each of the four racial groups in the Florida voter file analyzed in Section~\ref{sec-val-voterfiles} using an LLM, Gemini~3.0 Flash. This corresponds to a realistic setting in which researchers must create lists on their own without relying on name frequency tables.  To cover the subgroups within each racial group, we use separate prompts for different national origins. For the White list, for instance, we use three prompts, each requesting surnames of individuals with English and Irish descent; German descent; and Italian, Polish, and other European descent. We split Asian into Chinese; Indian; Vietnamese and Filipino; and Korean and Japanese. Appendix~\ref{app-prompts} provides the exact prompts.

For each of these prompts, we make six separate model calls at a temperature of $1.0$, and request $300$ surnames per call. With three prompts each, the White, Black, and Hispanic groups have eighteen calls and $5{,}400$ requested surnames, while the four Asian prompts give a total of $7{,}200$ surnames. Within each group, we pool the results of all calls across prompts, rank the unique surnames by the number of calls that returned them, and then retain the top $1{,}000$ surnames. Since the top 1,000 threshold often includes a number of ties, we break those ties by the order in which the model first returned each surname.

We find that for $\ell$BISG on the Florida voter file, the gain in discrimination from longer lists is essentially complete at a few hundred surnames per group (see the top panel of Figure~\ref{fig:listsize} in Appendix~\ref{app-listgen}). The same analysis on the North Carolina voter file (the bottom panel of the same appendix figure) produces a similar result for the White and Black groups. For the smaller Asian group in North Carolina, however, discrimination continues to improve beyond a few hundred surnames.

\begin{figure}[!t]\centering
\includegraphics[width=\textwidth]{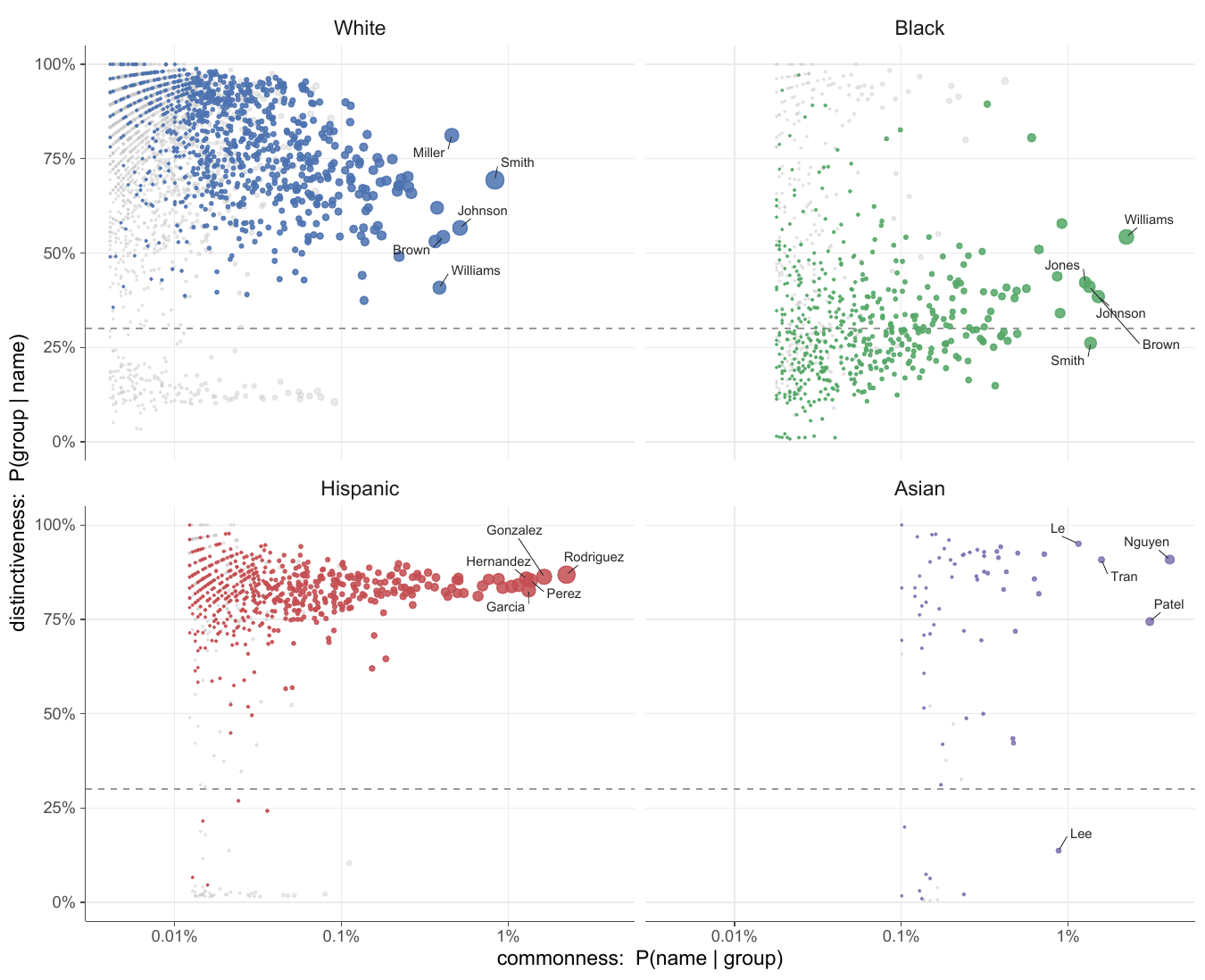}
\caption{Florida surnames plotted by log commonness $\Pr(\sur \mid \targ)$ ($x$-axis) against distinctiveness $\Pr(\targ \mid \sur)$ ($y$-axis) for the four race groups. Each point is a surname held by at least twenty-five members of the group, sized by how many voters hold it. Colored points are those included in the list of $1{,}000$ surnames generated by an LLM for that group. Labeled points are the most common surnames in each group. The dashed line marks the distinctiveness floor of $0.3$ used to build lists in Section~\ref{sec-val-voterfiles}. Because $\Pr(\targ \mid \sur)$ is equal to the product of $\Pr(\sur \mid \targ)$ and the number of group members divided by the number of voters holding the surname, surnames held by the same number of voters fall on a straight line. On the log scale of the $x$-axis, each such line appears as a curved arc, which is especially clear for White and Hispanic voters.}
\label{fig:listprops}
\end{figure}

Figure~\ref{fig:listprops} plots the surnames held by Floridians of each race along distinctiveness ($y$-axis) and commonness ($x$-axis). Each point represents a surname held by at least 25 individuals, and the colored points indicate the surnames included in the LLM-generated list of $1{,}000$ surnames for that group.

The four racial groups exhibit distinct patterns. White surnames span a broad range of both commonness and distinctiveness. Hispanic surnames form a dense band characterized by high distinctiveness and a range of commonness, making them well suited for both identifying and representing the group. Black surnames are generally less distinctive because they are more frequently shared with other groups, in particular white, although there are some quite common names. Asian surnames resemble Hispanic surnames in being both relatively common and highly distinctive, with some exceptions including ``Lee,'' which is relatively common but not distinctive at all. Overall, the LLM appears to generally select surnames that are both distinctive and common. The exception is Black surnames, for which it prioritizes commonness over distinctiveness.

Coverage matters for another reason. The embedding step extends the information in a list to unlisted names that are close in embedding space. A list which provides a broader coverage of the group's name distribution thus yields more accurate probabilities for unlisted names. In Appendix~\ref{app-embed}, we visually inspect the coverage of the LLM surname lists of Section~\ref{sec-val-voterfiles} in the embedding space of the Florida surnames. The high list scores for each group fall in distinct regions of this space, with Hispanic in particular demonstrating strong clustering. These regions include many surnames that are not on any list.

\subsubsection{Measure of list quality when geographic prevalence is available}\label{sec-listqualitymeas}

In the absence of ground-truth data, the relationship between list membership and true group membership cannot be directly measured. However, there are ways to recover this list quality information. We define two per-group metrics of list quality: \textit{coverage} and \textit{precision}, as well as the overall quality of a collection of lists. In this section, we show how to estimate these quality metrics under the settings considered in Section~\ref{sec-keymethod} (geographic prevalence estimates available) and Section~\ref{sec-recovergeography} (geographic prevalence recovered).

Coverage probabilities, which we define as $\Plisttarg_{\targlow,\targlow'} := \Pr(\listmem_{\targlow'} = 1 \mid \targ = \targlow)$ capture the fraction of group members whose name is on a given list. Coverage of a group by its own list is desirable ($\Plisttarg_{\targlow,\targlow}$ is high), while cross-coverage ($\Plisttarg_{\targlow,\targlow'}$ is high) is undesirable. When geographic prevalence is known as assumed in Section~\ref{sec-keymethod}, and lists additionally satisfy the assumption of no geographic selection (Assumption~\ref{a-listqual}(a)), then it is possible to recover all the probabilities $\Plisttarg_{\targlow, \targlow'}$ using the same strategy as the one used for recovering  $\Pr(\bin=\binlow \mid \targ=\targlow)$.  Specifically, we solve the following linear system of $\dimgeo\dimtarg$ equations with $\dimtarg^2$ unknowns,
\begin{equation}\label{eq-system1} \Pr(\listmem_{\targlow'}=1\mid \geo = \geolow) = \sum_{\targlow = 1}^{\dimtarg} \Plisttarg_{\targlow,\targlow'}\Pr(\targ = \targlow \mid \geo = \geolow).
\end{equation}
This system has a unique solution if there is sufficient variation in these rates to satisfy Assumption~\ref{a-georank}.

On the other hand, when no geographic prevalence are available, Assumption~\ref{a-listqual}(b) strongly imposes a list quality constraint by assuming no cross-coverage ($\Plisttarg_{\targlow,\targlow'}=0$). Under this assumption, we can then at least estimate the own-group coverage, i.e., $\Plisttarg_{\targlow,\targlow} = \Pr(\listmem_{\targlow}=1\mid \geo=\geolow)/\Pr(\targ=\targlow\mid \geo=\geolow)$, but these results are valid only when cross-coverage is zero.

Our overall list quality metric is based on the $\dimtarg \times \dimtarg$ coverage matrix $\bm{\Plisttargmat}$ where $\Plisttargmat_{\targlow,\targlow'}=\pi_{\targlow,\targlow'}$. We measure the quality of $\Plisttargmat$ using the diagonal dominance:
\begin{equation}
 D(\Plisttargmat) =  \frac{1}{\dimtarg}
\sum_{\targlow=1}^{\dimtarg}
\left[\Plisttarg_{\targlow,\targlow}
-
\frac{1}{\dimtarg-1}
\sum_{\targlow'\neq\targlow}
\Plisttarg_{\targlow,\targlow'}\right].
\end{equation}
which is maximized at the ideal $\pi_{\targlow,\targlow'}=\ind\{\targlow=\targlow'\}$ with value $1$ and minimized at a pathological scenario of $\pi_{\targlow,\targlow'}=\ind\{\targlow\ne\targlow'\}$ with value $-1$. Positive values indicate a tendency for diagonal $\pi_{\targlow,\targlow}$ values to be greater than the off-diagonal values. In the binary group case where $\targ=1$ indicates the group of interest and $\targ=0$ is simply the residual group with no list, we can solve for $\pi_1=\Pr(\listmem = 1 \mid \targ=1)$ and $\pi_0 = \Pr(\listmem=1 \mid \targ = 0)$ and measure quality as $\pi_1-\pi_0$. Estimating $D$ can be a useful way to assess list quality when true labels are unavailable for computing other evaluation metrics. Indeed, in Sections~\ref{sec-val-voterfiles}~and~\ref{sec-val-lebanon} below, we show that the estimates of $D(\Plisttargmat)$ are close to their ground truths. 

Our final measure of list quality is precision: the proportion of people on a list who belong to its intended target group. By Bayes' rule, this equals
\begin{equation}\label{eq-precisiondef}
\Pr(\targ=\targlow \mid \listmem_{\targlow}=1) = \frac{q_{\targlow}\pi_{\targlow,\targlow}}{\sum_{\targlow'} q_{\targlow'}\pi_{\targlow',\targlow}}
\end{equation}
where $q_{\targlow}=\Pr(\targ=\targlow)$ is the marginal group share, which can be obtained from the geographic prevalence estimates. A notion of cross-precision $\Pr(\targ=\targlow \mid \listmem_{\targlow'}=1)$ can be defined similarly. High precision indicates good list separation, but we caution that precision can be more affected by group base rates, whereas coverage is more affected by list base rates. For example, if the names of Asian people in Florida have even a moderate presence on White name lists ($\pi_{r'=\text{White},{r=\text{Asian}}} > 0$), the high base rate of the white group ($q_{\text{White}}\gg 0$) and low base rate of the Asian group ($q_{\text{Asian}}\approx 0$) can pull down the Asian precision by inflating the denominator of equation~\eqref{eq-precisiondef}, even if $\pi_{\text{Asian},\text{Asian}}$ is high.

\subsubsection{Discussion}

Ultimately, unless a population has a perfect partition of names by group with no overlap between groups (which is almost never the case), there is no one true ground truth list. For the $\ell$BISG method with geographic prevalence available (Section~\ref{sec-keymethod}), higher quality lists are those that cover all the most prevalent and distinctive underlying name patterns of a group, but lower quality lists may still be valid, though less informative, if they do not threaten Assumption~\ref{a-lbisg}. In the extreme, suppose lists were formed as random subsets of names. Then we expect the list scores $\bm{f}_{\emblow}$ to reflect no true predictive power and only vary due to random chance. The $\ell$BISG probability $\Pr(\targ=\targlow\mid\bin,\geo)$ will then simply not improve on the geographic baseline $\Pr(\targ=\targlow\mid\geo)$ because the clustering $B$, though trivially satisfying Assumption~\ref{a-lbisg}, will not be informative.

For the geographic recovery version of $\ell$BISG (Section \ref{sec-recovergeography}), we can take steps in list construction to make the required list exclusivity assumption more plausible. Here, one notion of an ideal list is a list that include all names that are exclusive (or in practice, almost exclusive) to its group. Yet including only these names could come at some cost if there are names that are not mutually exclusive to the group and correspond to a distinctive subgroup (i.e., there can be a list-exclusivity vs list-coverage trade-off). Moreover, we also find in our empirical examples that pushing lists to be larger has diminishing returns; fairly short lists can be sufficient. Quality matters more than quantity.

\subsection{Downstream estimation}

For downstream estimation, \cite{McCartan02102025} proposes the BIRDiE framework,
showing that if the outcome is conditionally independent of name given race and geography (i.e., $\outcome \indep \sur \mid \targ, \geo$), racial disparities can be estimated without bias using calibrated racial probabilities, $\Pr(\targ=\targlow \mid \sur, \geo)$, rather than imputing race through classification. This assumption is often more credible than the assumption required by the standard weighting approach, which assumes that race is conditionally independent of the outcome given name and geography (i.e., $Y\indep R\mid \sur, \geo$) (see \cite{McCartan02102025}, for details). Under either assumption, accurate estimation of calibrated race probabilities is essential for racial
disparity analysis. In our empirical validations (Section~\ref{sec-validation}), we use the BIRDiE approach with $\ell$BISG probabilities.

\section{Results}\label{sec-validation}

We empirically validate $\ell$BISG in three diverse settings with known ground truth. First, we use voter files from Florida and North Carolina, where self-reported race is available.  We compare the performance of $\ell$BISG with that of the standard BISG based on Census name frequency tables.  We show that $\ell$BISG performs as well as the standard BISG despite not using the name-race frequency data. 
Second, we use the full-count 1900 U.S. Census to predict Chinese and Japanese subgroups within the Asian population. This evaluates the performance of the methodological extension proposed in Section~\ref{sec-recovergeography} when geographic prevalence is not available for subgroups of interest. This analysis also corresponds to a level of granularity for which no name frequency tables naturally exist.  Third, we consider the prediction of religious sect in Lebanon, illustrating the applicability of $\ell$BISG outside the conventional race-based BISG setting in the United States. 

In each application, we evaluate the calibration of the predicted probabilities, their discriminatory power, and the bias and RMSE of downstream disparity estimates under the BIRDiE framework \citep{McCartan02102025}.

\subsection{Comparison with standard BISG using US voter files}\label{sec-val-voterfiles}

\paragraph{Setup.}
We obtain Florida and North Carolina voter files from L2 Inc., a leading national non-partisan firm and the oldest organization in the United States that supplies voter data and related technology to candidates, political parties, pollsters, and consultants for use in campaigns.  These voter files contain the self-reported race of each registrant alongside their geocoded residential address, turnout in recent elections, and political party registration.

We collapse the race field into five categories, White, Black, Hispanic, Asian, and other, the last of which includes American Indian and Alaska Native. We use the Florida file of March 7, 2026 and the North Carolina file of April 11, 2026. We exclude registrants whose race is not recorded in the voter file and those without Census geocodes, comprising $494{,}443$ and $627{,}520$ of Florida's $15{,}017{,}471$ voters and $649{,}911$ and $318{,}860$ of North Carolina's $7{,}249{,}018$, respectively. Lastly, we draw random samples of one million registrants from the remainder in each state.

We apply $\ell$BISG and standard BISG using the published 2020 Census block-level race shares as the geographic priors. For registrants whose blocks have zero population, we use the tract prior instead.\footnote{These blocks have zero population in the Census table partly because of the noise the Census Bureau adds to block counts to protect privacy. This accounts for $1.1\%$ of the Florida sample and $0.9\%$ of the North Carolina sample. We use the county prior for registrants whose blocks are absent from the Census table, totaling $0.02\%$ of the Florida sample and $0.13\%$ of the North Carolina sample.} In the downstream analysis, we use these predictions to estimate the Democratic share of each racial group aggregated at the county level, based on party registration.

To estimate the individual race probabilities using $\ell$BISG, we draw on two sets of name lists in turn. The first name list is an oracle one based on the 2020 Census surname table, which reports the racial composition of every surname occurring at least $100$ times in the Census. Admittedly, the $\ell$BISG method is targeted at applications where such tables are not available. However, because this list is directly derived from the Census data, it allows us to isolate the consequences of losing frequency information on the quality of individual race probability estimates and downstream analysis by comparing the resulting estimates with those based on the standard surname-frequency-table-based BISG. 

As described in Section~\ref{sec-listdesirable}, a useful list has names that are sufficiently common and yet distinctive. For each group $\targlow$, we construct the corresponding Census list by filtering to surnames that are at least $30\%$ likely to belong to that group, i.e., $\Pr(\targ=\targlow \mid \sur=\surlow) \geq 0.30$.\footnote{We set the floor at $0.30$ because it is a smallest value at which the four lists are close to disjoint. For example, $7.0\%$ of Florida registrants have a surname belonging to more than one list at a floor of $0.20$, whereas this proportion is only $0.04\%$ at a floor of $0.30$.  Raising the floor further reduces coverage among Black registrants while leaving the other three lists almost unchanged (see Table~\ref{tab:threshold} in Appendix~\ref{app-listgen}).}  We then rank each surname by how many group members have that surname, beyond the number expected in a random set of the same sample size, i.e., $\Pr(S = s \mid R = r) - \Pr(S = s)$. Lastly, we take the 1,000 highest-ranked surnames for each group. Because we construct each list independently across groups, a surname may qualify for more than one list. However, this is rare, as only seven surnames appear on two lists after filtering to the top 1,000.

We obtain the second set of name lists from an LLM, Gemini~3.0 Flash, as described in Section~\ref{sec-listdesirable}. For $\ell$BIFSG, we construct first-name lists in the same two ways, from the 2020 Census first-name table by the rule above and from the LLM with the same prompts. The LLM lists represent a more realistic generation approach for settings where $\ell$BISG is being used exactly because the surname frequency tables are unavailable.

Finally, we embed every distinct observed surname and first name using the multilingual E5-large model \citep{wang2024e5} and estimate the list score $f_{e,r}$ for each list separately using a multilayer perceptron.\footnote{The model has two hidden layers of $256$ and $128$ units, ReLU activations, and dropout of $0.3$.  We use PyTorch with Adam at a learning rate of $10^{-3}$ and weight decay $10^{-5}$ under binary cross-entropy.} Each model is fitted using five folds so that list scores are computed out-of-sample (see Steps 3--6 of Algorithm~\ref{alg-lbisg}). Within each training set (four folds), we choose the number of epochs by cross-validation, leading to one or two epochs.

We then apply the $\ell$BISG method via Algorithm~\ref{alg-lbisg}, selecting the number of clusters with Algorithm~\ref{alg-selectK}, which returns $K=7{,}500$ for the Florida LLM surname list and the North Carolina LLM first-name list, $K=5{,}000$ for the Florida LLM first-name list, the North Carolina LLM surname list, and the Florida Census first-name list, $K=10{,}000$ for both Census surname lists, and $K=3{,}000$ for the North Carolina Census first-name list.\footnote{We use the candidate grid ranging from $10$ to $20{,}000$ clusters. Above $3{,}000$, the grid has values of $5{,}000$, $7{,}500$, $10{,}000$, $15{,}000$, and $20{,}000$.} We fit standard BISG using the same 2020 Census surname table and geographic priors.

\begin{table}[!t]\centering
\begin{tabular}{llrrrrr}
\hline
 & & & \multicolumn{2}{c}{Coverage} & \multicolumn{2}{c}{Precision} \\
Group & List & In Census & Ground truth & Estimate & Ground truth & Estimate \\
\hline
White    & Census & $100\%$ & $0.28$ & $0.30$ & $0.79$ & $0.74$ \\
         & LLM    & $100\%$ & $0.31$ & $0.32$ & $0.71$ & $0.64$ \\
Black    & Census & $100\%$ & $0.33$ & $0.37$ & $0.46$ & $0.46$ \\
         & LLM    & $100\%$ & $0.53$ & $0.57$ & $0.25$ & $0.24$ \\
Hispanic & Census & $100\%$ & $0.50$ & $0.53$ & $0.82$ & $0.97$ \\
         & LLM    & $99\%$ & $0.50$ & $0.53$ & $0.81$ & $0.97$ \\
Asian    & Census & $100\%$ & $0.42$ & $0.41$ & $0.64$ & $0.73$ \\
         & LLM    & $97\%$ & $0.39$ & $0.33$ & $0.17$ & $0.17$ \\
\hline
\end{tabular}
\caption{Comparison of LLM-based and Census-based surname lists on the Florida registrants. The column labeled as ``In Census'' shows the share of list names that appear in the 2020 Census surname table at all. ``Coverage'' is the share of the group's registrants with a surname on the list, and ``Precision'' the share of the registrants with a list surname who belong to the group, each computed from the self-reported race (ground truth) and estimated from the tract system with the tract-level race shares as the geographic prior and no individual labels (Section~\ref{sec-listqualitymeas}).}
\label{tab:lists}
\end{table}

\paragraph{Quality of name lists.} Before we examine the performance of $\ell$BISG, we investigate the quality of LLM name lists relative to that of the Census-based lists, which can be thought of as a gold standard. For these list-quality diagnostics, we aggregate list membership and Census race shares at the tract level.  As shown in the column labeled ``In Census'' in Table~\ref{tab:lists}, between $97$ and $100\%$ of the LLM surnames generated for each group appear in the Census table, indicating that the model reliably produces plausible names. We note that the Census released the 2020 surname table in April 2026 and that Gemini~3.0 Flash was trained only on data through January 2025, so the validation data is not included in the training data. The 2000 and 2010 surname tables have been public for much longer, and so the model may be aware of them.

The two name list constructions diverge in coverage and precision, in particular for the Black and Asian lists. For Black registrants, the LLM covers $0.53$ of the group against $0.33$ for the Census list. It has a precision of $0.25$ against $0.46$, indicating that the model returns more common but less distinctive surnames when compared to the Census lists. The LLM is worse on both metrics for Asian registrants, $0.39$ against $0.42$ with respect to coverage and $0.17$ against $0.64$ for precision. For White and Hispanic registrants, the two lists are close on both measures. 

In addition, the LLM lists overlap across racial groups much more often than the lists based on the Census surname data: $26.9\%$ of Florida registrants have a surname belonging to more than one LLM list, compared to only $0.04\%$ for the Census-based lists.  This makes sense because our Census-based name list is intentionally constructed to favor distinctiveness while LLMs may be more prone to returning common names. 

We also report the measure of list quality proposed in Section~\ref{sec-listqualitymeas}, based on the diagonal dominance $D(\Pimat)$ of the matrix whose $(\targlow, \targlow')$ entry $\pi_{\targlow,\targlow'}$ is the rate at which members of group $\targlow$ hold a name on list $\targlow'$ (see Figure~\ref{fig:pi-quality}). Estimated from the tract-level list hit rates and the published tract race shares alone, $D(\Pimat)$ is $0.36$ for the Census lists and $0.34$ for the LLM lists. This suggests that the former is of slightly higher quality than the latter.  These estimates are substantively similar to their ground truths based on self-reported race, which are $0.34$ and $0.30$ for the Census and LLM lists, respectively, indicating that the proposed quality measure is reliable. 

The estimated $\Pimat$ matrix also recovers coverage and precision well. Table~\ref{tab:lists} reports both quantities, as estimated from the tract hit rates and race shares without individual race labels, next to their ground truths. Estimated coverage is within $0.06$ of its ground truth and estimated precision within $0.16$ for every list. The largest differences are observed for the Hispanic lists and the Census Asian list, where the estimates overstate precision. The Census lists are more precise than the LLM lists for every group, both in the ground truth and the estimates.

\begin{figure}[!t]\centering
\includegraphics[width=\textwidth]{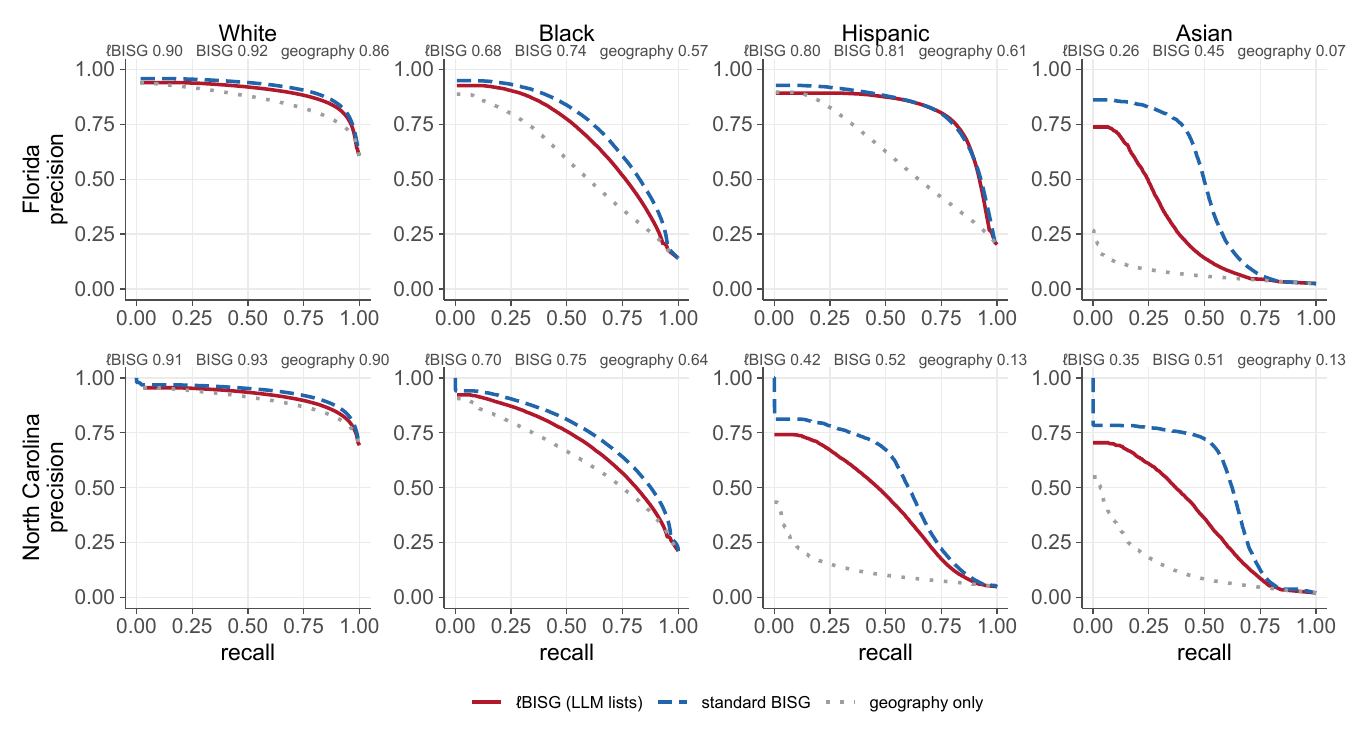}
\caption{Precision and recall for $\ell$BISG (LLM lists, red, solid), standard BISG (blue, dashed) and geography alone (grey, dotted), by state (rows) and race (columns). Average precision is reported above each panel.}
\label{fig:val-pr}
\end{figure}

\begin{figure}[!t]\centering
\includegraphics[width=\textwidth]{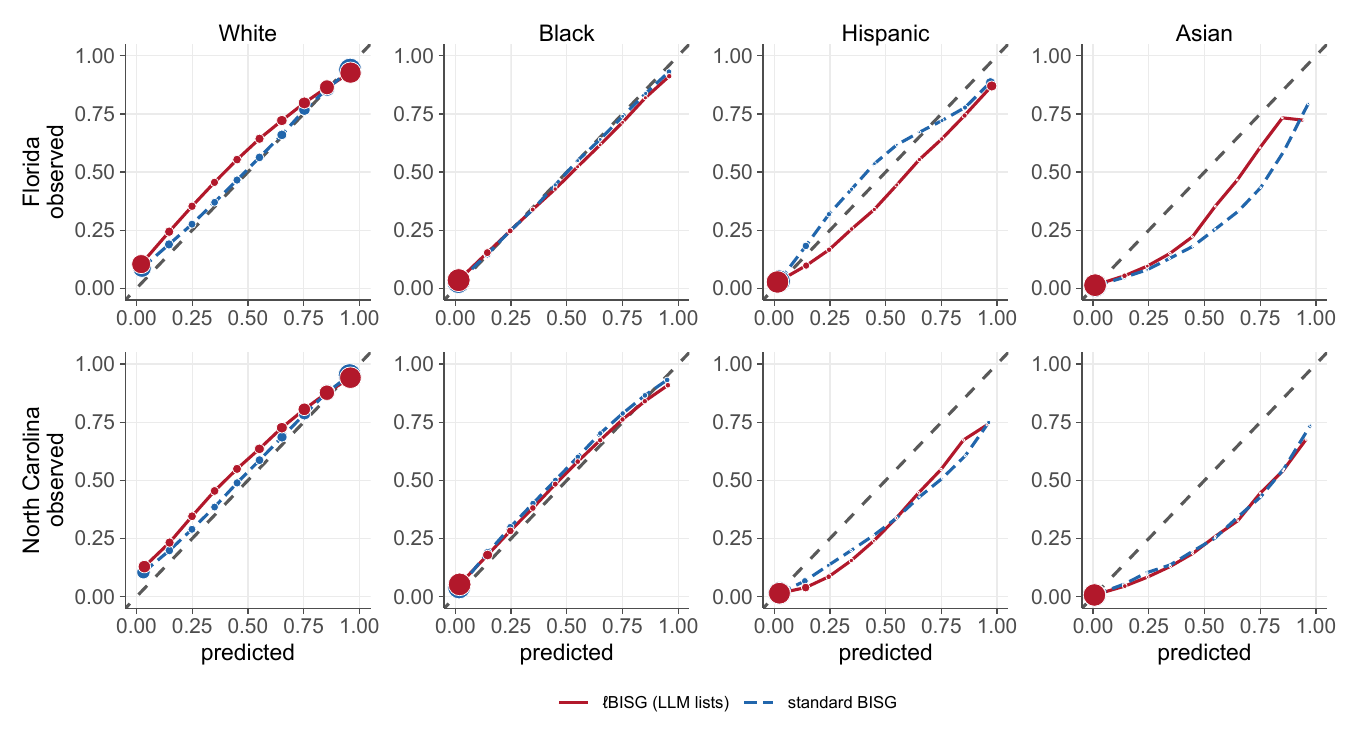}
\caption{Calibration of estimated individual race probabilities based on $\ell$BISG (LLM lists, red, solid) and standard BISG (blue, dashed) for Florida (top row) and North Carolina (bottom row). The dashed diagonal marks perfect agreement between estimated probabilities and observed proportions. The area of each circle is proportional to the number of registrants in the bin.}
\label{fig:val-calibration}
\end{figure}

\paragraph{Results.} We compare the empirical performance of $\ell$BISG with LLM lists and standard BISG for Florida and North Carolina.  As shown in Figure~\ref{fig:val-pr}, $\ell$BISG (red solid line) nearly matches the precision and recall of standard BISG (blue dashed line) for Whites in both states and Hispanics in Florida, despite having no access to a surname frequency table. For Black and Asian voters, standard BISG outperforms $\ell$BISG, by more for Asian voters; standard BISG also performs better for Hispanics in North Carolina. This reflects the lower precision of the LLM name lists for these two groups (Table~\ref{tab:lists}). In all cases, $\ell$BISG is still a substantial improvement on using geography-based probabilities alone (gray dotted line).

Figure~\ref{fig:val-calibration} evaluates calibration, showing how closely the predicted individual race probabilities correspond to observed group frequencies. With the exception of the Hispanic group in Florida, the calibration patterns are similar for $\ell$BISG and standard BISG. Calibration is notably imperfect in the Asian and Hispanic groups, but this is true for either method. For Black voters, both methods have expected calibration error of about $0.02$ or less in Florida and about $0.03$ in North Carolina. For White voters, $\ell$BISG probabilities tend to slightly underpredict this racial group.

We also evaluate the validity of downstream analysis using the individual race probabilities obtained from $\ell$BISG.  We estimate the Democratic share of registrants in each racial group with BIRDiE \citep[see][for details]{McCartan02102025}. BIRDiE takes the $\ell$BISG probabilities as its first-stage estimates and then uses the surname as a proxy predictive of race, while assuming that the surname is independent of the outcome conditional on race and geography. Although the BIRDiE assumption may not hold perfectly, it corrects the attenuation bias of other standard downstream estimation approaches that assume conditional independence between race and outcome conditional on name and geography.\footnote{This latter assumption, which is made in some standard estimation approaches, is especially difficult to justify for Black and white Americans who have similar surnames.}  We fit BIRDiE jointly over all five racial categories at the county level, with partial pooling across counties because the race-by-county cells for the smaller groups hold few registrants. For each racial group, we compute the bias and RMSE over every county, weighting each county via the number of the group's members implied by the geographic prior.

\begin{figure}[t]\centering
\includegraphics[width=\textwidth]{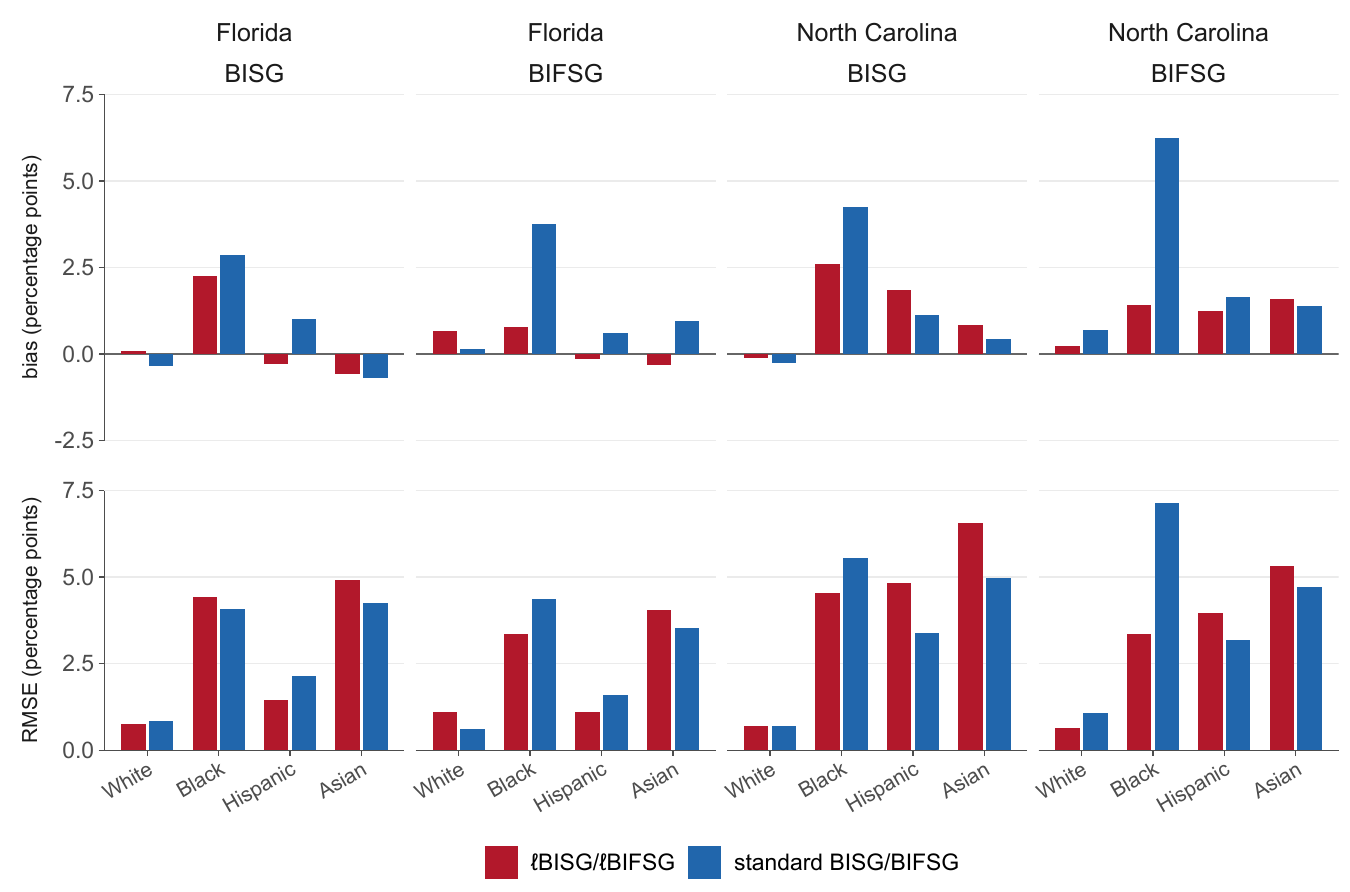}
\caption{County-level race-specific estimates of the Democratic share of registrants. Rows give signed bias (top) and root-mean-square error (bottom), both in percentage points and averaged over counties. The left two columns are Florida and the right two North Carolina, each pair showing the surname-only estimators (BISG) and then those that add first names (BIFSG). Red bars are the $\ell$BISG/$\ell$BIFSG (LLM lists) estimators, while blue bars are the estimates based on the standard BISG/BIFSG.}
\label{fig:val-disparity}
\end{figure}

Figure~\ref{fig:val-disparity} compares the estimated Democratic share of each racial group at the county level against the truth. The top panel shows that the magnitude of bias is smaller for $\ell$BISG (red bars) than for standard BISG (blue bars) for both White and Black registrants in both states.  For Hispanic and Asian registrants, the bias of $\ell$BISG is smaller in Florida, but in North Carolina, where biases are overall larger, there is an increase in bias relative to the standard BISG, again reflecting the low list quality identified above.

The bottom panel shows the root mean squared error (RMSE) has an inconsistent pattern for $\ell$BISG versus standard BISG. Specifically, $\ell$BISG improves on standard BISG for White and Hispanic registrants in Florida. $\ell$BISG is within $0.1$ percentage points (ppts) of standard BISG in county RMSE for White registrants in both states. For Black registrants, $\ell$BISG falls behind standard BISG by $0.4$ ppts in Florida and improves on it by $1.0$ ppts in North Carolina. For Hispanic registrants in North Carolina, $\ell$BISG falls behind standard BISG by $1.4$ ppts. $\ell$BISG also falls behind standard BISG for the rarer Asian group, whose name list is weaker, by $0.7$ ppts in Florida and $1.6$ ppts in North Carolina.

Adding first names ($\ell$BIFSG) lowers county-level error in most cases. In terms of RMSE, $\ell$BIFSG improves on $\ell$BISG for all four groups in North Carolina, by $1.2$ ppts for Black registrants and $1.3$ ppts for Asian registrants, with White RMSE essentially unchanged. Similarly, $\ell$BIFSG improves upon $\ell$BISG for three of the four groups in Florida, by between $0.3$ and $1.1$ ppts. For Black registrants in North Carolina, $\ell$BIFSG outperforms standard BIFSG by $3.8$ ppts. For bias, the addition of first names moves one of the four groups in Florida and two in North Carolina farther from zero, and Black absolute bias decreases in both states, from $2.3$ to $0.8$ ppts in Florida and from $2.6$ to $1.4$ ppts in North Carolina. Standard BIFSG likewise improves upon standard BISG in RMSE for five of the eight state-group pairs, the exceptions being Black registrants in both states and White registrants in North Carolina.

We also repeat the downstream analysis using the Census lists (see Figure~\ref{fig:val-disparity-census} in the Appendix). The results are similar to those based on the LLM lists shown here, with a few moderate differences. For both absolute bias and RMSE, the Census lists yield slightly lower or essentially identical results for all groups in Florida and North Carolina -- except for Asian registrants in North Carolina, who have lower RMSE but higher bias when using the census lists. 

Overall, we find that $\ell$BISG is comparable to standard BISG in accuracy, calibration, and downstream analysis validity, while first names improve performance further through $\ell$BIFSG. Group-specific performance results depend on measurable group-specific list quality. With respect to the LLM lists, the Asian list is weak in precision, and for this group, the downstream analysis based on $\ell$BISG is less accurate than that based on standard BISG.

As with eBISG \citep{DasanaikeImai2026}, which extends standard BISG to names that are absent from the Census surname table using embedding models, researchers can combine $\ell$BISG with standard BISG by using standard BISG for names that are listed in the frequency table(s) and $\ell$BISG for imputing probabilities for names that are not, to which standard BISG otherwise assigns the geographic prior. In Appendix~\ref{app-combined}, we report the results of standard BISG and BIFSG with $\ell$BISG probabilities for unlisted names in the Florida and North Carolina voter files. For surnames, using $\ell$BISG for unlisted names improves average precision for Hispanic registrants in both states and leaves the other groups largely unchanged. For first names, this approach performs comparably to standard BIFSG across all registrants, but has a lower average precision among registrants whose first name is missing from the corresponding Census first name table. 

\subsection{Predicting Asian subgroups using U.S. historical census}\label{sec-val-asian}

We next turn to the U.S. historical census, which provides population-scale, individual-level records decennially from 1850 to 1950 \citep{ipumsfullcount, Nelson2025}. We focus on the 1900 census, which includes, among other information, individual name, race, and country of origin, and enumerates Chinese and Japanese as separate ethnic groups. We validate $\ell$BISG estimates against these two groups as enumerated in the race field of each individual record, regardless of birthplace. 

We consider the challenging setting of Section~\ref{sec-subgroupextension}, in which the geographic prevalence of the target groups (i.e., Chinese and Japanese) is unknown while that of their coarser group (i.e., Asian) is available. This is the setting where name lists are required because we recover the subgroup geographic prevalence from the rates at which the people in a given county have surnames on each list. We do not assume that it is known who is in the Asian group, so we also include a group for each of the other races. This validation set-up mimics a relevant contemporary $\ell$BISG application: Standard BISG is not applicable for more recent Censuses  because neither surname frequency tables nor geographic prevalence information exists (or has been made public) for Asian sub-groups.

\paragraph{Setup.}
We use the de-anonymized (restricted access) 1900 Federal Census. By 1900, approximately $126{,}000$ Chinese and $94{,}000$ Japanese Americans were enumerated living in the U.S., concentrated in the West and in Hawaii. No other Asian group numbered more than $1{,}000$. We therefore use the Chinese and Japanese populations as the two target subgroups, and the Asian population coarse group as the union of these two. The published tables of the 1900 census report the size of the Asian population in every county but do not decompose this by country of origin. Therefore, if we were using only public information, we would know the geographic prevalence of the coarse group but not of the subgroups, which corresponds exactly to the setting described in Section~\ref{sec-subgroupextension}. We apply $\ell$BISG to all $75{,}790{,}807$ records of the census, of every race. The groups are Chinese and Japanese, within the Asian coarse group, the other coarse racial groups are white, Black, Native American (``American Indian''), and a residual group. Like today's Census, the 1900 Census published tables reporting the size of each coarse group in every county. 

We conduct two analyses. First, we consider ``oracle $\ell$BISG'' as a benchmark.  Specifically, we treat the county geography of the two Asian subgroups as known by constructing the geographic prior from the census records themselves and apply the $\ell$BISG method. Second, we conduct a more realistic ``coarse prior $\ell$BISG'' analysis from Section \ref{sec-recovergeography}, where we treat only the geographic prevalence of the coarse groups as known (the Asian, white, Black, American Indian, and other residents of each county). Here, we recover the subgroup shares in each county from the list hit rates by applying the method described in Section~\ref{sec-subgroupextension} under Assumption~\ref{a-listqual}. Then, we apply the main $\ell$BISG method to an LLM-based name list for each group. Finally, for both analyses, we conduct a downstream analysis to estimate each subgroup's socioeconomic standing at the same county level.

A key additional assumption for coarse prior $\ell$BISG is that a person has a surname on a subgroup's list only if they belong to that subgroup (Assumption~\ref{a-listqual}(b)). Chinese and Japanese surnames largely do not overlap with each other (only five surnames appear on both lists), but many white Americans have a surname on the Chinese list, so this assumption does not hold exactly here. The validation against oracle $\ell$BISG gives us the opportunity to see how much this matters.

For both versions of $\ell$BISG, we construct a name list for each subgroup by prompting an LLM. As in Section~\ref{sec-val-voterfiles}, we also use an LLM, Gemini~3.0 Flash, to assist in writing these prompts. For each subgroup, we ask the model to propose three short descriptions that identify subgroups within that population, then substitute these descriptions into the same template, which specifies surnames ``spelled as they were written in American census and immigration records of that period.'' For example, for the Chinese population, the descriptions are ``Cantonese families from the Pearl River Delta settling in California and the American West,'' ``Taishanese and Sze Yup laborers migrating to the sugar plantations of the Hawaiian Islands,'' and ``Fujianese and Hokkien merchants and sailors establishing communities in Atlantic port cities like New York.''  We form lists for each of the white, Black, and American Indian groups in the same way, with each prompt referring to its population ``in the United States around 1900''; the residual group of other residents has no list. Appendix~\ref{app-prompts} lists the descriptions generated for every group.

For each prompt, we make six separate model calls at temperature $1.0$, requesting $300$ surnames per call, rank the unique surnames by the number of calls that returned them, and retain the top $1{,}000$, breaking ties by the order in which the model first returned each surname. To make Assumption~\ref{a-listqual}(b) as credible as possible, we then apply the exclusivity filter between the Chinese and Japanese lists, dropping the five surnames that appear on both, and between either of them and the white, Black, and American Indian lists. These three lists are allowed to share surnames with one another because the geographic prevalence of their groups is known and does not need to be recovered. 

After filtering, the Chinese list holds $966$ surnames, the Japanese list $994$, and the white, Black, and American Indian lists $986$, $986$, and $978$. Even so, when comparing with the individual-level ground truth, we find that $81\%$ of the people with a surname on the Chinese list and $29\%$ of those with a surname on the Japanese list are not Asian. Examples of Japanese surnames that overlap with white Americans includes Ohara, with Irish immigrants, and Ide, with English and German immigrants. Various Cantonese romanizations such as Fay, Law, and Hay coincide with common English, Scottish, and German surnames. 

Because the white population is more than $300$ times larger than the Asian population at this time, however, only $0.4\%$ of the people who are not Asian have a surname on the Chinese list, and $0.01\%$ a surname on the Japanese list. Within the Asian population, $0.9\%$ of Japanese records have a surname on the Chinese list, and $0.05\%$ of Chinese records a surname on the Japanese list. Crucially, these numbers mean that the relevant cross probabilities $\Pr(\listmem_{\targlow}=1\mid \targ=\targlow')$, which we treat as $0$ for the method in Section~\ref{sec-subgroupextension}, are in fact small. The low precision $\Pr(\targ=\targlow\mid \listmem_{\targlow}=1)$ plays no direct role in our method.

The remaining steps follow Algorithm~\ref{alg-lbisg}. We embed surnames using the multilingual E5-large model and estimate $\hat{\bm f}$ in the same way as done in Section~\ref{sec-val-voterfiles}. We apply Algorithm~\ref{alg-selectK} to a random sample of $5{,}000{,}000$ records, which returns $K=67{,}500$ under both priors, and then apply Algorithm~\ref{alg-lbisg} to all records.

\begin{table}[!t]\centering
\begin{tabular}{lrrrrrrr}
\hline
 & & \multicolumn{3}{c}{Coverage} & \multicolumn{3}{c}{Precision} \\
List & Records & Truth & Oracle & Coarse & Truth & Oracle & Coarse \\
\hline
Chinese & $114{,}090$ & $0.55$ & $0.65$ & $0.51$ & $0.19$ & $0.21$ & $0.17$ \\
Japanese & $87{,}278$ & $0.26$ & $0.27$ & $0.24$ & $0.71$ & $0.72$ & $0.67$ \\
White & $66{,}834{,}064$ & $0.33$ & $0.32$ & $0.32$ & $0.83$ & $0.80$ & $0.80$ \\
Black & $8{,}717{,}147$ & $0.59$ & $0.74$ & $0.74$ & $0.20$ & $0.25$ & $0.25$ \\
American Indian & $34{,}207$ & $0.31$ & $0.22$ & $0.21$ & $0.00$ & $0.00$ & $0.00$ \\
\hline
\end{tabular}
\caption{Quality of the five LLM lists on the full-count 1900 census. ``Records'' counts the true number of people enumerated in each group. ``Coverage'' is the share of those records with a surname on the group's list, and ``Precision'' the share of all people with a surname on the list who belong to the group. ``Truth'' is computed from the enumerated race and ethnicity. ``Oracle'' is estimated by the method of Section~\ref{sec-listqualitymeas} with the true county prevalence of every group. ``Coarse'' uses only the coarse geography: for the Chinese and Japanese lists, coverage is obtained from the county-level list hit rates and the county Asian share by the method of Section~\ref{sec-recovergeography}, and precision is the estimated coverage times the recovered number of subgroup members, as a share of the people holding a list surname; for the other lists, whose county prevalence is known, it equals the method of Section~\ref{sec-listqualitymeas}.}
\label{tab:asian}
\end{table}

\paragraph{Quality of name lists.} Table~\ref{tab:asian} reports the true coverage and precision for each list computed from the enumerated ethnicity. The Chinese list covers $0.55$ of Chinese records with a precision of $0.19$, and the Japanese list covers $0.26$ of Japanese records with a precision of $0.71$. The precision of the Chinese list is low because most people holding a surname on it are the white Americans described above. The American Indian list has almost no precision because nearly every person holding one of its surnames is white or Black. With the true geographic prevalence of every group, the method of Section~\ref{sec-listqualitymeas} estimates the coverage of every list within $0.15$ of its ground truth and the precision within $0.05$. 

For the coarse prior $\ell$BISG, we cannot estimate the full $\Pimat$ when sub-group geography is unknown, but the subgroup recovery method of Section~\ref{sec-recovergeography} does, under the assumption that $\Pimat$ is diagonal, produce an estimate of the coverage of each list based on the county-level list hit rates and the coarse geography alone. Here, we find that this coverage estimate is within $0.04$ of its ground truth for both lists, and the recovered county subgroup shares correlate with the truth at $0.96$, weighting counties by their Asian population. The estimated precision is within $0.04$ of its ground truth for both lists. 

\begin{figure}[!t]\centering
\includegraphics[width=0.75\textwidth]{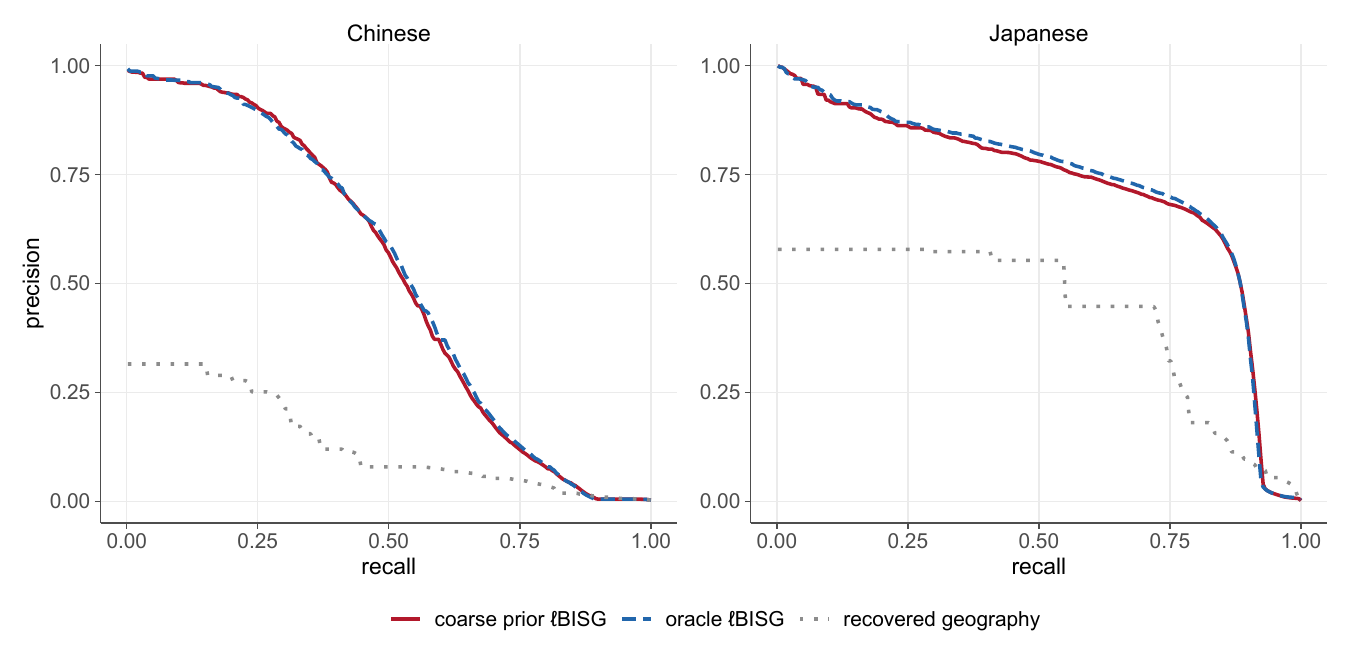}\\[6pt]
\includegraphics[width=0.75\textwidth]{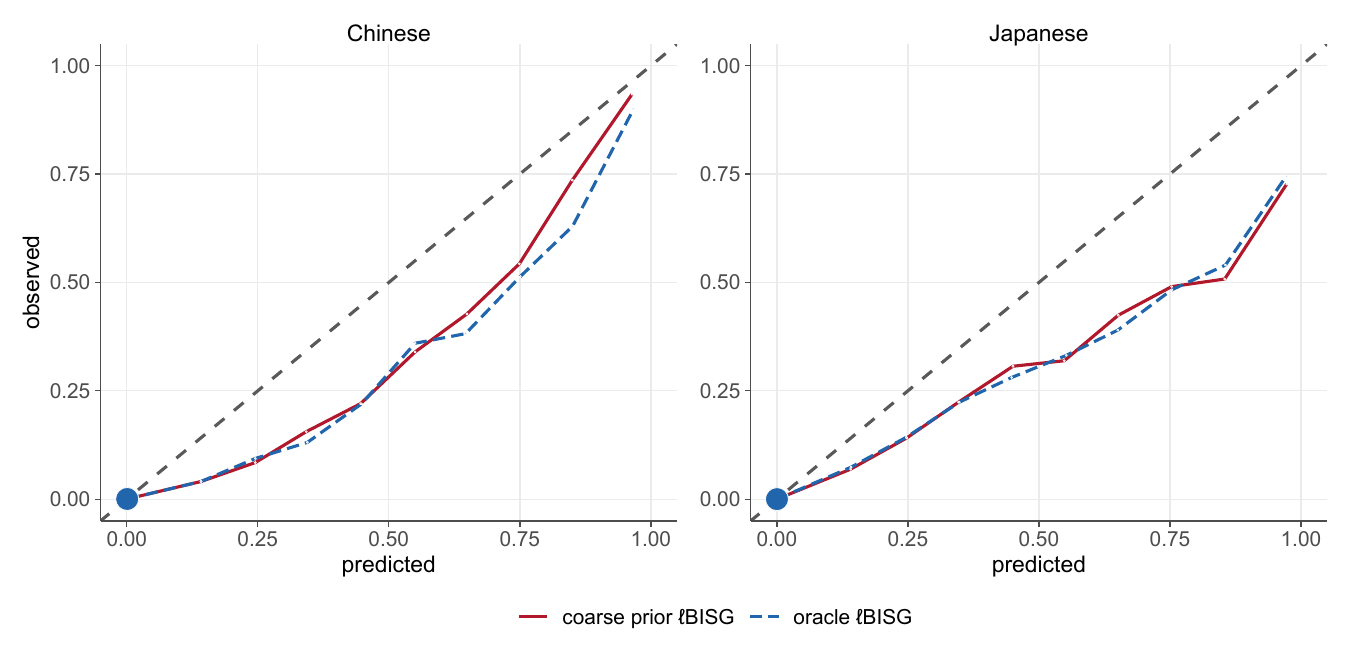}
\caption{Precision and recall (top) and calibration (bottom) over all records of the census for coarse prior $\ell$BISG, which recovers the subgroup geography from the coarse geography and the lists by the method of Section~\ref{sec-subgroupextension} (red, solid); oracle $\ell$BISG, which uses the true subgroup geography (blue, dashed); and, in the top panel, recovered geography, the county subgroup shares (grey, dotted), by subgroup. In the bottom panel, the dashed diagonal marks perfect agreement between estimated probabilities and observed proportions. The area of each circle is proportional to the number of records in the bin.}
\label{fig:val-asian}
\end{figure}

\paragraph{Results.} We compare the empirical performance of coarse prior $\ell$BISG against oracle $\ell$BISG, as well as against using the recovered geography alone. As shown in Figure~\ref{fig:val-asian}, coarse prior $\ell$BISG (red solid line) approaches oracle $\ell$BISG (blue dashed line): average precision over all records is $0.52$ against $0.52$ for Chinese records and $0.72$ against $0.73$ for Japanese records, whereas recovered geography yields only $0.12$ and $0.37$, respectively. Among the Asian records alone, average precision is $0.88$ against $0.89$ for Chinese records and $0.88$ against $0.90$ for Japanese records.

Figure~\ref{fig:val-asian} shows that, under both $\ell$BISG and its oracle version, the two subgroups are over-predicted at high probabilities: among the records assigned a probability near $0.97$ of being Japanese, $0.73$ are Japanese. The expected calibration error over all records is about $0.001$ for both $\ell$BISG and its oracle version because nearly every record receives a probability near zero.

\begin{figure}[!t]\centering
\includegraphics[width=\textwidth]{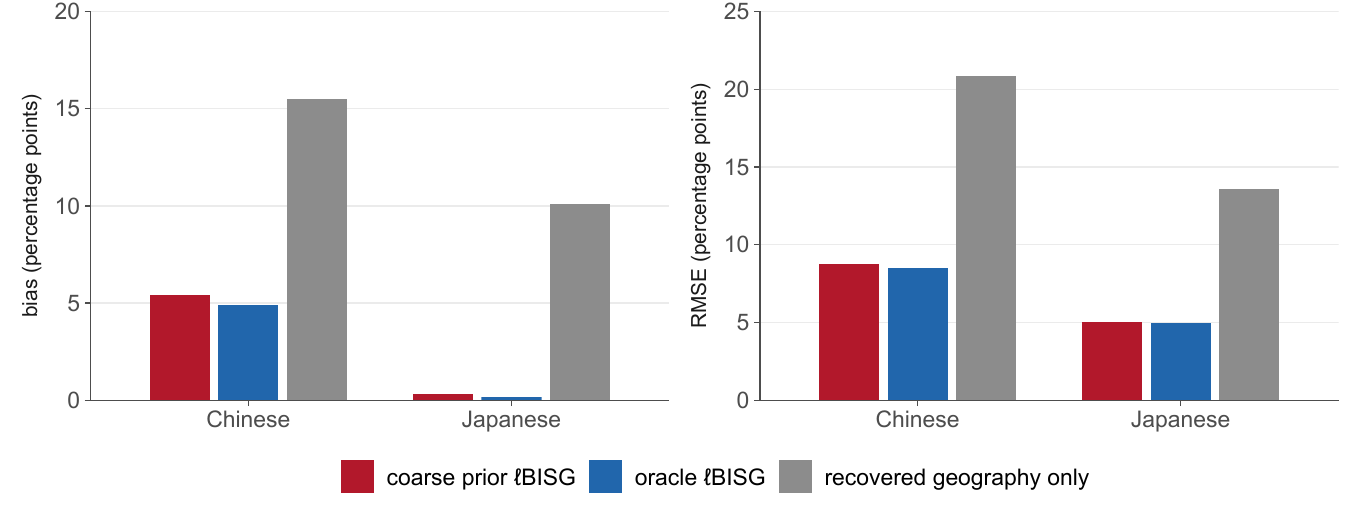}
\caption{County-level subgroup-specific estimates of the share in the top quartile of the 1900 socioeconomic index. Panels give signed bias (left) and root-mean-square error (right), both in percentage points. Red bars are coarse prior $\ell$BISG, blue bars oracle $\ell$BISG, and grey bars recovered geography only.}
\label{fig:asian-birdie}
\end{figure}

We also evaluate the validity of downstream analysis based on the probabilities obtained using $\ell$BISG. We estimate the share of each subgroup in the top quartile of the 1900 socioeconomic index within each county, among the $24{,}235{,}916$ records of every race in the labor force with a recorded index value. The index is the Duncan socioeconomic index provided by IPUMS \citep{ipumsfullcount}, which scores each occupation by the income and education of its members in 1950. We fit a single jointly estimated BIRDiE model over all six groups with complete pooling \citep{McCartan02102025}, and score every county, weighting each county's error by the subgroup's expected count under the recovered geographic prior. We compare $\ell$BISG against oracle $\ell$BISG and geography alone.

Figure~\ref{fig:asian-birdie} compares the estimated top-quartile share of each subgroup at the county level against the truth. The bias of $\ell$BISG is close to that of oracle $\ell$BISG: $5.4$ against $4.9$ for Chinese records and $0.3$ against $0.2$ for Japanese, whereas using recovered geography alone yields $15.5$ and $10.1$ ppts. The county RMSE of $\ell$BISG is $8.8$ against $8.5$ ppts for Chinese records and $5.0$ against $5.0$ for Japanese, against $20.9$ and $13.6$ for recovered geography only. Overall, we find that coarse prior $\ell$BISG approaches oracle $\ell$BISG in accuracy, calibration, and downstream analysis validity. Thus, this validation exercise demonstrates the applicability of $\ell$BISG to subgroups for which neither a surname table nor a geographic prior exists.

\subsection{Predicting religious sects in Lebanon using voter files}\label{sec-val-lebanon}

Finally, we turn to Lebanon and apply $\ell$BISG outside the United States to predict individuals' religious sect memberships rather than their racial and ethnic identities.

\paragraph{Setup.} We use full-count voter files, containing $3{,}676{,}322$ registrants, which were published by the Lebanese Ministry of Interior and Municipalities through the Directorate General of Civil Status on 1 February 2018. We drop the $14{,}388$ voters whose sectarian memberships are missing. After also requiring usable names and geography, the file contains $3{,}661{,}930$ registrants. Each file contains the full name, location of residence, and self-reported religious sect of each registrant \citep{Dasanaike2026zeroshot}. We consider the seven largest sects, together comprising $96.6\%$ of all registrants, and collapse the remaining $125{,}534$ registrants into the residual group. These seven sects are the Shia, at $29.9\%$ of all registrants, Sunni ($27.1\%$), Maronite ($19.8\%$), Roman Orthodox ($7.0\%$), Druze ($5.5\%$), Roman Catholic ($4.9\%$), and Armenian Orthodox ($2.4\%$). No other sect comprises more than $1\%$ of the roll.

Surnames are more informative about religious sects in Lebanon than about racial groups in the United States because families tend to be confessionally endogamous. However, that informativeness is unevenly distributed. The least distinctive names belong to the smaller Christian sects, whose surnames come from a Levantine Arabic ancestry shared with larger neighbors. 

Unfortunately, the last census in Lebanon was conducted in 1932 under French supervision, so the geographic distribution of sect from external sources is not available. Therefore, for the purposes of this validation exercise, we form geographic priors using the voter file itself. Specifically for each voter, we calculate the share of each sect among the registrants of the voter's ``locality", which is a formal administrative term referring to towns and villages, aggregated from the self-reported sects in the registry itself (there are $1{,}550$ localities). 

Our downstream analysis uses locality-level $\ell$BISG predictions to estimate the age composition of each sect at the district level ($n = 29$). The age composition of Lebanon's sectarian groups is politically consequential because political representation is explicitly allocated along sectarian lines. Based on the 1932 census, which found Christians to constitute 51\% of the population, the 1943 National Pact allocated parliamentary seats to Christians and Muslims in a 55--45 ratio. Following fifteen years of sectarian civil war, the 1989 Taif Agreement revised this allocation to an equal 50--50 split. By that time, however, unofficial estimates placed the Muslim share of the population at roughly 70\% \citep{faour2007religion}.

Because younger cohorts are disproportionately Muslim, a proposal to lower the voting age from $21$ to $18$, endorsed by parliament in 2009, has stalled amid concerns that expanding the electorate to younger voters would place Christian confessions at a substantial numerical disadvantage \citep{faour2007religion}. Although no census has been conducted since 1932, in part to avoid reigniting sectarian conflict over the demographic balance, sectarian affiliation continues to be recorded in the electoral rolls because civil status in Lebanon is administered by confession and parliamentary seats are allocated by sect.

Without a published surname table, we construct lists using an approach similar to the one employed in Section~\ref{sec-val-asian}.  Specifically, we use Gemini~3.0 Flash to propose prompts written in Arabic (see Appendix~\ref{app-prompts} for details). As done in the previous applications, we make six separate model calls per prompt at temperature $1.0$, rank the surnames by the number of calls that returned them, and retain up to $1{,}000$ unique surnames for each sect. The Maronite and Druze lists contain $861$ and $853$ surnames, respectively; the other five contain $1{,}000$.

We embed surnames using the paraphrase-multilingual-mpnet-base-v2 model \citep{reimers-gurevych-2020-making}, which accommodates Arabic script, instead of the E5 encoder used in the previous applications. In Step~4 of Algorithm~\ref{alg-lbisg}, the list-membership models, one per sect list, are fit on all registrants. The models score each surname out-of-sample across five folds that randomly partition the distinct surnames, with the number of epochs selected by cross-validation within each training set. The remaining steps follow Algorithm~\ref{alg-lbisg}, with $K=57{,}567$ clusters chosen by applying Algorithm~\ref{alg-selectK} to all $3{,}661{,}930$ registrants.\footnote{We use the candidate grid ranging from $10$ to $57{,}567$ clusters.}

\begin{table}[!t]\centering
\begin{tabular}{lrrrrr}
\hline
 & & \multicolumn{2}{c}{Coverage} & \multicolumn{2}{c}{Precision} \\
Sect & Registrants & Ground truth & Estimate & Ground truth & Estimate \\
\hline
Shia & $1{,}096{,}570$ & $0.72$ & $0.72$ & $0.48$ & $0.48$ \\
Sunni & $991{,}158$ & $0.56$ & $0.58$ & $0.35$ & $0.36$ \\
Maronite & $724{,}231$ & $0.75$ & $0.77$ & $0.37$ & $0.38$ \\
Roman Orthodox & $254{,}712$ & $0.64$ & $0.68$ & $0.11$ & $0.11$ \\
Druze & $201{,}756$ & $0.80$ & $0.83$ & $0.11$ & $0.12$ \\
Roman Catholic & $180{,}195$ & $0.61$ & $0.68$ & $0.08$ & $0.08$ \\
Armenian Orthodox & $87{,}774$ & $0.24$ & $0.29$ & $0.77$ & $0.93$ \\
\hline
\end{tabular}
\caption{Quality of the seven LLM sect lists on the Lebanese voter registry. ``Coverage'' is the share of a sect's registrants holding a surname on that sect's list, and ``Precision'' is the share of the registrants holding a list surname who belong to the sect. Lists contain up to $1{,}000$ unique surnames.}
\label{tab:sects}
\end{table}

\paragraph{Quality of name lists.} We evaluate the quality of LLM sect lists using an approach similar to the one used in the previous applications.  Table~\ref{tab:sects} reports coverage and precision for each of the seven sects, both as computed from the recorded sects and as estimated from the locality system, using the locality sect shares as the geographic prior and no individual labels (Section~\ref{sec-listqualitymeas}). The estimates align closely to their corresponding ground-truths with the only exception being the estimated precision for Armenian Orthodox, which overestimates precision by 16 ppt.

Across the groups, the Armenian lists yield the lowest coverage but the highest precision, while the Druze and Roman Catholic lists hold surnames that cover most registrants of those sects as well as many non-members. The sect lists also overlap more often than any of the American lists: $57.1\%$ of registrants have a surname that belongs to more than one list.
Our overall list quality measure is $D(\Pimat) = 0.36$, which is similar to its ground truth $0.33$ based on individual sectarian memberships, suggesting that the Lebanese lists are similar in quality to the American lists. 

\paragraph{Results.} We compare the empirical performance of $\ell$BISG based on the LLM lists against that of geography alone prediction as a baseline. Aside from constructing the geographic prior, the sect recorded in the voter file is used only for evaluation. As shown in Figure~\ref{fig:val-lebanon-pr}, locality alone (grey dots) already predicts sect well because most localities are confessionally homogeneous. Still, $\ell$BISG (red solid line) improves on locality-alone prediction for every sect, though only marginally for the Shia, and by the most for the Armenian Orthodox, whose average precision rises from $0.48$ under geography alone to $0.73$; the Armenian Orthodox list is the most precise of the seven (Table~\ref{tab:sects}).

\begin{figure}[!t]\centering
\includegraphics[width=\textwidth]{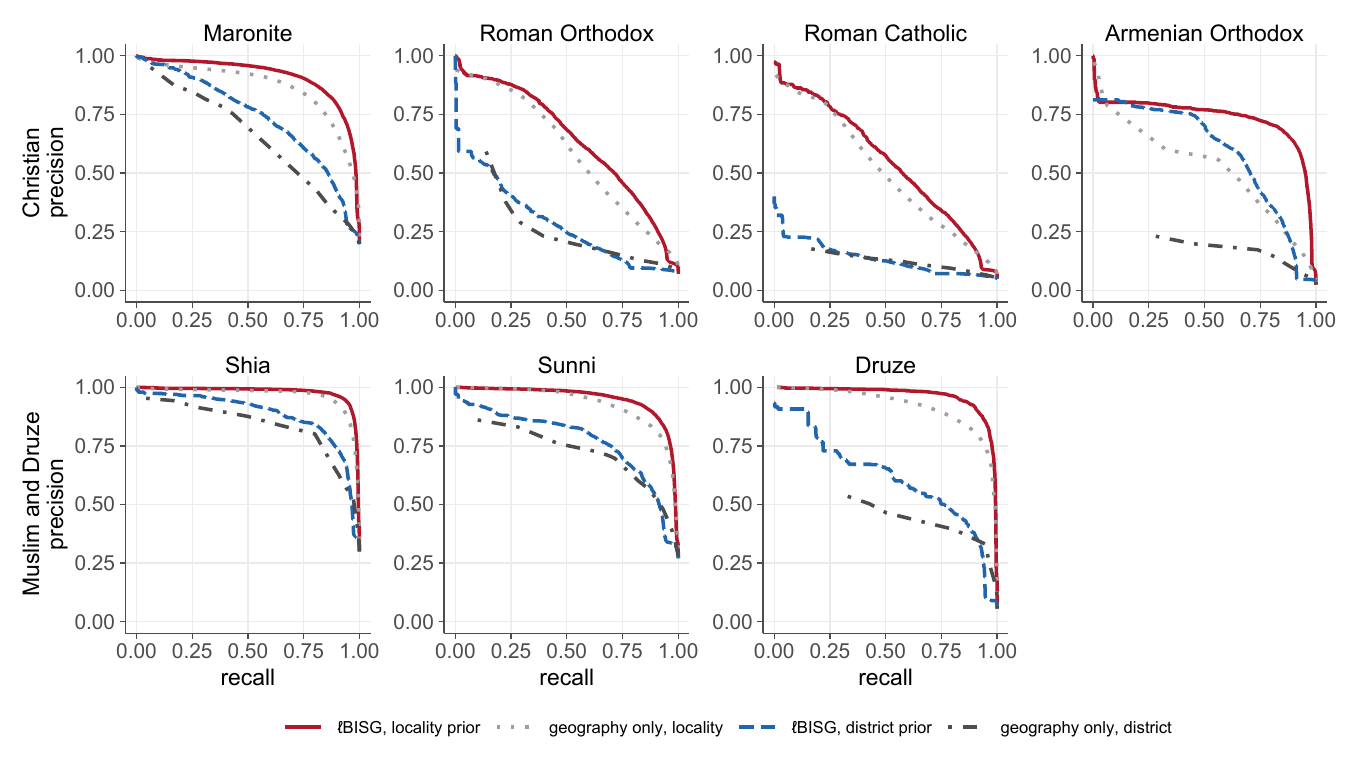}
\caption{Precision and recall for $\ell$BISG and geography alone for the seven principal Lebanese sects, by confession (rows) and sect (columns), under the locality-level prior ($n = 1{,}550$; $\ell$BISG red solid, geography grey dotted) and the district-level prior ($n = 29$; $\ell$BISG blue dashed, geography dark grey dot-dashed).}
\label{fig:val-lebanon-pr}
\end{figure}

The Muslim sects and the Druze are predicted best, with average precision above $0.94$, and the smaller Christian sects worst, at $0.55$ for Roman Catholics and $0.64$ for Roman Orthodox because their surnames are shared with the larger Christian sects. A low-precision list still supports precise predictions. The Roman Catholic list has a precision of only $0.08$ at a recall equal to its coverage of $0.61$ (Table~\ref{tab:sects}), whereas the $\ell$BISG posterior reaches a precision of $0.47$ at the same recall. This is because the posterior also draws on geography and on the other six lists.

Moreover, if we only look at the localities where no single sect comprises more than 90\% of the roll (43\% of $1,550$ localities), $\ell$BISG outperforms geography alone by between $0.05$ (Roman Catholic) and $0.26$ (Armenian Orthodox) in average precision, and by $0.10$ to $0.14$ for the three largest sects.

If we use the more heterogeneous district-level ($n=29$) share as the geographic prior instead of the locality, $\ell$BISG (blue dashed line) outperforms district-alone prediction (dark grey dot-dashed line) by between $0.01$ (Roman Catholic) and $0.41$ (Armenian Orthodox) in average precision. However, the downstream age estimates in Figure~\ref{fig:lebanon-cohort} become less accurate when probabilities are calculated at the district level, with the RMSE rising for six of the seven sects. This is because the coarser prior discards within-district geographic information.

\begin{figure}[!t]\centering
\includegraphics[width=\textwidth]{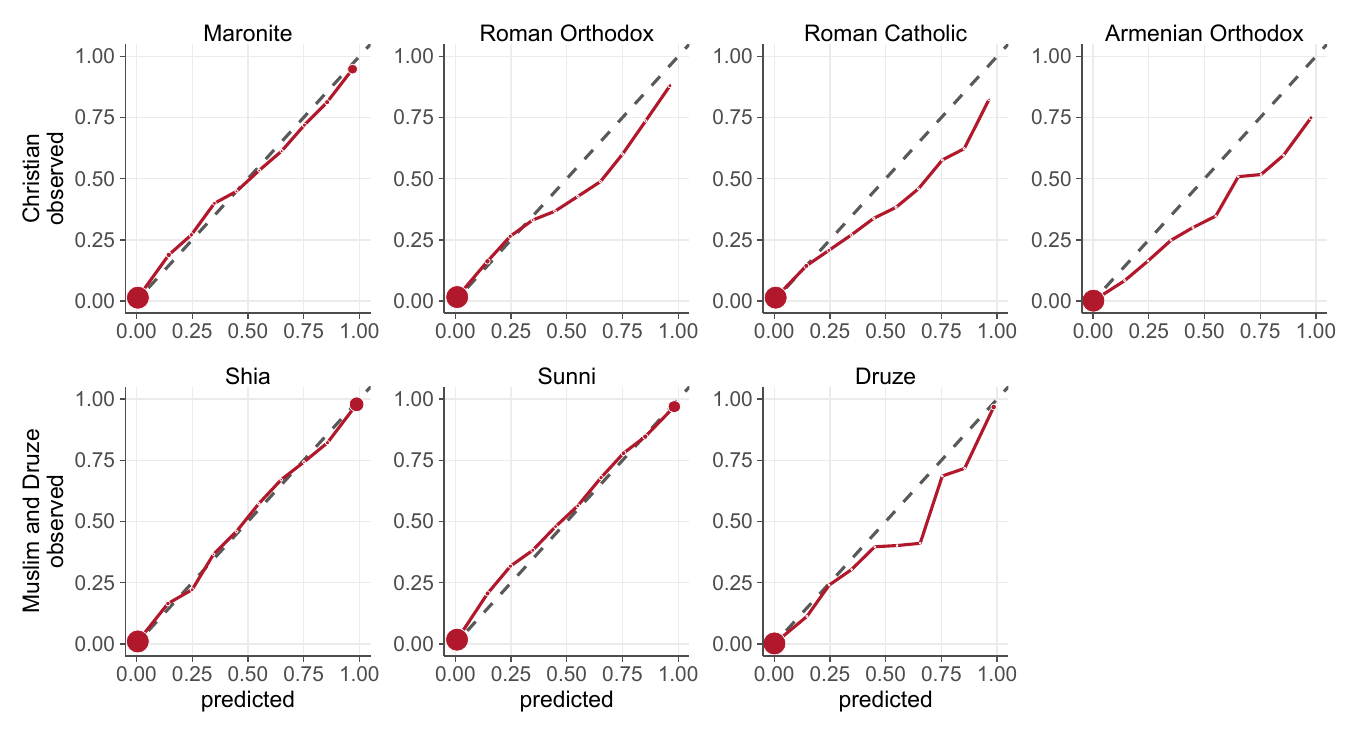}
\caption{Calibration of estimated individual sect probabilities based on $\ell$BISG (LLM lists) for the seven principal Lebanese sects, by confession (rows) and sect (columns). The dashed diagonal marks perfect agreement between estimated probabilities and observed proportions. The area of each circle is proportional to the number of registrants in the bin.}
\label{fig:val-lebanon-calibration}
\end{figure}

Figure~\ref{fig:val-lebanon-calibration} shows that the predicted individual sect probabilities based on $\ell$BISG are well calibrated for the Maronites, the Muslim sects, and the Druze, whose reliability curves lie close to the diagonal. The three smaller Christian sects are over-predicted at high probabilities. Among registrants assigned a probability near $0.98$ of being Armenian Orthodox, $0.75$ are; most of the remainder are Armenian Catholics and Armenian Protestants in the residual class, since the Armenian Orthodox list identifies Armenian surnames rather than the denomination. The expected calibration error is nevertheless low for every sect, from $0.004$ for the Druze to $0.015$ for the Roman Catholics, because most registrants receive a probability near zero for any given sect.

We also evaluate the validity of the downstream analysis of age based on the estimated $\ell$BISG individual sect probabilities under the locality-level prior. The estimand is the share of each sect's members who are young adults (born in $1985$ or later) within each district, which is a politically salient statistic in Lebanon. This share varies across the sects, ranging from $12.9\%$ among Armenian Orthodox to $28.4\%$ among Shia. We estimate it for each sect within districts using separate target-versus-rest BIRDiE models with a saturated district specification, among the $90.2\%$ of registrants with a recorded birth year. We then compute the bias and RMSE over every sect-by-district pair, weighting each pair by its expected number of sect members under the geographic prior. We evaluate at the district level rather than at the locality level of the prior because most localities are confessionally homogeneous, such that geography alone already recovers the age composition of each sect within most localities. Again, we compare these estimates against using geography alone.

\begin{figure}[t]\centering
\includegraphics[width=\textwidth]{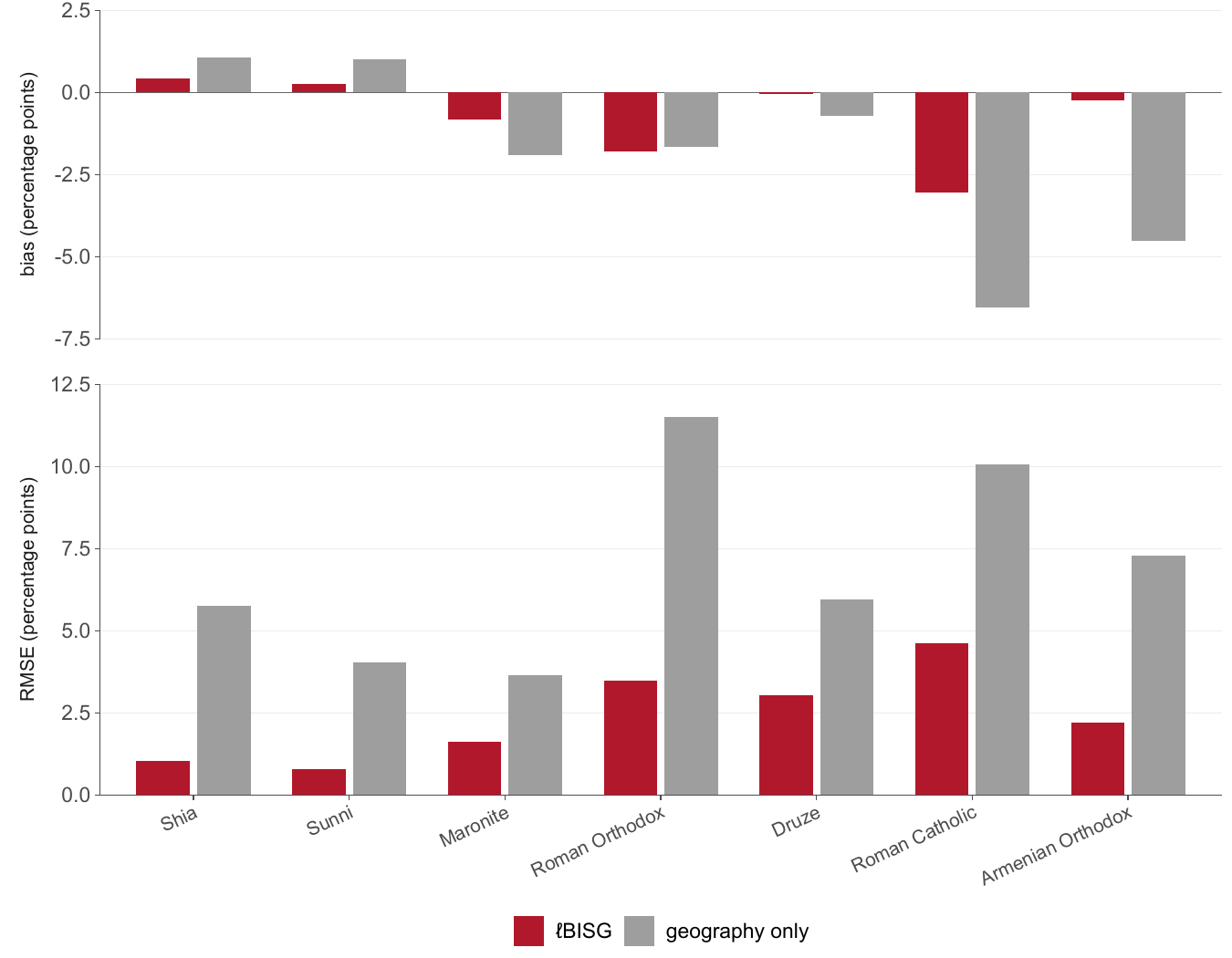}
\caption{Share of each sect born in $1985$ or later, estimated by district. Rows give signed bias
and root-mean-square error against the enumerated sect, both in ppts; red bars are $\ell$BISG
and grey bars geography alone. Sects are ordered by their share of the roll.}
\label{fig:lebanon-cohort}
\end{figure}

Figure~\ref{fig:lebanon-cohort} shows that across the $202$ sect-by-district pairs, $\ell$BISG reduces the overall RMSE of the geographic baseline by about $68\%$, from $5.9$ to $1.9$ ppts. RMSE (bottom panel) is lower for all seven sects, and absolute bias (top panel) is lower for six.

The size of the reduction reflects two separate components: how well geography alone recovers the sect due to residential segregation, and how well surnames identify group members. The Roman Catholics, whose members are older and dispersed as a minority across districts, have a baseline RMSE of $10.1$ ppts, of which $\ell$BISG removes $54\%$. The largest proportional reductions are for the Shia, from $5.8$ to $1.1$ ppts ($82\%$), and the Sunni, from $4.1$ to $0.8$ ppts ($80\%$).

Overall, we find that $\ell$BISG predicts sectarian memberships with average precision between $0.55$ and $0.98$, improves on geography alone for every sect -- though the extent depends on the precision of the initial geographic probabilities -- and reduces the error of the geographic baseline on the district-level estimand by about $68\%$. This validation demonstrates the potential application of $\ell$BISG outside of the United States, to non-English names, and to a non-traditional group membership for which surnames are informative.

\section{Conclusion}

Bayesian Improved Surname Geocoding (BISG) and its variants have become the standard method for predicting individual race in the U.S.  However, the requirement of race-name frequency tables has severely limited the applicability of BISG beyond this specific context.  Often, these methods have been inapplicable beyond the coarse racial categories defined by the Census Bureau and outside the U.S., where such information is rarely available. 

In this paper, we overcome this difficulty by developing list-powered BISG ($\ell$BISG), which only requires group-specific name lists that can easily be obtained from an LLM. Importantly, when geographic prevalence information is available, the only new assumption is a rank condition on geographic prevalence that will often be satisfied if the groups of interest have at least some geographic variation. That is, $\ell$BISG only imposes this single additional assumption beyond what is required for the standard BISG method.   Moreover, by replacing the traditionally assumed independence of surname and geography given race with an assumed surname \textit{cluster} and geography independence given race, $\ell$BISG actually weakens the foundational BISG assumption.

$\ell$BISG works by using a previously unrecognized consequence of the standard BISG assumption that surname and geography are independent given group membership. If a variable $B$ predictive of group membership (such as surname, surname embeddings, or embedding clusters) is not directly related to geography conditional on group membership (and if group membership has a geographic pattern), then $B$ will also exhibit a geographic pattern that indirectly reveals how $B$ relates to group membership. By setting up a system of equations, we are able to recover that relationship, which can in turn be used to produce BISG probabilities. 

Building on our earlier work of embedding-powered BISG \citep[eBISG;][]{DasanaikeImai2026}, $\ell$BISG uses a proxy task of predicting list membership from name embeddings followed by clustering of those predictions to construct such a variable $B$. Under additional assumptions on list quality, we have also shown how to extend $\ell$BISG to settings in which prior information about subgroup geographic prevalence is not available.  

Our empirical validations show that the performance of $\ell$BISG is comparable to the standard BISG in terms of precision, accuracy, calibration, and downstream estimation performance --- more so when lists are of higher quality. This is even though the former only requires group specific name lists. The case where geographic priors are unavailable is much more challenging, but when lists are distinctive enough, we are able to achieve comparable performance to standard BISG. Ultimately, $\ell$BISG represents a substantial improvement over using a geographic prior alone when standard BISG is not applicable. In Appendix~\ref{app-llm-direct}, we also compare $\ell$BISG with directly prompting an LLM for the probability of each group given a surname \citep{Dasanaike2026zeroshot}. We find that the downstream estimates based on $\ell$BISG are more accurate for most groups, and that $\ell$BISG is better calibrated in the 1900 census and for most sects in Lebanon.

Because one can use a large-language model to generate an informative name list essentially for any group, $\ell$BISG extends probabilistic inference to fine-grained ethnic groups, to settings outside the United States, and to subgroups within a coarse racial group.  We expect $\ell$BISG to considerably widen the applicability of probabilistic prediction methods that estimate group disparities when their membership is not observed. 

Future research could improve $\ell$BISG in several directions. One direction is to extend our approach to recovering unavailable subgroup geographic prevalence to settings where additional prior information about list quality and overlap is available. Another is to explore alternatives to our combination of list-based prediction and K-means clustering for constructing the low-dimensional summary variable. A third is to further explore list quality and the trade-offs between selecting distinctive, common, group-exclusive, and group-covering names.

As we noted in Section~\ref{sec-keymethod}, in the settings where geographic priors are already available, the $\ell$BISG framework can be applied effectively without lists so long as embeddings are spatially clustered by the group of interest. However, when geographic information is not available, and for diagnostics, lists are essential. In some of our applications, we find that list-guided clustering only moderately improves on direct clustering. Still, it remains an open question how the added value of lists varies across application types and embedding models.

Broadly, this paper offers a general strategy for incorporating information from (LLM-generated) lists and name embeddings into well-calibrated probabilities that support valid downstream inference while accounting for uncertainty in unobserved group membership.

\newpage
\bibliography{references}

@article{kaziAoual1995refined,
  author  = {Kazi-Aoual, Fr{\'e}d{\'e}rique and Hitier, Simon and
             Sabatier, Robert and Lebreton, Jean-Dominique},
  title   = {Refined Approximations to Permutation Tests for Multivariate Inference},
  journal = {Computational Statistics \& Data Analysis},
  year    = {1995},
  volume  = {20},
  number  = {6},
  pages   = {643--656},
  doi     = {10.1016/0167-9473(94)00064-2}
}

@inproceedings{chen2019fairness,
  title     = {Fairness Under Unawareness: Assessing Disparity When Protected Class Is Unobserved},
  author    = {Chen, Jiahao and Kallus, Nathan and Mao, Xiaojie and Svacha, Geoff and Udell, Madeleine},
  booktitle = {Proceedings of the Conference on Fairness, Accountability, and Transparency},
  pages     = {339--348},
  year      = {2019}
}

@article{KurokiPearl2014,
 ISSN = {00063444},
 URL = {http://www.jstor.org/stable/43305623},
 author = {Manabu Kuroki and Judea Pearl},
 journal = {Biometrika},
 number = {2},
 pages = {423--437},
 publisher = {[Oxford University Press, Biometrika Trust]},
 title = {Measurement bias and effect restoration in causal inference},
 urldate = {2023-04-04},
 volume = {101},
 year = {2014}
}

@inproceedings{wood2018predicting,
  title={Predicting Twitter user demographics from names alone},
  author={Wood-Doughty, Zach and Andrews, Nicholas and Marvin, Rebecca and Dredze, Mark},
  booktitle={Proceedings of the second workshop on computational modeling of people’s opinions, personality, and emotions in social media},
  pages={105--111},
  year={2018}
}

@article{faour2007religion,
  title={Religion, demography, and politics in Lebanon},
  author={Faour, Muhammad A},
  journal={Middle Eastern Studies},
  volume={43},
  number={6},
  pages={909--921},
  year={2007},
  publisher={Taylor \& Francis}
}

@article{greengard2025calibrated,
  title={A calibrated BISG for inferring race from surname and geolocation},
  author={Greengard, Philip and Gelman, Andrew},
  journal={Journal of the Royal Statistical Society Series A: Statistics in Society},
  volume={189},
  number={1},
  pages={512--543},
  year={2025},
  publisher={Oxford University Press UK}
}

@article{Imai_etal_2022,
author = {Kosuke Imai  and Santiago Olivella  and Evan T. R. Rosenman },
title = {Addressing census data problems in race imputation via fully Bayesian Improved Surname Geocoding and name supplements},
journal = {Science Advances},
volume = {8},
number = {49},
pages = {eadc9824},
year = {2022},
doi = {10.1126/sciadv.adc9824},
URL = {https://www.science.org/doi/abs/10.1126/sciadv.adc9824},
eprint = {https://www.science.org/doi/pdf/10.1126/sciadv.adc9824},
}

@article{Voicu2018,
author = {Ioan Voicu},
title = {Using First Name Information to Improve Race and Ethnicity Classification},
journal = {Statistics and Public Policy},
volume = {5},
number = {1},
pages = {1--13},
year = {2018},
doi = {10.1080/2330443X.2018.1427012},
URL = {https://doi.org/10.1080/2330443X.2018.1427012},
}

@misc{DasanaikeImai2026,
Author = {Noah Dasanaike and Kosuke Imai},
Title = {Using Embedding Models to Improve Probabilistic Race Prediction},
Year = {2026},
Eprint = {arXiv:2604.22555},
}

@article{McCartan02102025,
author = {Cory McCartan and Robin Fisher and Jacob Goldin and Daniel E. Ho and Kosuke Imai},
title = {Estimating Racial Disparities When Race is Not Observed},
journal = {Journal of the American Statistical Association},
volume = {120},
number = {552},
pages = {2140--2153},
year = {2025},
publisher = {Taylor \& Francis},
doi = {10.1080/01621459.2025.2526695},
URL = {https://doi.org/10.1080/01621459.2025.2526695},
eprint = { https://doi.org/10.1080/01621459.2025.2526695}
}

@misc{ipumsfullcount,
  title        = {{IPUMS} Ancestry Full Count Database: Version 3.0 [dataset]},
  author       = {Ruggles, Steven and Fitch, Catherine A. and Goeken, Ronald and Hacker, J. David and Nelson, Matt A. and Roberts, Evan and Schouweiler, Megan and Sobek, Matthew},
  publisher    = {IPUMS},
  address      = {Minneapolis, MN},
  year         = {2021},
  doi          = {10.18128/D014.V3.0}
}

@article{Nelson2025,
author = {Nelson, Matt A. and Magnuson, Diana L. and Hacker, J. David and Sobek, Matthew and Huynh, Lap and Roberts, Evan and Ruggles, Steven},
title = {New data sources for research on the nineteenth-century {United States}: {IPUMS} full count datasets of the censuses of population 1850--1880},
journal = {Historical Methods: A Journal of Quantitative and Interdisciplinary History},
volume = {58},
number = {4},
pages = {199--213},
year = {2025},
publisher = {Taylor \& Francis}
}

@Manual{wru,
  title = {wru: Who are You? Bayesian Prediction of Racial Category Using Surname, First Name, Middle Name, and Geolocation},
  author = {Kabir Khanna and Brandon Bertelsen and Santiago Olivella and Evan Rosenman and Alexander Rossell Hayes and Kosuke Imai},
  year = {2024},
  note = {R package version 3.0.3},
  url = {https://CRAN.R-project.org/package=wru},
  doi = {10.32614/CRAN.package.wru}
}

@article{greenwald2024regulatory,
  title={Regulatory Arbitrage or Random Errors? {I}mplications of Race Prediction Algorithms in Fair Lending Analysis},
  author={Greenwald, Daniel and Howell, Sabrina T. and Li, Cangyuan and Yimfor, Emmanuel},
  journal={Journal of Financial Economics},
  volume={157},
  pages={103857},
  year={2024}
}

@book{fraga2018turnout,
  title={The Turnout Gap: Race, Ethnicity, and Political Inequality in a Diversifying America},
  author={Fraga, Bernard L.},
  year={2018},
  publisher={Cambridge University Press}
}

@article{adjayegbewonyo2014bisg,
  title={Using the {B}ayesian Improved Surname Geocoding Method ({BISG}) to Create a Working Classification of Race and Ethnicity in a Diverse Managed Care Population: A Validation Study},
  author={Adjaye-Gbewonyo, Dzifa and Bednarczyk, Robert A. and Davis, Rachel L. and Omer, Saad B.},
  journal={Health Services Research},
  volume={49},
  number={1},
  pages={268--283},
  year={2014}
}

@article{fiscella2006geocoding,
  title={Use of Geocoding and Surname Analysis to Estimate Race and Ethnicity},
  author={Fiscella, Kevin and Fremont, Allen M.},
  journal={Health Services Research},
  volume={41},
  number={4p1},
  pages={1482--1500},
  year={2006}
}

@article{elliott2009using,
  title={Using the {C}ensus {B}ureau's Surname List to Improve Estimates of Race/Ethnicity and Associated Disparities},
  author={Elliott, Marc N. and Morrison, Peter A. and Fremont, Allen and McCaffrey, Daniel F. and Pantoja, Philip and Lurie, Nicole},
  journal={Health Services and Outcomes Research Methodology},
  volume={9},
  number={2},
  pages={69--83},
  year={2009}
}

@techreport{cfpb2014proxy,
  title={Using Publicly Available Information to Proxy for Unidentified Race and Ethnicity},
  author={{Consumer Financial Protection Bureau}},
  year={2014},
  institution={CFPB},
  address={Washington, DC}
}

@article{imaikhanna2016,
  title={Improving Ecological Inference by Predicting Individual Ethnicity from Voter Registration Records},
  author={Imai, Kosuke and Khanna, Kabir},
  journal={Political Analysis},
  volume={24},
  number={2},
  pages={263--272},
  year={2016}
}

@article{wang2024e5,
  title={Multilingual {E5} Text Embeddings: A Technical Report},
  author={Wang, Liang and Yang, Nan and Huang, Xiaolong and Yang, Linjun and Majumder, Rangan and Wei, Furu},
  journal={arXiv preprint arXiv:2402.05672},
  year={2024}
}

@article{sood2018predicting,
  title={Predicting Race and Ethnicity from the Sequence of Characters in a Name},
  author={Chintalapati, Rajashekar and Laohaprapanon, Suriyan and Sood, Gaurav},
  journal={arXiv preprint arXiv:1805.02109},
  year={2018}
}

@article{parasurama2021racebert,
  title={race{BERT}: A Transformer-Based Model for Predicting Race and Ethnicity from Names},
  author={Parasurama, Prasanna},
  journal={arXiv preprint arXiv:2112.03807},
  year={2021}
}

@article{jain2022importance,
  title={The Importance of Being {E}rnest, {E}kundayo, or {E}swari: An Interpretable Machine Learning Approach to Name-Based Ethnicity Classification},
  author={Jain, Vaishali and Enamorado, Ted and Rudin, Cynthia},
  journal={Harvard Data Science Review},
  volume={4},
  number={3},
  year={2022}
}

@inproceedings{ye2017nationality,
  title={Nationality Classification Using Name Embeddings},
  author={Ye, Junting and Han, Shuchu and Hu, Yifan and Coskun, Baris and Liu, Meizhu and Qin, Hong and Skiena, Steven},
  booktitle={Proceedings of the 2017 ACM Conference on Information and Knowledge Management},
  pages={1897--1906},
  year={2017}
}

@inproceedings{ye2019secret,
  title={The Secret Lives of Names? {N}ame Embeddings from Social Media},
  author={Ye, Junting and Skiena, Steven},
  booktitle={Proceedings of the 25th ACM SIGKDD International Conference on Knowledge Discovery \& Data Mining},
  pages={3000--3008},
  year={2019}
}

@article{rosenman2023names,
  title={Race and Ethnicity Data for First, Middle, and Surnames},
  author={Rosenman, Evan T. R. and Olivella, Santiago and Imai, Kosuke},
  journal={Scientific Data},
  volume={10},
  pages={299},
  year={2023}
}

@techreport{census2016surnames,
  title={Frequently Occurring Surnames from the 2010 Census},
  author={{U.S. Census Bureau}},
  year={2016},
  institution={U.S. Census Bureau},
  address={Washington, DC}
}

@article{himmelfarb1983surnames,
  title={Sampling by Ethnic Surnames: The Case of {A}merican {J}ews},
  author={Himmelfarb, Harold S. and Loar, R. Michael and Mott, Susan H.},
  journal={Public Opinion Quarterly},
  volume={47}, number={2}, pages={247--260}, year={1983}
}

@article{lauderdale2000asian,
  title={Asian {A}merican Ethnic Identification by Surname},
  author={Lauderdale, Diane S. and Kestenbaum, Bert},
  journal={Population Research and Policy Review},
  volume={19}, number={3}, pages={283--300}, year={2000}
}

@article{shah2010surname,
  title={Surname Lists to Identify {S}outh {A}sian and {C}hinese Ethnicity from Secondary Data in {O}ntario, {C}anada: A Validation Study},
  author={Shah, Baiju R. and Chiu, Maria and Amin, Shubarna and Ramani, Meena and Sadry, Shahriar and Tu, Jack V.},
  journal={BMC Medical Research Methodology},
  volume={10}, pages={42}, year={2010}
}

@unpublished{Dasanaike2026zeroshot,
  title={Pre-Trained Language Models as Zero-Shot Tools for Social Science Research},
  author={Dasanaike, Noah},
  year={2026},
  note={Working paper, Harvard University}
}

@article{chasalow2026sampling,
  title={Improving Minority Population Sampling with {BISG} Probabilities: Evidence from a Survey of {J}ewish {A}mericans},
  author={Chasalow, Kyla and Hersh, Eitan and Imai, Kosuke and Royden, Laura},
  journal={arXiv preprint arXiv:2605.05384},
  year={2026}
}

@inproceedings{reimers-gurevych-2020-making,
  title = {Making Monolingual Sentence Embeddings Multilingual using Knowledge Distillation},
  author = {Reimers, Nils and Gurevych, Iryna},
  booktitle = {Proceedings of the 2020 Conference on Empirical Methods in Natural Language Processing},
  year = {2020},
  publisher = {Association for Computational Linguistics}
}

@article{xin2026proxy,
  title = {How Proxy Race Distorts Regression-Based Fairness Audits},
  author = {Xin, Xi and Hooker, Giles and Huang, Fei},
  journal = {arXiv preprint arXiv:2603.17106},
  year = {2026},
}

\newpage 
\section*{Acknowledgments}

We thank reviewers of the Alexander and Diviya Magaro Peer Pre-Review
Program and attendees of the Harvard Institute for Quantitative Social Science Applied Statistics Workshop for useful comments. 

\paragraph{Data and materials availability:} 

All data necessary to evaluate the conclusions in the paper will be made publicly available at Harvard Dataverse, with the following restrictions. The versions of the Florida and North Carolina voter files we used are proprietary and available for purchase from L2, Inc., and thus cannot be redistributed. The full-count 1900 U.S. Census with full names is available to approved researchers under a restricted-use agreement with IPUMS USA \citep{ipumsfullcount}; we accessed it through the National Bureau of Economic Research (NBER) and release only anonymized aggregate outputs. The 2018 Lebanese voter registry was published by the Lebanese Ministry of Interior and Municipalities. However, because it records the religious sect of named individuals, we do not redistribute it, and similarly release only aggregate outputs. 

All code, the LLM-generated name lists, prompts, and aggregate outputs will be made available in the Harvard Dataverse. With these materials, every table and figure can be reproduced from the aggregate outputs, and the analyses can be rerun by researchers with access to the restricted data. $\ell$BISG is implemented in the open-source R package {\tt wru}, available from CRAN (\url{https://CRAN.R-project.org/package=wru}).

\newpage 
\appendix

\setcounter{section}{0}
\setcounter{equation}{0}
\setcounter{figure}{0}
\setcounter{table}{0}
\renewcommand {\theequation} {S\arabic{equation}}
\renewcommand {\thefigure} {S\arabic{figure}}
\renewcommand {\thetable} {S\arabic{table}}
\renewcommand {\thesection} {S\arabic{section}}

\section{Bias caused by hard thresholding}\label{app-threshbias}

\subsection{Bias expression given in Section~\ref{sec-valueofprediction}}

Without loss of generality, we ignore conditioning on geography $\geo=\geolow$ below. A simple approach to imputing the missing group membership would be to set 
\begin{equation}
    \targ^*= \argmax_\targlow \Pr(\targ=\targlow\mid \sur=\surlow)
\end{equation}
and then estimate group means using
\begin{equation}
\frac{\sum_{i=1}^{N}\ind\{Y_i=y\}\ind\{\targ_i^*=\targlow\}}{\sum_{i=1}^{N}\ind\{\targ_i^*=\targlow\}}
\end{equation}
which converges as $N\rightarrow \infty$ to $\Pr(\outcome=\outcomelow\mid \targ^*=r)$. The question is how this compares to $\Pr(\outcome=\outcomelow\mid \targ=\targlow)$. For simplicity, we will suppose $\targ$ is binary, taking values of $0$ or $1$, and without loss of generality, we exclude conditioning on geography from this proof. Let $\mathcal{C}_0$ be the set of surnames that are classified as $0$ under this rule and $\mathcal{C}_1$ be the set classified as $1$. Note also that because $\targ^*$ is a function of $S$, $Y\indep \targ^*\mid R,S$, and define
\begin{equation}
    \mu_{y\mid r,s}:=\Pr(Y=y\mid R=r,S=s) \quad\quad \mu_{y\mid r}:=\Pr(Y=y\mid R=r)
\end{equation}
Then
\begin{align*}
    \Pr(\outcome=\outcomelow\mid \targ^*=1) 
    &= \sum_{\targlow=0}^{1}\sum_{\surlow}\Pr(\outcome=\outcomelow\mid \targ^*=1,\targ=\targlow,\sur=\surlow)\Pr(\targ=\targlow,\sur=\surlow\mid \targ^*=1)\\
    &=\sum_{\targlow=0}^{1}\sum_{\surlow}\mu_{y\mid r,s}\frac{\Pr(\targ^*=1\mid \sur=\surlow,\targ=\targlow)\Pr(\sur=\surlow\mid \targ=\targlow)\Pr(\targ=\targlow)}{\Pr(\targ^*=1)}\\
    &=\sum_{\targlow=0}^{1}\sum_{\surlow}\mu_{y\mid r,s}\ind\{\surlow\in \mathcal{C}_1\}\frac{\Pr(\sur=\surlow\mid \targ=\targlow)\Pr(\targ=\targlow)}{\Pr(\targ^*=1)}\\
    &=\sum_{\targlow=0}^{1}\sum_{\surlow\in \mathcal{C}_1}\mu_{y\mid r,s}\frac{\Pr(\sur=\surlow\mid \targ=\targlow)\Pr(\targ=\targlow)}{\Pr(\sur\in C_1)}
\end{align*}
Noting in the last line that $\{\targ^*=1\}\leftrightarrow \{\sur\in C_1\}$. Using this, we have
\small 
\begin{align*}
   &\Pr(\outcome=\outcomelow\mid \targ^*=1)-\Pr(\outcome=\outcomelow\mid \targ=1)\\
    &=\sum_{\targlow=0}^{1}\sum_{\surlow\in \mathcal{C}_1}\mu_{y\mid r,s}\frac{\Pr(\sur=\surlow\mid \targ=\targlow)\Pr(\targ=\targlow)}{\Pr(\sur\in \mathcal{C}_1)} - \sum_{\surlow}\mu_{y\mid 1,s}\Pr(\sur=\surlow\mid \targ=1)\\
    &=\sum_{\surlow\in \mathcal{C}_1}\mu_{y\mid 0,s}\frac{\Pr(\sur=\surlow\mid \targ=0)\Pr(\targ=0)}{\Pr(\sur\in \mathcal{C}_1)} \\
    &\quad \quad + \sum_{\surlow\in \mathcal{C}_1}\mu_{y\mid 1,s}\Pr(\sur=\surlow\mid \targ=1)\left[\frac{\Pr(\targ=1)}{\Pr(\sur\in \mathcal{C}_1)}-1\right] - \sum_{\surlow\in \mathcal{C}_0}\mu_{y\mid 1,s}\Pr(\sur=\surlow\mid \targ=1)\\
\end{align*}
\normalsize
If we add $\outcome\indep \sur\mid \targ$, this becomes
\begin{align*}
    &=\frac{\mu_{y\mid 0}}{\Pr(\sur\in \mathcal{C}_1)}\Pr(\sur\in \mathcal{C}_1,\targ=0)\\
    &\quad \quad +\mu_{y\mid 1} \Pr(\sur\in \mathcal{C}_1\mid \targ=1)\left[\frac{\Pr(\targ=1)}{\Pr(\sur\in \mathcal{C}_1)}-1\right] - \mu_{y\mid 1}\Pr(\sur\in \mathcal{C}_0\mid \targ=1)\\ 
    &=\frac{\mu_{y\mid 0}}{\Pr(\sur\in \mathcal{C}_1)}\Pr(\sur\in \mathcal{C}_1,\targ=0)\\
    &\quad \quad +\mu_{y\mid 1} \left[\Pr(\sur\in \mathcal{C}_1\mid \targ=1)\left[\frac{\Pr(\targ=1)}{\Pr(\sur\in \mathcal{C}_1)}-1\right] - \Pr(\sur\in \mathcal{C}_0\mid \targ=1)\right] \\
    &=\mu_{y\mid 0}\Pr(\targ=0\mid \sur\in\mathcal{C}_1) +\mu_{y\mid 1} \left[\Pr(\targ=1\mid \sur\in\mathcal{C}_1) -1\right] \\
    &=\mu_{y\mid 0}\Pr(\targ=0\mid \sur\in\mathcal{C}_1) -\mu_{y\mid 1} \left[\Pr(\targ=0\mid \sur\in\mathcal{C}_1) \right] \\
    &= [\mu_{y\mid 0}-\mu_{y\mid 1}]\Pr(\targ=0\mid \sur\in \mathcal{C}_1)
\end{align*}
If classification is perfect, $\Pr(\targ=0\mid \sur\in\mathcal{C}_1)=0$ and this is $0$. If $Y\indep \targ$ (no group disparity), then the left term becomes $\Pr(\outcome=\outcomelow)-\Pr(\outcome=\outcomelow)=0$. The calculation for $\Pr(\outcome=\outcomelow\mid \targ^*=0)-\Pr(\outcome=\outcomelow\mid \targ=0)$ is identical but flipped
\begin{align*}
    &\Pr(\outcome=\outcomelow\mid \targ^*=0)-\Pr(\outcome=\outcomelow\mid \targ=0)= [\mu_{y\mid 1}-\mu_{y\mid 0}]\Pr(\targ=1\mid \sur\in \mathcal{C}_0)
\end{align*}
\noindent Hence the asymptotic bias in the disparity is
\begin{align*}
&[\Pr(\outcome=\outcomelow\mid \targ^*=1)-\Pr(\outcome=\outcomelow\mid \targ^*=0)]-[\mu_{y\mid 1}-\mu_{y\mid 0}]\\
&= \left([\mu_{y\mid 0}-\mu_{y\mid 1}]\Pr(\targ=0\mid \sur\in \mathcal{C}_1)\right) - \left([\mu_{y\mid 1}-\mu_{y\mid 0}]\Pr(\targ=1\mid \sur\in \mathcal{C}_0)\right) \\ 
&=-[\mu_{y\mid 1}-\mu_{y\mid 0}]
[\Pr(\targ=1\mid \sur\in\mathcal{C}_0)+\Pr(\targ=0\mid \sur\in\mathcal{C}_1)]
\end{align*}

\subsection{Bias without the conditional independence assumption}

If we do not assume $Y\indep S\mid R$, then a more algebraically complicated calculation yields that the asymptotic bias in the disparity estimation is as follows:
\small
\begin{align*}
&\left(\Pr(\outcome=\outcomelow\mid \targ^*=1)-\Pr(\outcome=\outcomelow\mid \targ=1)\right)-\left(\Pr(\outcome=\outcomelow\mid \targ^*=0)-\Pr(\outcome=\outcomelow\mid \targ=0)\right)\\
    &=\sum_{\surlow\in \mathcal{C}_1}\mu_{y\mid 0,s}\frac{\Pr(\sur=\surlow\mid \targ=0)\Pr(\targ=0)}{\Pr(\sur\in \mathcal{C}_1)} \\
    &\quad \quad + \sum_{\surlow\in \mathcal{C}_1}\mu_{y\mid 1,s}\Pr(\sur=\surlow\mid \targ=1)\left[\frac{\Pr(\targ=1)}{\Pr(\sur\in \mathcal{C}_1)}-1\right]- \sum_{\surlow\in \mathcal{C}_0}\mu_{y\mid 1,s}\Pr(\sur=\surlow\mid \targ=1)\\%
     &\quad\quad -\sum_{\surlow\in \mathcal{C}_0}\mu_{y\mid 1,s}\frac{\Pr(\sur=\surlow\mid \targ=1)\Pr(\targ=1)}{\Pr(\sur\in \mathcal{C}_0)} \\
    &\quad \quad - \sum_{\surlow\in \mathcal{C}_0}\mu_{y\mid 0,s}\Pr(\sur=\surlow\mid \targ=0)\left[\frac{\Pr(\targ=0)}{\Pr(\sur\in \mathcal{C}_0)}-1\right] + \sum_{\surlow\in \mathcal{C}_1}\mu_{y\mid 0,s}\Pr(\sur=\surlow\mid \targ=0) \\%
    &= \left(\frac{\Pr(R=0)}{\Pr(S\in \mathcal{C}_1)} +1\right)\sum_{\surlow\in\mathcal{C}_1}\mu_{y\mid 0,s}\Pr(S=s\mid R=0)-\left(\frac{\Pr(R=1)}{\Pr(S\in \mathcal{C}_0)} +1\right)\sum_{s\in \mathcal{C}_0}\mu_{y\mid 1,s}\Pr(S=s\mid R=1)\\
    &+\left[\frac{\Pr(\targ=1)}{\Pr(\sur\in \mathcal{C}_1)}-1\right]\sum_{\surlow\in \mathcal{C}_1}\mu_{y\mid 1,s}\Pr(\sur=\surlow\mid \targ=1)-\left[\frac{\Pr(\targ=0)}{\Pr(\sur\in \mathcal{C}_0)}-1\right]\sum_{\surlow\in \mathcal{C}_0}\mu_{y\mid 0,s}\Pr(\sur=\surlow\mid \targ=0)
\end{align*}
\normalsize
Letting $\mu_{y\mid r,C_{r'}}=\Pr(Y=y\mid \targ=\targlow,\sur\in\mathcal{C}_{\targlow'})$ this is
\small
\begin{align*}
    &=\left(\frac{\Pr(R=0)}{\Pr(S\in \mathcal{C}_1)} +1\right)\mu_{y\mid 0,\mathcal{C}_1}\Pr(S\in \mathcal{C}_1\mid R=0)\\
    &\quad-\left(\frac{\Pr(R=1)}{\Pr(S\in \mathcal{C}_0)} +1\right)\mu_{y\mid 1,\mathcal{C}_0}\Pr(S\in \mathcal{C}_0\mid R=1)\\
    &+\left[\frac{\Pr(\targ=1)}{\Pr(\sur\in \mathcal{C}_1)}-1\right]\mu_{y\mid 1,\mathcal{C}_1}\Pr(S\in \mathcal{C}_1\mid R=1)\\
    &\quad-\left[\frac{\Pr(\targ=0)}{\Pr(\sur\in \mathcal{C}_0)}-1\right]\mu_{y\mid 0,\mathcal{C}_0}\Pr(S\in \mathcal{C}_0\mid R=0)\\
    &\\ 
    &=\mu_{y\mid 0,\mathcal{C}_1}
      \left[
      \Pr(R=0\mid S\in\mathcal{C}_1)
      +\Pr(S\in\mathcal{C}_1\mid R=0)
      \right]\\
    &\quad-\mu_{y\mid 1,\mathcal{C}_0}
      \left[
      \Pr(R=1\mid S\in\mathcal{C}_0)
      +\Pr(S\in\mathcal{C}_0\mid R=1)
      \right]\\
    &\quad+\mu_{y\mid 1,\mathcal{C}_1}
      \left[
     \Pr(R=1\mid S\in\mathcal{C}_1)
      -\Pr(S\in\mathcal{C}_1\mid R=1)
      \right]\\
    &\quad-\mu_{y\mid 0,\mathcal{C}_0}
      \left[
      \Pr(R=0\mid S\in\mathcal{C}_0)
      -\Pr(S\in\mathcal{C}_0\mid R=0)
      \right]\\
     & \\
    &=\mu_{y\mid 0,\mathcal{C}_1}
      \left[
      \Pr(R=0\mid S\in\mathcal{C}_1)
      +\Pr(S\in\mathcal{C}_1\mid R=0)
      \right]\\
    &\quad-\mu_{y\mid 1,\mathcal{C}_0}
      \left[
      \Pr(R=1\mid S\in\mathcal{C}_0)
      +\Pr(S\in\mathcal{C}_0\mid R=1)
      \right]\\
    &\quad+\mu_{y\mid 1,\mathcal{C}_1}
      \left[
      \Pr(S\in\mathcal{C}_0\mid R=1)
      -\Pr(R=0\mid S\in\mathcal{C}_1)
      \right]\\
    &\quad-\mu_{y\mid 0,\mathcal{C}_0}
      \left[
      \Pr(S\in\mathcal{C}_1\mid R=0)-\Pr(R=1\mid S\in\mathcal{C}_0)
      \right]\\
    &  \\
    &=\mu_{y\mid 0,\mathcal{C}_1}
      \Pr(R=0\mid S\in\mathcal{C}_1)
      -\mu_{y\mid 1,\mathcal{C}_1}
      \Pr(R=0\mid S\in\mathcal{C}_1)\\
    &\quad+\mu_{y\mid 0,\mathcal{C}_1}
      \Pr(S\in\mathcal{C}_1\mid R=0)
      -\mu_{y\mid 0,\mathcal{C}_0}
      \Pr(S\in\mathcal{C}_1\mid R=0)\\
    &\quad+\mu_{y\mid 1,\mathcal{C}_1}
      \Pr(S\in\mathcal{C}_0\mid R=1)
      -\mu_{y\mid 1,\mathcal{C}_0}
      \Pr(S\in\mathcal{C}_0\mid R=1)\\
    &\quad+\mu_{y\mid 0,\mathcal{C}_0}
      \Pr(R=1\mid S\in\mathcal{C}_0)
      -\mu_{y\mid 1,\mathcal{C}_0}
      \Pr(R=1\mid S\in\mathcal{C}_0)\\
\end{align*}
\normalsize
which gives a final expression
\begin{align*}
&= \Pr(R=0\mid S\in\mathcal{C}_1)
      \left[
      \mu_{y\mid 0,\mathcal{C}_1}
      -\mu_{y\mid 1,\mathcal{C}_1}
      \right] +\Pr(S\in\mathcal{C}_1\mid R=0)
      \left[
      \mu_{y\mid 0,\mathcal{C}_1}
      -\mu_{y\mid 0,\mathcal{C}_0}
      \right]\\
      &\quad +\Pr(S\in\mathcal{C}_0\mid R=1)
      \left[
      \mu_{y\mid 1,\mathcal{C}_1}
      -\mu_{y\mid 1,\mathcal{C}_0}
      \right] +\Pr(R=1\mid S\in\mathcal{C}_0)
      \left[
      \mu_{y\mid 0,\mathcal{C}_0}
      -\mu_{y\mid 1,\mathcal{C}_0}
      \right]
\end{align*}

\section{Using the list score directly}\label{app-rawf}

Using the Florida data, we show that the list score $f_{\targlow,\emblow} = \Pr(\listmem_{\targlow}{=}1\mid \emb=\emblow)$ cannot serve as the race probability on its own.  That is, the probability that a name is on the list does not equal the probability that a person with that name belongs to the group. This is because lists have imperfect precision -- some on the list are not in group $\targlow$ and some off it are in the group -- and because lists ignore base rates.

\begin{figure}[h!]\centering
\includegraphics[width=\textwidth]{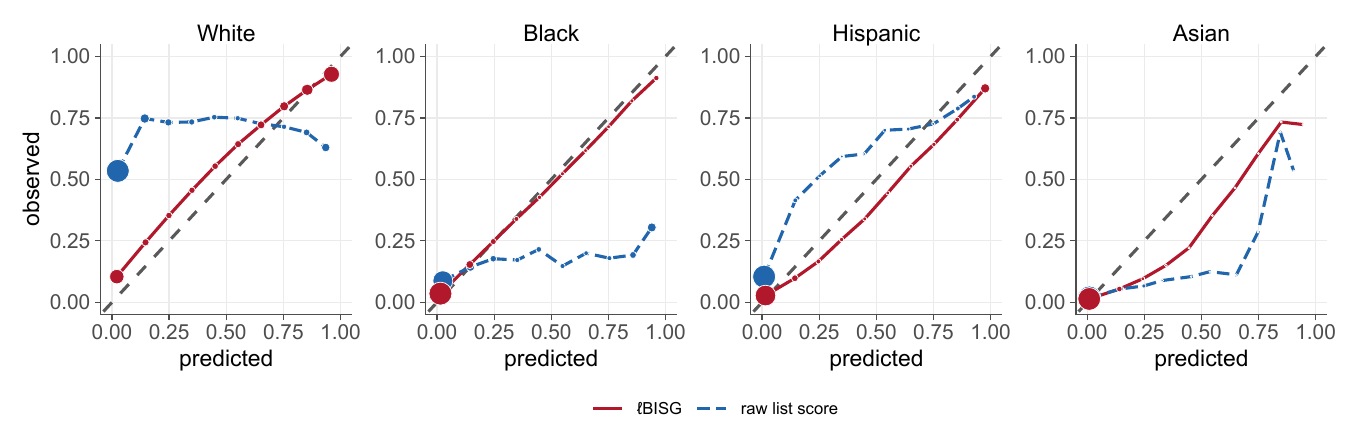}
\caption{Calibration of the raw LLM list score $f_\targlow = P(\listmem_{\targlow}{=}1\mid \emb)$ used directly as a race probability (blue, dashed) against $\ell$BISG (red, solid), by race, on the Florida registrants. The dashed diagonal indicates perfect agreement between estimated probabilities and observed proportions. The area of each circle is proportional to the number of registrants in the bin.}
\label{fig:rawf}
\end{figure}

Figure~\ref{fig:rawf} demonstrates that $f$ (blue dashed line) is not well calibrated. The failure is most dramatic for White, the majority group, whose raw scores substantially underpredict its prevalence. For comparison, Figure~\ref{fig:rawf} also includes the surname-only $\ell$BISG probabilities based on the LLM lists of Section~\ref{sec-val-voterfiles}, which adjust the raw scores using information gained from the geographic prevalence information.

\FloatBarrier

\section{Surname embeddings and the list score}\label{app-embed}

\begin{figure}[!t]\centering
\vspace{-.1in}
\includegraphics[width=0.75\textwidth]{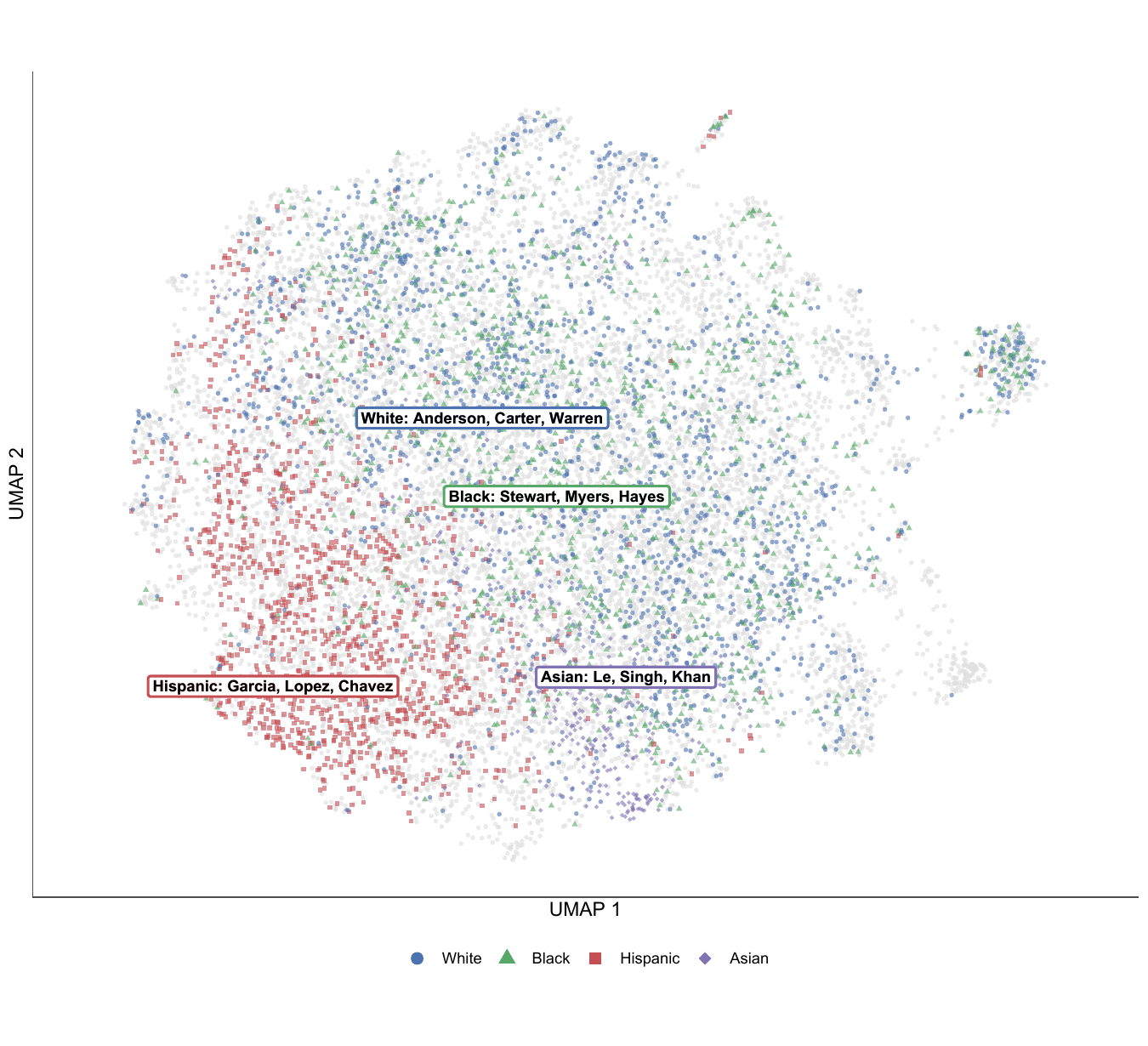}
\vspace{-.3in}
\caption{Surname embeddings (E5), Florida. A two-dimensional UMAP projection colored by dominant group (i.e., $\arg\max_\targlow f_\targlow$), with one exemplar label per group.}
\label{fig:umap}
\end{figure}

The list score $f_{\targlow,\emblow} = \Pr(\listmem_{\targlow}{=}1\mid \emb=\emblow)$ is a smooth function of the name embedding. In Figure~\ref{fig:umap}, we project the Florida surname embeddings to two dimensions. We color each embedding point $\emblow$ by the dominant predicted group list membership (i.e., $\argmax_\targlow f_{\targlow,\emblow}$). We find that this partitions the space into interpretable regions which we label with a few familiar names for each group (Figure~\ref{fig:umap}). A name that belongs to no list still receives a score learned from other names in its embedding neighborhood.

\begin{figure}[b]\centering
\includegraphics[width=0.62\textwidth]{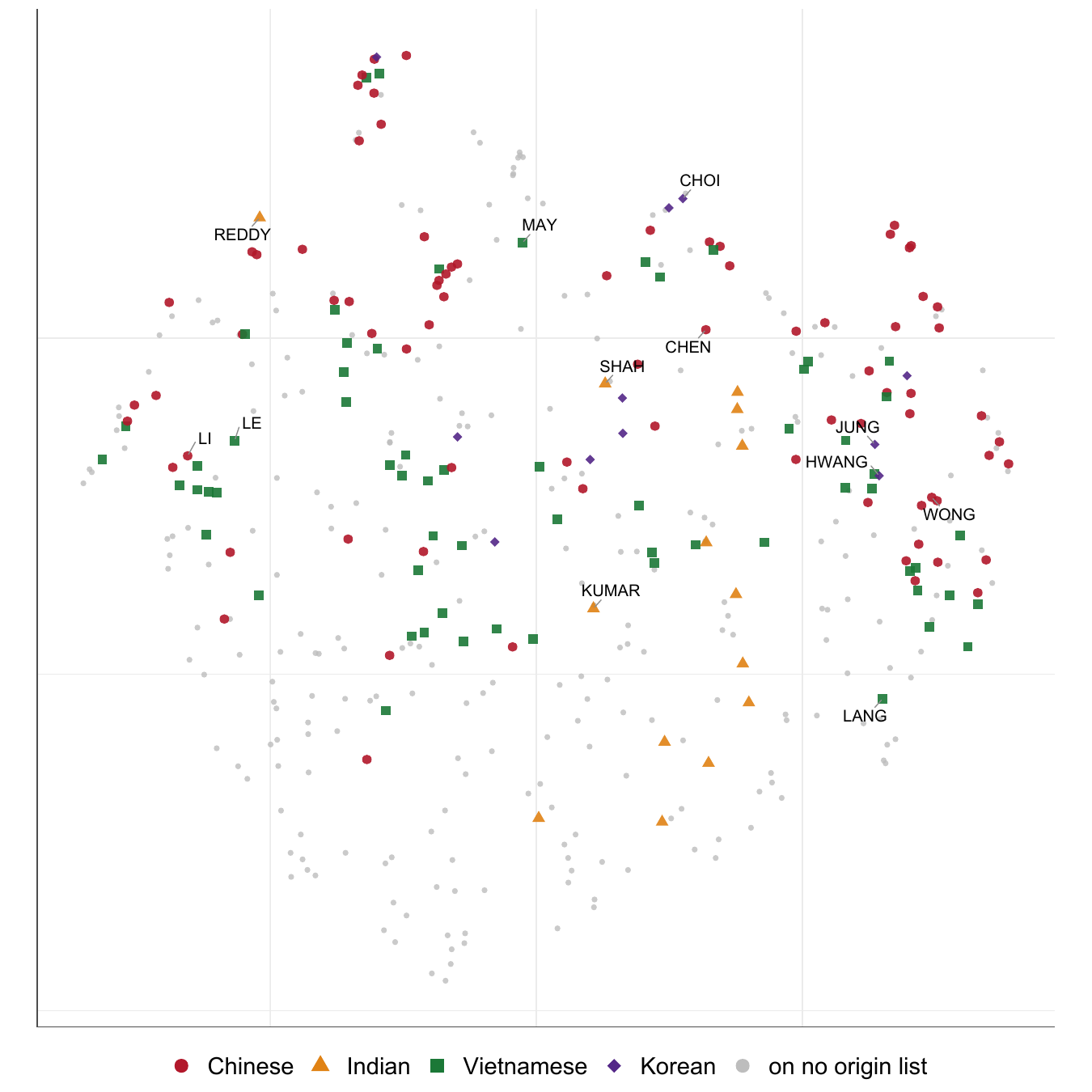}
\caption{The $428$ Asian-dominant Florida surnames, colored by country-of-origin list. Grey points are on no origin list.}
\label{fig:umap-asian}
\end{figure}

Figure~\ref{fig:umap-asian} highlights that within the Asian group, Indian surnames are located apart from East and Southeast Asian ones in the embedding space, while Chinese, Vietnamese, and Korean surnames overlap. This illustrates the value of setting the number of clusters in the $K$-means clustering algorithm to a value larger than $\dimtarg$ -- we may need to pick up multiple Asian clusters.

\FloatBarrier

\section{Empirical results for the choice of the number of clusters}\label{app-selectM}

Figure~\ref{fig:selectM-criteria} reports $Q(K)$ of equation~\eqref{e-georatio} for the surname and first-name models in the one-million-record voter file samples. All four surname models and both LLM first-name models select $K=7{,}500$, and both Census first-name models select $K=5{,}000$.

\begin{figure}[t]\centering
\includegraphics[width=\textwidth]{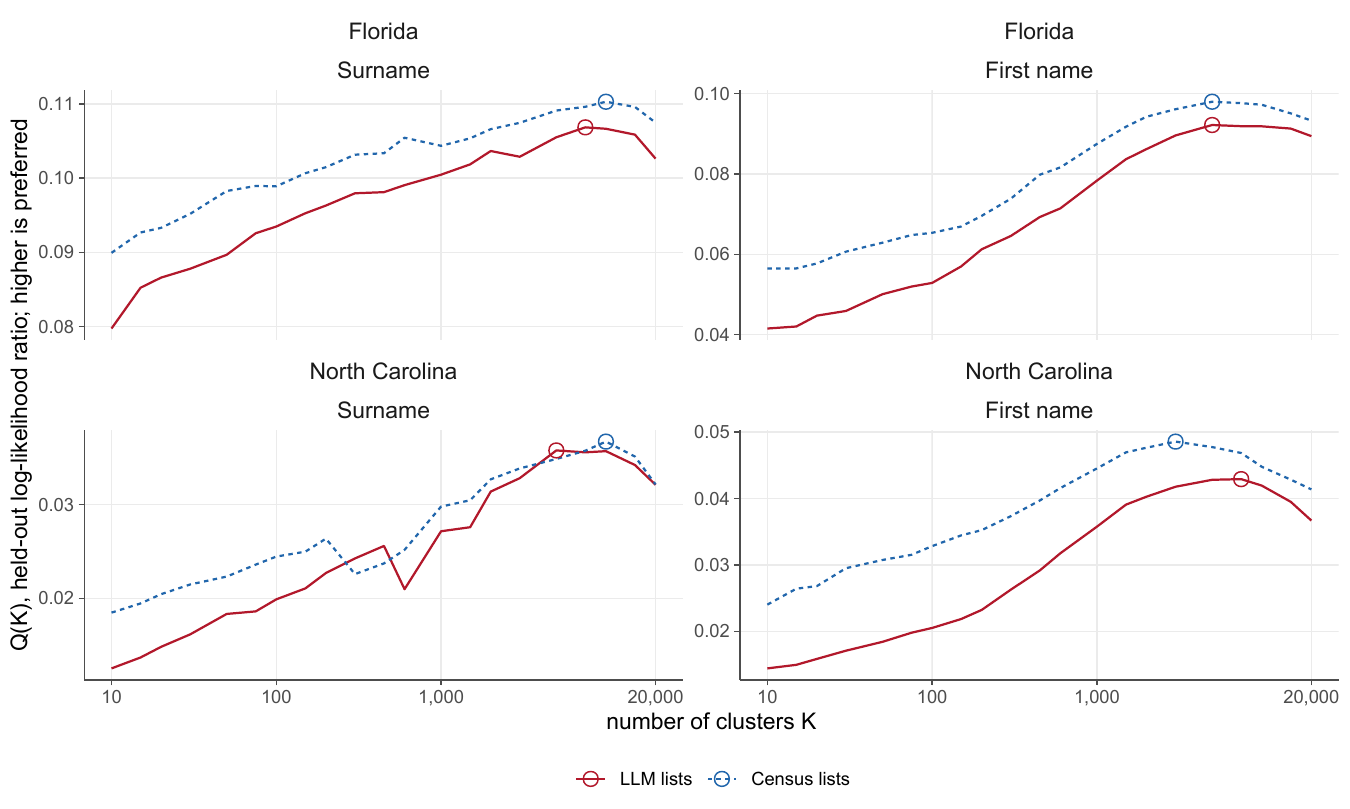}
\caption{The label-free criterion $Q(K)$ for choosing the number of clusters on the one million record Florida and North Carolina samples. Panels facet state and name field; colors distinguish LLM and Census lists. Higher values are preferred. Open circles mark the selected value of $K$. The candidate grid ranges from $10$ to $20{,}000$.}
\label{fig:selectM-criteria}
\end{figure}

Figure~\ref{fig:selectK-truth} compares $Q(K)$ with the ground truth. We use the Florida surname model with the LLM lists and draw random subsamples of $100{,}000$ and $300{,}000$ registrants from the one-million-record sample, which we also show. The ground truth is the PR-AUC and the Brier score of the resulting $\ell$BISG probabilities against self-reported race, averaged over the four groups. At every sample size, $Q(K)$ follows the ground truth, selecting $K=75$, $K=999$, and $K=7{,}500$, where the PR-AUC is respectively highest at $K=100$, $K=450$, and $K=5{,}000$ and the Brier score is lowest at $K=150$, $K=600$, and $K=7{,}500$. The selected value of $K$ is greater in larger samples because more clusters can be estimated reliably from more observations. 

\begin{figure}[!t]\centering
\includegraphics[width=\textwidth]{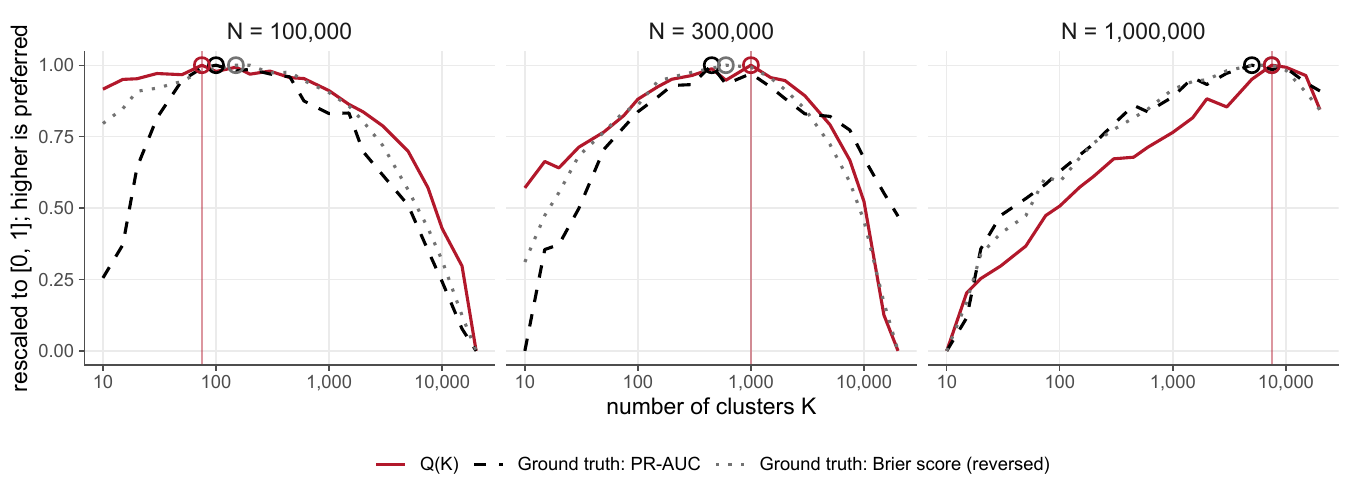}
\caption{The criterion $Q(K)$ of equation~\eqref{e-georatio} and the ground truth criterion for choosing the number of clusters $K$, by sample size, for the Florida surname model with the LLM lists. The ground truth is the PR-AUC (dashed) and the Brier score (dotted) against self-reported race, averaged over the four groups. The Brier score is reversed so that higher values are preferred for every curve. Each curve is rescaled to the unit interval within its panel, open circles denote the largest value of each curve, and the vertical line marks the selected $K$.}
\label{fig:selectK-truth}
\end{figure}

\section{Comparison with K-means on embeddings}
\label{app:cluster-embeddings}

\begin{figure}[!b]\centering
\includegraphics[width=\textwidth]{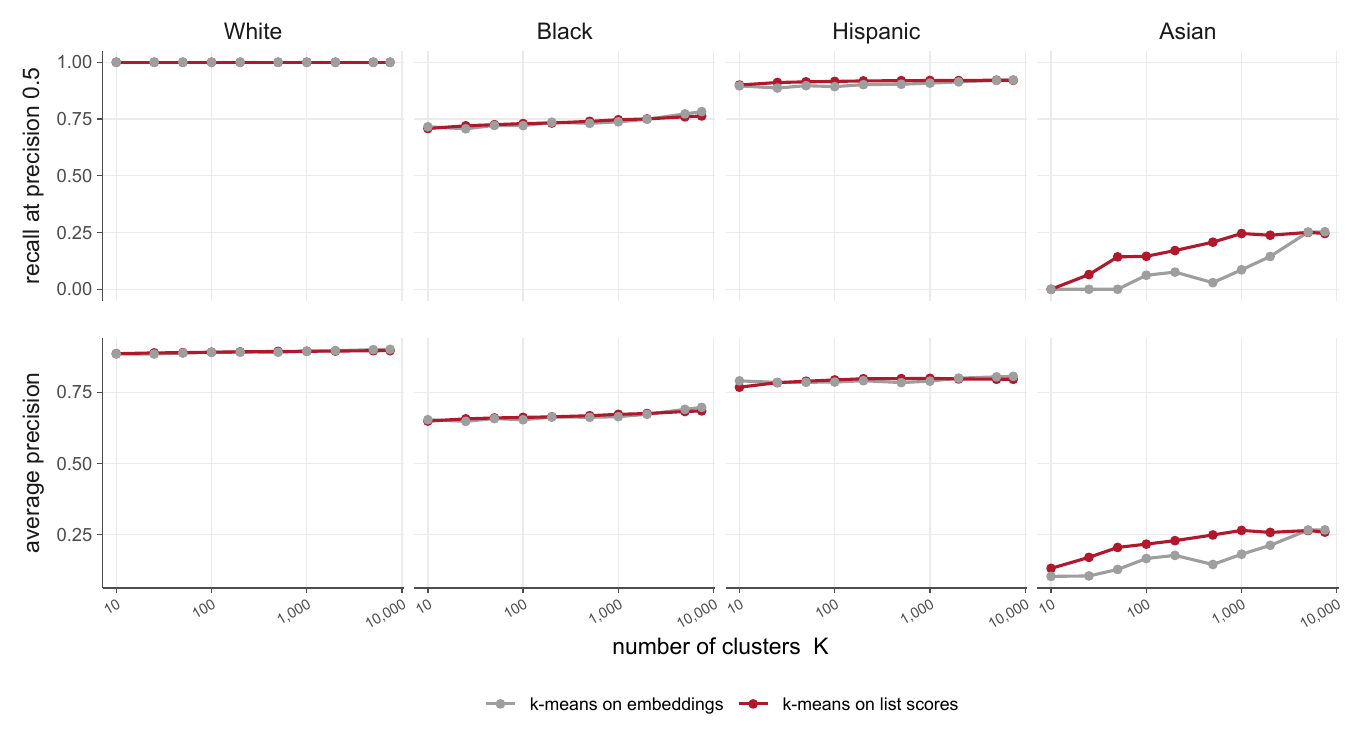}
\caption{Discrimination when the clusters are formed by K-means on the raw E5 surname embeddings (no lists) against K-means on the list scores. Florida sample with the block-level prior, by number of clusters $K$. Top row: recall at a precision of $0.5$. Bottom row: average precision.}
\label{fig:kmeans-emb-lists}
\end{figure}

Figure~\ref{fig:kmeans-emb-lists} compares the recall at a precision of $0.5$ and the average precision of $\ell$BISG under two approaches to the dimension reduction described in Section \ref{sec-dimreduction}. Our proposed approach is K-means clustering on the vectors of log-odds of list scores $f_{e}=(f_{1,e},...,f_{\dimtarg,e})$. We compare this to running K-means clustering on the embeddings directly, which relies more heavily on the embeddings being clustered by group membership even if they come from an embedding model trained for other tasks.

In the Florida example, clustering the list scores yields more informative clusters than clustering the raw embeddings directly for the rare Asian group when the number of clusters is moderate. At $K=500$, the average precision for the Asian group is $0.25$ using the list scores compared to $0.15$ with the raw embeddings. The recall for each approach at a precision of $0.5$ is $0.21$ compared to $0.03$, respectively (Figure~\ref{fig:kmeans-emb-lists}). For the other three groups, the two approaches perform similarly at every selected $K$. At the $K=7{,}500$ selected by Algorithm~\ref{alg-selectK}, both approaches perform similarly for every group. Differential performance under list scores and raw embeddings therefore appears to arise primarily when there are too few observations to estimate many clusters.

\begin{figure}[!t]\centering
\includegraphics[width=\textwidth]{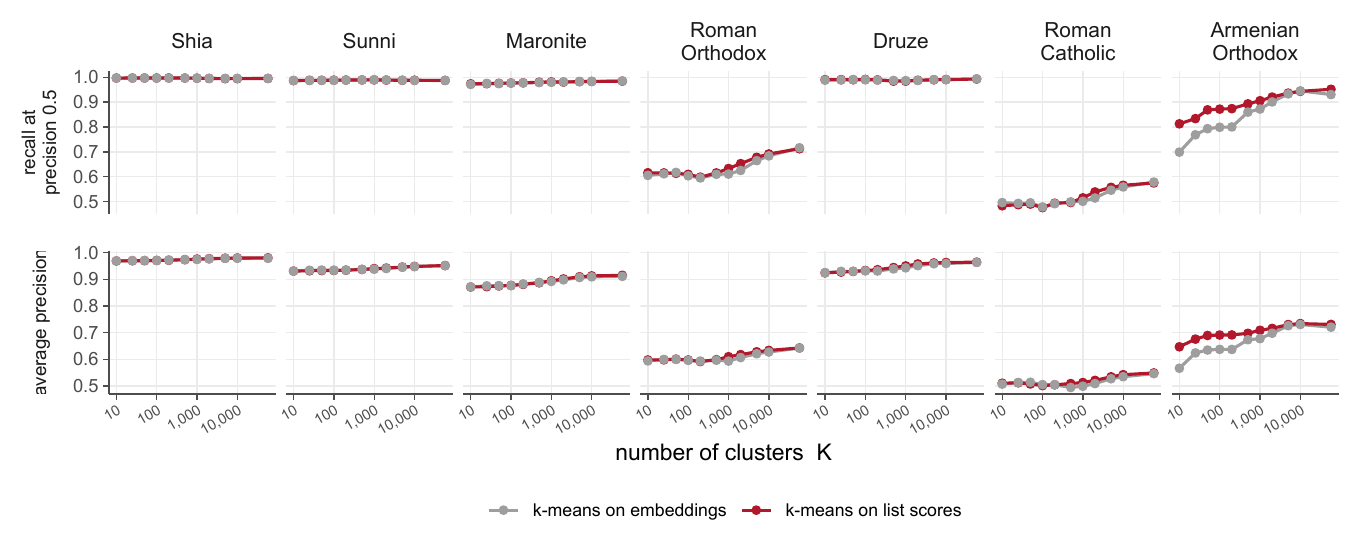}
\caption{As Figure~\ref{fig:kmeans-emb-lists}, for the seven principal Lebanese sects, with the multilingual surname embeddings of Section~\ref{sec-val-lebanon} and the locality-level prior.}
\label{fig:kmeans-emb-lists-lebanon}
\end{figure}

In the Lebanese voter file (Figure~\ref{fig:kmeans-emb-lists-lebanon}), clustering the list scores rather than the raw embeddings directly yields more informative clusters for the Armenian Orthodox, the smallest sect, when the number of clusters is small or moderate. At $K=100$, the average precision for the Armenian Orthodox is $0.69$ using the list scores compared to $0.64$ with the raw embeddings, and the recall at a precision of $0.5$ is $0.87$ compared to $0.80$. For the other six sects, the two approaches perform similarly, and at the $K=57{,}567$ selected by Algorithm~\ref{alg-selectK}, they perform similarly for every sect.

\FloatBarrier

\section{Details for recovery of subgroup geographic probabilities}\label{app-geography-recovery}

In this appendix, we provide details of recovering $\Pr(\targ=\targlow \mid \geo=\geolow)$ as discussed in Section~\ref{sec-recovergeography}. We maintain Assumption~\ref{a-listqual}(a) throughout but weaken its part (b) to allow $\Pr(\listmem_{\targlow'}=1\mid \targ=\targlow)$ to take positive values, with the original assumption represented by the $\lambda=0$ or $\lambda=\gamma=0$ cases of the sensitivity analysis below. Appendix Sections~\ref{app-geoprevsensitivityresults}~and~\ref{app-subgroup-geoprevsensitivityresults} provide a detailed description of the systems of equations we solve and the sensitivity analysis results. Section~\ref{app-nonnegative-ls} describes implementation details, while Appendix~\ref{app-sensitivityproofs} provides proofs.

\subsection{Recovery and sensitivity analysis for geographic prevalence}\label{app-geoprevsensitivityresults}

We start with the regime where $\Pr(\targ=\targlow\mid \geo)$ is entirely unavailable. Stacking main text equations~\eqref{eq-geo-recover2}~and~\eqref{eq-sumto1} over geographies, we can write the system of constraints as
\begin{align}
   \mmat \Pimat &= \betamat \label{eq-georecov-matrixform1}\\
   \mmat \bm{1}_{\dimtarg} &= \bm{1}_{\dimgeo} \label{eq-georecov-matrixform2}
\end{align}
 where $\betamat$ is the full rank (Assumption~\ref{a-betarank}) $\dimgeo\times\dimtarg$ matrix of list rates $\Pr(\listmem_{\targlow}=1\mid\geo)$, $\Pimat$ is the $\dimtarg\times\dimtarg$ matrix of within-group list rates $\Pr(\listmem_{\targlow'}=1\mid \targ=\targlow)$, and $\mmat$ is the $\dimgeo\times\dimtarg$ matrix of geographic prevalence $\Pr(\targ=\targlow\mid \geo)$. The matrix $\betamat$ is directly identified from the data, but $\Pimat$ and $\mmat$ are unknown. We wish to solve for $\mmat$ but cannot do so without further constraints on $\Pimat$.

As discussed in the main text, one way to simplify $\Pimat$ is Assumption~\ref{a-listqual}, which imposes that $\Pimat$ is diagonal. We now suppose that violations of list group separation are allowed but are uniform in nature so that for any $\targlow\neq\targlow'$,
\begin{equation}
    \Pr(\listmem_{\targlow'}=1\mid\targ=\targlow) = \lambda 
\end{equation}
for some $\lambda \in [0,1]$. Thus $\lambda$ captures the degree of ``error'' or lack of separation among the lists: if the lists perfectly separate the groups, then $\lambda=0$. We leave $\Pr(\listmem_{\targlow}=1\mid \targ=\targlow)$ unrestricted.

Formally, we assume:
\begin{assumption}[Homogeneous non-separation]\label{a-homog-nonsep}
  Let $\Pimat$ be the $\dimtarg\times \dimtarg$ matrix of $\pi_{\targlow',\targlow}=\Pr(\listmem_{\targlow'}=1\mid \targ=\targlow)$ with diagonal $\pi_{\targlow}=\pi_{\targlow,\targlow}$. Assume:
\begin{enumerate}
    \item $\Pimat$ is invertible
    \item $\Pimat = \bm{\Omega} + \lambda \bm{1}_{\dimtarg}\bm{1}_{\dimtarg}^\top$ where $\bm{\Omega}=\text{diag}(\omega_1,...,\omega_{\dimtarg})$ with $\omega_\targlow=\pi_\targlow-\lambda$ and  $\lambda$ satisfying $\lambda < \min_{\targlow}(\pi_{\targlow})$ so that $\omega_{\targlow}>0$ for all $\targlow$.
\end{enumerate}
\end{assumption}

Strictly speaking, the requirement on $\lambda$ could be relaxed to requiring $1+\sum_{\targlow}\frac{1}{\omega_{\targlow}}\neq 0$ and $\omega_{\targlow}\neq 0$ for all $\targlow$, but the above restriction will often be reasonable as we think of $\lambda$ as representing a kind of error rate that we hope would not be so large as to be larger than the rate of non-errors $\pi_{\targlow}$. In Appendix~\ref{app-sensitivityproofs}, we prove that under this assumption, we can solve for $\pi_{\targlow}$ given any known $\lambda$ to get:
\begin{equation}\label{app-eq-georecovery-sensitivityresult}
\pi_{\targlow} =  \lambda + \frac{1-\lambda \bm{1}_{\dimtarg}^\top (\betamat^\top\betamat)^{-1}\betamat^\top\bm{1}_{\dimgeo} }{\left(\betamat^\top\betamat)^{-1}\betamat^\top\bm{1}_{\dimgeo}\right)_{\targlow}}
\end{equation}
Here we assume that $\left(\betamat^\top\betamat)^{-1}\betamat^\top\bm{1}_{\dimgeo}\right)_{\targlow}\neq 0$. If all assumptions used to derive this expression, this must hold since $\pi_r$ is a probability in $[0,1]$. However, in practice, especially if some assumptions hold imperfectly, one should check that the denominator is not $0$ -- if it is, this indicates an assumption violation. The same is true for similar expressions below.

By setting $\lambda=0$ and calculating $\Pr(L_{\targlow}=1\mid \geo)/\pi_{\targlow}$, we obtain our main estimate of $\Pr(R=1\mid \geo)$ from equation~\eqref{eq-georecoverysolution} in the main text.

\subsection{Recovery and sensitivity analysis given coarse group geographic prevalence}\label{app-subgroup-geoprevsensitivityresults}

Next, we consider the regime where coarse group prevalence $\Pr(\coarsegrp=\coarsegrplow\mid \geo)$ is available but $\Pr(\targ = \targlow\mid \geo)$ is not. Stacking equation~\eqref{eq-geo-recover2} in the main text over geographies and groups gives the following matrix form, 
\begin{equation}\label{app-matrixform-georec}
\betamat = \bm{\Theta}\,\Pimat,   
\end{equation}
where $\betamat$ is the $\dimgeo\times\dimtarg$ matrix of list rates $\Pr(\listmem_{\targlow}=1\mid\geo)$, $\bm{\Theta}$ is the $\dimgeo\times\dimtarg$ matrix of subgroup geography proportions $\Pr(\targ=\targlow\mid\geo)$, and $\Pimat$ is the $\dimtarg\times\dimtarg$ matrix of within-subgroup list rates $\pi_{\targlow',\targlow}=\Pr(\listmem_{\targlow'}=1\mid\targ=\targlow)$. We also index the $\targ$ groups and order each of these matrices so that $\targ=\targlow$ levels belonging to the same $\coarsegrp=\coarsegrplow$ coarse group are arranged in blocks.

Unlike in Assumption~\ref{a-georank}, we now
let $\mmat$ be the the $\dimgeo\times\dimcoarsegrp$ matrix of \textit{coarse} group probabilities $\Pr(\coarsegrp=\coarsegrplow\mid \geo=\geolow)$. We also let $\bm{U}$ be a $\dimtarg\times \dimcoarsegrp$ block-diagonal indicator matrix with, in each column $\coarsegrplow$, a vector $\bm{1}_{\dimtarg(\coarsegrplow)} $ padded by $0$'s (where $\bm{1}_d$ denotes a $d$-dimensional vector of ones) that encodes the $\targlow$ values such that $\coarsegrplow(\targlow)=\coarsegrplow$. Then, equation~\eqref{eq-subsumtogroup} in the main text can be written as,
\begin{equation}\label{eq-subsumtogroupmatversion}
  \bm{\Theta}\bm{U} = \mmat.
\end{equation}

Under our set-up, both $\betamat$ and $\mmat$ are identifiable and full rank by Assumptions~\ref{a-georank}~(applied to coarse group $M$)~and~\ref{a-betarank}.  Our goal is to recover $\bm{\Theta}$, but the key complication is that $\Pimat$ is also unknown. We again relax Assumption~\ref{a-listqual}(b) with Assumption~\ref{a-homog-nonsep}. Note that for now, we makes no distinction between the off-diagonal entries for $\targlow,\targlow'$ in the same coarse group or different coarse groups.

Under these assumptions, we prove in Section~\ref{app-sensitivityproofs} that
\begin{align}
\pi_\targlow & = \lambda + \frac{\left(  1- \lambda \bm{1}_{\dimtarg} ^\top (\betamat^\top\betamat)^{-1}\betamat^\top\mmat \bm{1}_{\dimcoarsegrp} \right)}{((\betamat^\top\betamat)^{-1}\betamat^\top \mmat \bm{1}_{\dimcoarsegrp})_{\targlow} }  \label{eq-pi-result}
\end{align}
where if $\lambda =0$, this gives
\begin{equation}
    \pi_\targlow = \frac{1}{((\betamat^\top\betamat)^{-1}\betamat^\top \mmat \bm{1}_{\dimcoarsegrp})_{\targlow}}
\end{equation}

Alternatively, we can relax Assumption~\ref{a-homog-nonsep} to allow $\bm{\Pi}$ to have different cross-rates in different coarse groups $\coarsegrplow$. Specifically, suppose we relax part 2 of this assumption to
\begin{equation*}
    \Pr(\listmem_{\targlow} = 1\mid \targ=\targlow) =\left\{\begin{array}{cc}
      \lambda   & \text{if } \targlow\neq \targlow' \text{ and } \coarsegrplow(\targlow)=\coarsegrplow(\targlow') \\
      \gamma   &  \text{if } \targlow\neq \targlow' \text{ and } \coarsegrplow(\targlow)\neq \coarsegrplow(\targlow')
    \end{array}\right.
\end{equation*}

This says that the chance of a member of group $\targlow$ having their name on the list of another group within the same coarser group is $\lambda$ while the chance of their name being on a list for a group in a different coarser group is $\gamma$. Usually, we would expect $\gamma < \lambda$. We might even be willing to assume $\gamma = 0$ so that the coarse groups are separated even if the subgroups are not. 

In matrix form, this new $\Pimat$ can be written
\begin{equation*}
\Pimat = \bm{\Omega}+\bm{U}\bm{C}\bm{U}^\top
\end{equation*}
\noindent where again, $\bm{\Omega}$ is the diagonal matrix of $\pi_\targlow-\lambda$ and
\begin{equation}
    \bm{C} = (\lambda-\gamma) I_{\dimcoarsegrp}  + \gamma \bm{1}_{\dimcoarsegrp}\bm{1}_{\dimcoarsegrp}^\top
\end{equation}

Under this specification of $\Pimat$, we prove two results. First, if $\gamma=0$, then $\Pimat$ has a block diagonal form and all relevant equations are completely separable by $\coarsegrplow$. We then simply have equation \eqref{eq-pi-result} within each $\coarsegrplow$ group. That is, for subgroup $\targlow$ from course group $\coarsegrplow$,
\begin{align}
\pi_\targlow & = \lambda + \frac{\left(  1- \lambda \bm{1}_{\dimtarg(\coarsegrplow)} ^\top (\betamat_{\coarsegrplow}^\top\betamat_{\coarsegrplow})^{-1}\betamat_{\coarsegrplow}^\top\mmat_{\coarsegrplow} \right)}{((\betamat_{\coarsegrplow}^\top\betamat_{\coarsegrplow})^{-1}\betamat_{\coarsegrplow}^\top \mmat_{\coarsegrplow})_{\targlow} }  \label{eq-pi-result-blocked}
\end{align}
\noindent where $\betamat_\coarsegrplow$ is the $\dimgeo\times \dimtarg(\coarsegrplow)$ matrix of $\Pr(\listmem_{\targlow}=1\mid \geo)$ for $\targlow:\coarsegrplow(\targlow)=\coarsegrplow$; where $\mmat_\coarsegrplow$ is the vector of $\Pr(\coarsegrp=\coarsegrplow\mid \geo=\geolow)$ over $g$; and where we are solving for elements of $\bm{\Pi}_\coarsegrplow$, the $\dimtarg(\coarsegrplow) \times \dimtarg(\coarsegrplow)$ matrix of $\Pr(\listmem_{\targlow'}=1\mid \targ=\targlow)$ for $\targlow,\targlow'$ in the same coarse group. When also $\lambda=0$, this gives equation \eqref{eq-subgroup-geo-recov} in the main text. More generally, for $\gamma >0$ and $\gamma\neq \lambda$, 
\begin{equation}
   \pi_\targlow = \lambda + \frac{\left(\bm{1}_{\dimcoarsegrp} - \bm{C}\bm{U}^\top(\betamat^\top\betamat)^{-1}\betamat^\top\mmat\bm{1}_{\dimcoarsegrp}\right)_{\coarsegrplow}}{\left((\betamat^\top\betamat)^{-1}\betamat^\top\mmat\bm{1}_{\dimcoarsegrp}\right)_{\targlow}}\label{eq-generalized-sensitivity}
\end{equation} 
Note that the $\gamma=\lambda>0$ case is equation \eqref{eq-pi-result}.

\subsection{Implementation details under $\lambda=0$}
\label{app-nonnegative-ls}

To estimate $\Pr(\targ=\targlow\mid \geo=\geolow)$ in practice, we use non-negative least squares fit separately for each coarse group (as discussed above, the $\gamma=0$ case allows us to partition $\Pimat$ into a block for each coarse group $\coarsegrplow$).
Let $\rho_\targlow(\geolow)=\Pr(\listmem_\targlow = 1 \mid \geo = \geolow)$ and recall that $\pi_\targlow=\Pr(\listmem_\targlow=1\mid \targ=\targlow)$. Under $\lambda=0$, we have $\rho_\targlow(\geolow)=\pi_\targlow\Pr(\targ=\targlow\mid \geo=\geolow)$. Since the subgroups partition the coarse group, it follows that
\begin{equation}   \sum_{\targlow:\coarsegrplow(\targlow)=\coarsegrplow}\frac{\rho_\targlow(\geolow)}{\pi_\targlow} = \Pr(\coarsegrp=\coarsegrplow\mid \geo=\geolow).
\end{equation}
This equation is linear in the inverse coverages $u_\targlow := 1/\pi_\targlow$. For each $\coarsegrplow$, we therefore solve
\begin{equation*}
\hat{\bm{u}} \;=\; \arg\min_{\bm{u}\,\ge\,0}\; \sum_{\geolow} w_{\geolow} \Big( \sum_{\targlow:\coarsegrplow(\targlow)=\coarsegrplow} u_{\targlow} \rho_\targlow(\geolow) - \Pr(\coarsegrp=\coarsegrplow\mid \geo=\geolow) \Big)^{2},
\end{equation*}
where $\bm{u}\geq 0$ is a shorthand for all its elements being non-negative. We weight each geography by its expected number of coarse group members, $w_\geolow = n_\geolow \Pr(\coarsegrp=\coarsegrplow\mid \geo=\geolow)$, where $n_\geolow$ is the number of units in geography $\geolow$. This weight requires only the known coarse geographic distribution and not the unobserved $\targ_i$. The estimated coverage of list $\targlow$ is $\hat{\pi}_\targlow = 1/\hat{u}_\targlow$, meaning we could estimate $\Pr(\targ=\targlow\mid \geo)$ using the following initial estimator,
\begin{equation}
\widehat{\Pr}^{(ini)}(\targ=\targlow\mid \geo) = \hat{u}_\targlow\rho_\targlow(\geolow)=\rho_\targlow(\geolow)/\hat{\pi}_\targlow
\end{equation}
However, while their theoretical versions must sum to $\Pr(\coarsegrp=\coarsegrplow\mid \geo)$, these estimates need not sum to the known coarse group share. To impose this constraint, our final estimates are
\begin{equation}
  \widehat{\Pr}(\targ=\targlow\mid \geo) = \Pr(\coarsegrp=\coarsegrplow\mid \geo=\geolow)\,\frac{\widehat{\Pr}^{(ini)}(\targ=\targlow\mid \geo)}{\sum_{\targlow':\coarsegrplow(\targlow')=\coarsegrplow}\widehat{\Pr}^{(ini)}(\targ=\targlow'\mid \geo)},
\end{equation}
Because the fit is constrained to be non-negative, a subgroup whose geographic prevalence cannot be identified by the data is equal to zero.

\subsection{Proofs}\label{app-sensitivityproofs}

We start with a useful Lemma.
\begin{lemma}[]\label{lem-piinverse}
If $\Pimat = \bm{\Omega}+\lambda \bm{1}_{\dimtarg}\bm{1}_{\dimtarg}^\top$ for $\dimtarg\times\dimtarg$ invertible diagonal matrix $\bm{\Omega}$ with diagonal entries $\omega_{\targlow}$ and for $\lambda\in \mathbb{R}$ such that $1+\lambda \sum_{r=1}^{\dimtarg}\frac{1}{\omega_{\targlow}}\neq 0$, Then 
\begin{align*}
(\bm{\Omega}+\lambda \bm{1}_{\dimtarg}\bm{1}_{\dimtarg}^\top)^{-1}
\bm{1}_{\dimtarg} &=
\frac{\bm{\Omega}^{-1}\bm{1}_{\dimtarg}}{
1+\lambda\displaystyle\sum_{\targlow=1}^{\dimtarg}
\frac{1}{\omega_{\targlow}}}.
\end{align*}
\end{lemma}
\begin{proof}
By the Sherman-Morrison formula
\begin{align*}
(\bm{\Omega}+\lambda \bm{1}_{\dimtarg}\bm{1}_{\dimtarg}^\top)^{-1}
&=\bm{\Omega}^{-1}-\frac{\lambda\,\bm{\Omega}^{-1}\bm{1}_{\dimtarg}\bm{1}_{\dimtarg}^\top
\bm{\Omega}^{-1}
}{1+\lambda\bm{1}_{\dimtarg}^\top
\bm{\Omega}^{-1}\bm{1}_{\dimtarg}
}
\end{align*}
where $\bm{1}_{\dimtarg}^\top\bm{\Omega}^{-1}\bm{1}_{\dimtarg}=\sum_{\targlow}\frac{1}{\omega_{\targlow}}$. This then implies
\begin{align*}
(\bm{\Omega}+\lambda \bm{1}_{\dimtarg}\bm{1}_{\dimtarg}^\top)^{-1}\bm{1}_{\dimtarg}
&=\bm{\Omega}^{-1}\bm{1}_{\dimtarg}-\frac{\lambda\bm{\Omega}^{-1}\bm{1}_{\dimtarg}\left(\sum_{\targlow}\frac{1}{\omega_{\targlow}}\right)
}{1+\lambda \left(\sum_{\targlow}\frac{1}{\omega_{\targlow}}\right)
}\\
&=\bm{\Omega}^{-1}\bm{1}_{\dimtarg}\left(1-\frac{\lambda\left(\sum_{\targlow}\frac{1}{\omega_{\targlow}}\right)
}{1+\lambda \left(\sum_{\targlow}\frac{1}{\omega_{\targlow}}\right)
}\right)\\ 
&=\bm{\Omega}^{-1}\bm{1}_{\dimtarg}\left(\frac{1}{1+\lambda \left(\sum_{\targlow}\frac{1}{\omega_{\targlow}}\right)
}\right)
\end{align*}
\end{proof}

\begin{proof}[\textbf{Proof of equation \eqref{app-eq-georecovery-sensitivityresult}}]
Start $\mmat \Pimat = \betamat$ and $\mmat \bm{1}_{\dimtarg}=\bm{1}_{\dimgeo}$. Rearranging gives:
\begin{align*}
\mmat &= \betamat\Pimat^{-1}\\ 
\betamat\Pimat^{-1}\bm{1}_{\dimtarg} &= \bm{1}_{\dimgeo}\\ 
   \Pimat^{-1}\bm{1}_{\dimtarg} &= (\betamat^\top\betamat)^{-1}\betamat^\top \bm{1}_{\dimgeo}\\
  \frac{
\bm{\Omega}^{-1}\bm{1}_{\dimtarg}
}{
1+\lambda\displaystyle\sum_{\targlow=1}^{\dimtarg}
\frac{1}{\omega_{\targlow}}
}&= (\betamat^\top\betamat)^{-1}\betamat^\top\bm{1}_{\dimgeo} & (\text{Lemma }\ref{lem-piinverse})
\end{align*}
Since $\bm{\Omega}^{-1}\bm{1}_{\dimtarg}$ is simply the column vector of $1/\omega_{\targlow}$ values, this implies that for each $\targlow$,
\begin{equation}\label{eq-prf-keyomegaimp}
    \frac{1}{\omega_{\targlow}} = \left(1 + \lambda\sum_{\targlow=1}^{\dimtarg} \frac{1}{\omega_{\targlow}} \right)\left((\betamat^\top\betamat)^{-1}\betamat^\top\bm{1}_{\dimgeo}\right)_{\targlow}
\end{equation}
Finally, we can solve for the value of $\sum_{\targlow}\frac{1}{\omega_{\targlow}}$ by noting that
\begin{align*}
 \sum_{r=1}^{\dimtarg}\frac{1}{\omega_{\targlow}} = \left(1 + \lambda\sum_{\targlow=1}^{\dimtarg} \frac{1}{\omega_{\targlow}} \right)\bm{1}_{\dimtarg}^\top (\betamat^\top\betamat)^{-1}\betamat^\top\bm{1}_{\dimgeo}
\end{align*}
which, rearranged, implies
\begin{equation*}
   \sum_{r=1}^{\dimtarg} \frac{1}{\omega_{\targlow}} = \frac{\bm{1}_{\dimtarg}^\top (\betamat^\top\betamat)^{-1}\betamat^\top\bm{1}_{\dimgeo}}{1-\lambda \bm{1}_{\dimtarg}^\top (\betamat^\top\betamat)^{-1}\betamat^\top\bm{1}_{\dimgeo}}
\end{equation*}
Overall then, plugging this into equation~\eqref{eq-prf-keyomegaimp} and using $\omega_{\targlow}=\pi_{\targlow}-\lambda$, we have
\begin{align*}
\pi_{\targlow} &= \lambda + \left(1 + \lambda \frac{\bm{1}_{\dimtarg}^\top (\betamat^\top\betamat)^{-1}\betamat^\top\bm{1}_{\dimgeo}}{1-\lambda \bm{1}_{\dimtarg}^\top (\betamat^\top\betamat)^{-1}\betamat^\top\bm{1}_{\dimgeo}} \right)^{-1}\left((\betamat^\top\betamat)^{-1}\betamat^\top\bm{1}_{\dimgeo}\right)_{\targlow}^{-1}\\
&= \lambda + \left(\frac{1}{1-\lambda \bm{1}_{\dimtarg}^\top (\betamat^\top\betamat)^{-1}\betamat^\top\bm{1}_{\dimgeo}} \right)^{-1}\left((\betamat^\top\betamat)^{-1}\betamat^\top\bm{1}_{\dimgeo}\right)_{\targlow}^{-1}\\
&= \lambda + \frac{1-\lambda \bm{1}_{\dimtarg}^\top (\betamat^\top\betamat)^{-1}\betamat^\top\bm{1}_{\dimgeo} }{\left(\betamat^\top\betamat)^{-1}\betamat^\top\bm{1}_{\dimgeo}\right)_{\targlow}}
\end{align*}
\end{proof}

\begin{proof}[\textbf{Proof of equation~\eqref{eq-pi-result}}]
The steps of this proof are similar to the previous, only slightly more complicated. Start with $\betamat= \bm{\Theta}\,\Pimat$ and $\bm{\Theta}\bm{U} = \mmat$. Under the invertibility of $\Pimat$ and $\betamat$ from Assumptions~\ref{a-betarank}~and~\ref{a-homog-nonsep} and the form of $\Pimat$ under Assumption~\ref{a-homog-nonsep}, these imply
\begin{align*}
\bm{\Theta} &= \betamat \Pimat^{-1}\\
     \betamat \Pimat^{-1} \bm{U} &= \mmat \\
     \Pimat^{-1} \bm{U} &= (\betamat^\top\betamat)^{-1}\betamat^\top\mmat\\
     \Pimat^{-1} \bm{U}\bm{1}_{\dimcoarsegrp} &= (\betamat^\top\betamat)^{-1}\betamat^\top\mmat \bm{1}_{\dimcoarsegrp}\\
    \Pimat^{-1}\bm{1}_{\dimtarg} &= (\betamat^\top\betamat)^{-1}\betamat^\top\mmat \bm{1}_{\dimcoarsegrp} &(\bm{U}\bm{1}_{\dimcoarsegrp}=\bm{1}_{\dimtarg})\\
\frac{\bm{\Omega}^{-1}\bm{1}_{\dimtarg}}{
1+\lambda\displaystyle\sum_{\targlow=1}^{\dimtarg}
\frac{1}{\omega_{\targlow}}} &= (\betamat^\top\betamat)^{-1}\betamat^\top\mmat\bm{1}_{\dimcoarsegrp}  &(\text{Lemma \ref{lem-piinverse}})
\end{align*}
We then again have
\begin{equation*}
\frac{1}{\omega_{\targlow}} = \left(1 + \lambda\sum_{\targlow'=1}^{\dimtarg} \frac{1}{\omega_{\targlow'}} \right)\left((\betamat^\top\betamat)^{-1}\betamat^\top\mmat \bm{1}_{\dimcoarsegrp} \right)_{\targlow}
\end{equation*}
\noindent and can again solve for $\sum_{\targlow=1}^{\dimtarg} \frac{1}{\omega_{\targlow}}$ in the same way as in the previous proof to get
\begin{align*}
\pi_\targlow & = \lambda + \left(1 + \lambda \frac{\bm{1}_{\dimtarg} ^\top (\betamat^\top\betamat)^{-1}\betamat^\top\mmat \bm{1}_{\dimcoarsegrp}}{1- \lambda \bm{1}_{\dimtarg} ^\top (\betamat^\top\betamat)^{-1}\betamat^\top\mmat \bm{1}_{\dimcoarsegrp}} \right)^{-1} \left\{((\betamat^\top\betamat)^{-1}\betamat^\top \mmat \bm{1}_{\dimcoarsegrp})_{\targlow} \right\}^{-1}\\
& = \lambda + \frac{\left(  1- \lambda \bm{1}_{\dimtarg} ^\top (\betamat^\top\betamat)^{-1}\betamat^\top\mmat \bm{1}_{\dimcoarsegrp} \right)}{((\betamat^\top\betamat)^{-1}\betamat^\top \mmat \bm{1}_{\dimcoarsegrp})_{\targlow} } 
\end{align*}
\end{proof}

\begin{proof}[\textbf{Proof of equation \eqref{eq-pi-result-blocked}}] Suppose $\gamma =0$. Then $\Pimat$ consists of blocks for each coarse group, each of which takes exactly the form of the overall $\Pimat$ in the previous calculations. Moreover, each $\sum_{\targlow:\coarsegrplow(\targlow)=\coarsegrplow}\Pr(\coarsegrp=\coarsegrplow\mid \geo)=\Pr(\coarsegrp=\coarsegrplow\mid \geo)$ is entirely separate from those for other $\coarsegrplow'$, so the entire system of constraints applies to and can be solved only within each coarse group block. Hence the result follows immediately from doing the calculations above within each coarse group block, with $\mmat$ now a column vector and $\bm{1}_{\dimcoarsegrp}$ replaced by a scalar $1$. 
\end{proof}

\begin{proof}[\textbf{Proof of equation \eqref{eq-generalized-sensitivity}}]
Suppose $\gamma>0$ and $\gamma\neq \lambda$. We again start with
\begin{align*}
\Pimat^{-1}\bm{U} = (\betamat^\top\betamat)^{-1}\betamat^\top\mmat 
\end{align*}
\noindent Rearranging this equation and substituting in the form of $\bm{\Pi}$ yields
\begin{align*}
    \bm{U} &= \bm{\Pi} (\betamat^\top\betamat)^{-1}\betamat^\top\mmat\\
   \bm{U} &= (\bm{\Omega}+\bm{U}\bm{C}\bm{U}^\top)(\betamat^\top\betamat)^{-1}\betamat^\top\mmat\\
   \bm{U} &= \bm{\Omega}(\betamat^\top\betamat)^{-1}\betamat^\top\mmat + \bm{U}\bm{C}\bm{U}^\top(\betamat^\top\betamat)^{-1}\betamat^\top\mmat\\
      \bm{U}(I_{\dimcoarsegrp} - \bm{C}\bm{U}^\top(\betamat^\top\betamat)^{-1}\betamat^\top\mmat))&= \bm{\Omega}(\betamat^\top\betamat)^{-1}\betamat^\top\mmat
\end{align*}
The left side of this equation consists of entirely known quantities. The right side includes multiplication by a diagonal matrix $\bm{\Omega}$. Consider some $\targlow$ such that $\coarsegrplow=\coarsegrplow(\targlow)$. On the left, row $\targlow$ of $\bm{U}$ is  $e_{\coarsegrplow}$, which denotes the $\coarsegrplow^{th}$ length $\dimcoarsegrp$ standard basis vector. On the right, the fact that $\bm{\Omega}$ is diagonal means that row $\targlow$ corresponds to row $\targlow$ of the matrix $(\betamat^\top\betamat)^{-1}\betamat^\top\mmat$ multiplied by $\omega$. Hence for any column $j=1,...,\dimcoarsegrp$, we have
\begin{align*}
    \omega_{\targlow} \left((\betamat^\top\betamat)^{-1}\betamat^\top\mmat\right)_{\targlow,j}  &= e_{\coarsegrplow}^\top \left(I_{\dimcoarsegrp} - \bm{C}\bm{U}^\top(\betamat^\top\betamat)^{-1}\betamat^\top\mmat\right)_{\cdot  j}\\
    &=\left(I_{\dimcoarsegrp} - \bm{C}\bm{U}^\top(\betamat^\top\betamat)^{-1}\betamat^\top\mmat\right)_{\coarsegrplow,  j}
\end{align*}
Summing each side, this implies
\begin{align*}
    \omega_{\targlow} \left((\betamat^\top\betamat)^{-1}\betamat^\top\mmat\bm{1}_{\dimcoarsegrp}\right)_{\targlow} &=\left(\left(I_{\dimcoarsegrp} - \bm{C}\bm{U}^\top(\betamat^\top\betamat)^{-1}\betamat^\top\mmat\right)\bm{1}_{\dimcoarsegrp}\right)_{\coarsegrplow}\\
     &=\left(\bm{1}_{\dimcoarsegrp} - \bm{C}\bm{U}^\top(\betamat^\top\betamat)^{-1}\betamat^\top\mmat\bm{1}_{\dimcoarsegrp}\right)_{\coarsegrplow}
\end{align*}
As long as the left entry is non-zero, and using $\omega_{\targlow} = \pi_\targlow -\lambda$ this can be rearranged to
\begin{equation}
    \pi_\targlow = \lambda + \frac{\left(\bm{1}_{\dimcoarsegrp} - \bm{C}\bm{U}^\top(\betamat^\top\betamat)^{-1}\betamat^\top\mmat\bm{1}_{\dimcoarsegrp}\right)_{\coarsegrplow}}{\left((\betamat^\top\betamat)^{-1}\betamat^\top\mmat\bm{1}_{\dimcoarsegrp}\right)_{\targlow}}
\end{equation}
\end{proof}

\section{Exact Prompts Used}\label{app-prompts}

All lists are generated by Gemini~3.0 Flash at temperature $1.0$, with six calls for each prompt and $300$ surnames requested per call ($500$ full names in Lebanon), ranked by the number of calls returning each surname. Every prompt substitutes a group description into a fixed template.

\subsection{United States voter files} 

The template is:
\begin{quote}\emph{List \{n\} common SURNAMES (last names only) of \{description\} in the United States. Maximize diversity; avoid famous/celebrity surnames. One surname per line, UPPERCASE, no first names, no numbering, no commentary.}\end{quote}
with the following descriptions:
\begin{itemize}
\item White:
\begin{itemize}
\item ``non-Hispanic White Americans of English or Irish descent''
\item ``non-Hispanic White Americans of German descent''
\item ``non-Hispanic White Americans of Italian, Polish, or other European descent''
\end{itemize}
\item Black:
\begin{itemize}
\item ``Black / African American people (ordinary, everyday)''
\item ``Black Americans from the US South''
\item ``a diverse set of Black Americans, ordinary people and not celebrities''
\end{itemize}
\item Hispanic:
\begin{itemize}
\item ``Mexican Americans''
\item ``Cuban and Puerto Rican Americans''
\item ``Central and South American (Salvadoran, Colombian, Dominican) people in the US''
\end{itemize}
\item Asian:
\begin{itemize}
\item ``Chinese Americans''
\item ``Asian Indian Americans''
\item ``Vietnamese and Filipino Americans''
\item ``Korean and Japanese Americans''
\end{itemize}
\end{itemize}

\subsection{Historical census} 

We use Gemini~3.0 Flash to write a meta prompt that asks the model to write the descriptions itself:
\begin{quote}\small\emph{You are helping build surname lists for \{origin\} immigrants to the United States around 1900. Propose exactly three short descriptions of subgroups that together cover that immigrant population, distinguished by the region they came from or the region of America where they settled. Each description must complete the sentence `list 300 common surnames (last names only) of ...'. Return only the three descriptions, one per line, with no numbering.}\end{quote}
The three descriptions returned for each origin are substituted into the template:
\begin{quote}\emph{List \{n\} common SURNAMES (last names only) of \{description\}. Maximize diversity; avoid famous/celebrity surnames. One surname per line, UPPERCASE, no first names, no numbering, no commentary.}\end{quote}
For the Chinese and Japanese lists of Section~\ref{sec-val-asian}, the meta prompt refers to the population ``in the United States and Hawaii around 1900'', and the template adds ``spelled as they were written in American census and immigration records of that period''. For the white, Black, and American Indian lists of Section~\ref{sec-val-asian}, the meta prompt refers to ``white Americans'', ``Black Americans'', and ``American Indians'' ``in the United States around 1900'', and the descriptions are substituted into the template above. Table~\ref{tab:census-prompts} lists every description used.

\begin{longtable}{p{2.0cm}p{12.6cm}}
\hline Group & Descriptions \\ \hline \endfirsthead
\hline Group & Descriptions \\ \hline \endhead
Chinese & ``Cantonese families from the Pearl River Delta settling in California and the American West'' \\
 & ``Taishanese and Sze Yup laborers migrating to the sugar plantations of the Hawaiian Islands'' \\
 & ``Fujianese and Hokkien merchants and sailors establishing communities in Atlantic port cities like New York'' \\[5pt]
Japanese & ``Japanese immigrants from Hiroshima and Yamaguchi prefectures settling in Hawaii'' \\
 & ``Japanese immigrants from Kyushu and Okinawa settling in the United States and Hawaii'' \\
 & ``Japanese immigrants from Central and Northern Honshu settling in the Pacific Northwest and California'' \\[5pt]
White & ``white Americans in the Northeast and Midwest primarily of British, German, and Irish descent'' \\
 & ``white Americans in the South primarily of English, Scots-Irish, and French descent'' \\
 & ``white Americans in the West and industrial cities primarily of Scandinavian, Central European, and Southern European descent'' \\[5pt]
Black & ``Black Americans residing in the rural Southeastern United States'' \\
 & ``Black Americans living in Northern and Midwestern industrial cities'' \\
 & ``Black Americans in the coastal Lowcountry and Sea Islands regions'' \\[5pt]
American Indian & ``American Indians originally from the Eastern Woodlands and Southeastern United States'' \\
 & ``American Indians from the Great Plains and Great Lakes regions'' \\
 & ``American Indians from the Southwest, Pacific Northwest, and California regions'' \\[5pt]
\hline
\caption{Model-written descriptions for the historical census lists.}\label{tab:census-prompts}
\end{longtable}

\subsection{Lebanon} 
Each prompt embeds an English group description and four to five example full names, held fixed within the sect, into an Arabic template. In English, the template reads:
\begin{quote}\emph{Write \{n\} full names of people from \{descriptor\}. Examples: \{examples\}. Avoid the family names of famous politicians and celebrated families. Write varied, ordinary family names. Write exactly one name per line, with no numbers, no explanation, and no commentary. (Rendered from the Arabic; the original prompts appear in the replication archive.)}\end{quote}
The example names are ordinary full names of the sect. Table~\ref{tab:lebanon-prompts} lists every description used.

\begin{longtable}{p{5.6cm}p{8.9cm}}
\caption{Group descriptions used in the Lebanese sect prompts.}\label{tab:lebanon-prompts}\\
\hline Prompt stem & Completions \\ \hline \endfirsthead
\hline Prompt stem & Completions \\ \hline \endhead
Lebanese Shia Muslim families from [x] & southern villages (Nabatieh, Tyre rural areas — ordinary families, NOT famous religious or political figures) \\
 & Baalbek-Hermel and the northern Beqaa (ordinary villagers) \\
 & the Beirut southern suburbs (Dahieh) — everyday families \\[3pt]
Lebanese Shia Muslim families from [x] (ordinary families, not famous people) & the Nabatieh and Bint Jbeil districts \\
 & Beirut's southern suburbs \\
 & the Baalbek-Hermel countryside \\
 & the Tyre and Sarafand coast \\
 & the Marjeyoun and Khiam area \\
 & the Jezzine and Iqlim al-Tuffah villages \\
 & the Bekaa villages around Hermel \\
 & the Shia families of Baalbek city \\
 & the villages of the Bint Jbeil highlands \\
 & the Sidon and Zahrani hinterland \\
 & the Beirut families of Chiyah and Ghobeiry \\
 & the southern border villages near Aitaroun and Aita al-Shaab \\
 & the villages of Kfar Roummane and Zawtar \\
 & the Shia families of Baalbek's Sharawneh quarter \\[6pt]
Lebanese Sunni Muslim families from [x] & rural Akkar district (northern Lebanon, ordinary villagers — NOT politicians or famous figures) \\
 & Tripoli and the northern coastal region (ordinary residents) \\
 & Beirut and the Beqaa valley (everyday families, not political dynasties) \\[3pt]
Lebanese Sunni Muslim families from [x] (ordinary families, not famous people) & the city of Tripoli and its old quarters \\
 & Beirut's Tariq al-Jadida and Basta \\
 & the Iqlim al-Kharrub and Sidon countryside \\
 & the western Bekaa and Rashaya \\
 & the Akkar plain and Halba \\
 & the Minieh-Dinnieh district \\
 & the old Sidon families \\
 & the Arqoub and Shebaa area \\
 & the Bekaa town of Majdal Anjar and Saadnayel \\
 & the coastal villages between Batroun and Chekka \\
 & the Beirut families of Mazraa and Ras al-Nabaa \\
 & the Hermel and northern Bekaa Sunni villages \\[6pt]
Lebanese Maronite Christian families from [x] & the Kesrouan and Metn districts in Mount Lebanon (ordinary families) \\
 & Bsharri, Zgharta, and the northern mountain region \\
 & the Chouf and southern Mount Lebanon (everyday families) \\[3pt]
Lebanese Maronite Christian families from [x] (ordinary families, not famous people) & the Bsharri and Qadisha valley villages \\
 & the Kesrouan and Jbeil hills \\
 & the Metn and upper Matn villages \\
 & the Zgharta and Koura districts \\
 & the Batroun highlands \\
 & the Jezzine and southern Maronite villages \\
 & the Baabda and Aley Christian villages \\
 & the Ehden and Zawiya families \\
 & the Aakoura and Laqlouq mountain villages \\
 & the Zahle and Bekaa Maronite families \\
 & the Ftouh Kesrouan coastal towns \\
 & the Chouf Maronite villages \\
 & the Tannourine and Hadath al-Jebbeh villages \\
 & the Kfarsghab and Hasroun families \\
 & the Kartaba and Mayfouq villages \\
 & the Ghosta and Harissa families \\
 & the Bikfaya and Dhour Choueir families \\
 & the southern villages of Ain Ebel and Rmeish \\
 & the Deir al-Ahmar and Baalbek Maronite families \\
 & the Douma and Tannourine highland families \\[6pt]
Lebanese Greek (Roman) Orthodox Christian families from [x] & the Koura district in northern Lebanon (ordinary families, NOT famous politicians) \\
 & Beirut's Achrafieh and Ras Beirut quarters \\
 & Tripoli, El Mina, and surrounding northern coastal areas \\[3pt]
Lebanese Greek Orthodox Christian families from [x] (ordinary families, not famous people) & the Koura district villages \\
 & Beirut's Achrafieh and Rmeil \\
 & the Marjeyoun and Hasbaya area \\
 & the Zahle and central Bekaa \\
 & the city of Tripoli's Orthodox families \\
 & the Amioun and Kfaraakka villages \\
 & the Metn Orthodox villages \\
 & the Akkar Orthodox villages \\
 & the Shwayfat and Aley Orthodox families \\
 & the Beirut families of Mousaitbeh and Ras Beirut \\
 & the Batroun and Hamat coastal families \\
 & the Rashaya and Wadi al-Taym Orthodox villages \\
 & the villages of Kousba and Bterram \\
 & the Anfeh and Ras Masqa coastal villages \\
 & the Beit Mery and Baabdat families \\[6pt]
Lebanese Druze families from [x] & the Chouf district villages (ordinary villagers — NOT the Jumblatt, Arslan, or Hamadeh political families, NOT famous figures, ordinary common surnames) \\
 & Aley and Baabda mountain villages (everyday families) \\
 & the Hasbaya and Rashaya districts (southeastern Lebanon) \\[3pt]
Lebanese Druze families from [x] (ordinary families, not famous people) & the Chouf mountain villages \\
 & the Aley and Upper Metn district \\
 & the Hasbaya and Rashaya area \\
 & the Wadi al-Taym villages \\
 & the Baakline and Deir al-Qamar families \\
 & the Shwayfat and Choueifat area \\
 & the Hermon foothill villages \\
 & the Aley district towns of Bhamdoun and Souk al-Gharb \\
 & the Iqlim al-Kharrub Druze villages \\
 & the Mount Lebanon villages of Ain Dara \\
 & the Metn Druze families \\
 & the Rashaya al-Wadi villages \\
 & the Mukhtara and Barouk villages \\
 & the Ain Zhalta and Niha families \\
 & the Abey and Baysour villages \\
 & the Ras al-Metn and Salima families \\
 & the Kfarmatta and Aitat villages \\
 & the Btater and Kayfoun families \\
 & the Jahliyeh and Semqaniyeh villages \\
 & the Aramoun and Bennay families \\
 & the Hasbaya villages of Shwayya and Mimess \\[6pt]
Lebanese Melkite Greek Catholic (Roman Catholic) families from [x] & the Zahle area in the Beqaa valley (ordinary families) \\
 & the Metn district and Beirut \\
 & southern Lebanon (Jezzine area) and other regions \\[3pt]
Lebanese Greek Catholic (Melkite) families from [x] (ordinary families, not famous people) & the Zahle and Furzol area \\
 & the Sidon and Jezzine hinterland \\
 & the Beirut Melkite families \\
 & the Rashaya and western Bekaa \\
 & the Marjeyoun and Deir Mimas villages \\
 & the Baalbek Melkite families \\
 & the Chouf Melkite villages \\
 & the Tyre and southern Melkite families \\
 & the Metn Melkite families \\
 & the Aley and Bhamdoun Melkite families \\
 & the Baabda district Melkite families \\
 & the Bekaa villages around Rayak and Talia \\
 & the Zahle quarters of Maalaka and Karak Nuh \\
 & the Jezzine villages of Roum and Kfarhouna \\
 & the Ras Baalbek and Qaa Melkite families \\[6pt]
\multicolumn{2}{p{14.2cm}}{``Lebanese Armenian Orthodox families living in Bourj Hammoud and Beirut. Surnames are Armenian in origin but spelled in Arabic letters as used in Lebanese voter registries. Provide diverse Armenian family names (NOT just famous figures), and include common Armenian surname suffixes -ian and -oun and -jian rendered in Arabic''} \\[3pt]
Lebanese Armenian Orthodox families from [x] & the Antelias and Anjar areas (Armenian-origin surnames in Arabic script) \\[3pt]
\multicolumn{2}{p{14.2cm}}{``Lebanese Armenian Orthodox families with surnames ending in the typical Armenian patronymic suffixes (-yan, -ian) transliterated to Arabic script''} \\[3pt]
Lebanese Armenian Apostolic families from [x] (ordinary families, not famous people) & the Bourj Hammoud quarter \\
 & the Anjar settlement in the Bekaa \\
 & Beirut's Nor Marash and Nor Sis quarters \\
 & the Antelias and Metn Armenian families \\
 & the Zalka and Jdeideh Armenian families \\
 & the Tripoli Armenian community \\
 & the Ainjar and Bekaa Armenian villages \\
 & the Armenian families of Dora and Nabaa \\
 & the Armenian families descended from Cilicia and Adana \\
 & the Armenian families descended from Marash and Aintab \\
 & the Armenian families descended from Urfa and Kharpert \\
 & the Bourj Hammoud districts of Sanjak and Arax \\[6pt]
\hline
\end{longtable}

\section{Additional Empirical Results}

\subsection{Census-based name list generation}

\begin{table}[H]\centering
\begin{tabular}{lrrrrrrrrr}
\hline
 & & \multicolumn{2}{c}{White} & \multicolumn{2}{c}{Black} & \multicolumn{2}{c}{Hispanic} & \multicolumn{2}{c}{Asian} \\
$\tau$ & Shared & Cov. & Prec. & Cov. & Prec. & Cov. & Prec. & Cov. & Prec. \\
\hline
$0.20$ & $7.0\%$   & $0.28$ & $0.79$ & $0.51$ & $0.37$ & $0.50$ & $0.82$ & $0.42$ & $0.60$ \\
$0.30$ & $0.04\%$  & $0.28$ & $0.79$ & $0.33$ & $0.46$ & $0.50$ & $0.82$ & $0.42$ & $0.64$ \\
$0.40$ & $0.00\%$  & $0.28$ & $0.79$ & $0.17$ & $0.61$ & $0.50$ & $0.82$ & $0.41$ & $0.65$ \\
$0.50$ & $0.00\%$  & $0.28$ & $0.79$ & $0.11$ & $0.75$ & $0.50$ & $0.83$ & $0.40$ & $0.73$ \\
$0.60$ & $0.00\%$  & $0.28$ & $0.79$ & $0.07$ & $0.85$ & $0.49$ & $0.83$ & $0.39$ & $0.74$ \\
\hline
\end{tabular}
\caption{Sensitivity of the Census-based lists to the probability floor $\tau$, where a surname
enters the list for group $\targlow$ when $\Pr(\targ=\targlow\mid \sur=\surlow)\geq\tau$. Coverage and precision are defined as
in Table~\ref{tab:lists}, computed on the Florida file. ``Shared'' is the share of registrants
holding a surname on more than one of the four lists. The floor used in the paper is $0.30$.}
\label{tab:threshold}
\end{table}

Table~\ref{tab:threshold} reports the Census lists under values of $\tau$ ranging from $0.20$ to $0.60$. The White, Hispanic, and Asian lists are largely unchanged. For these groups, the top $1{,}000$ surnames by the ranking rule clear even the highest value of $\tau$. The Black list is far more sensitive, clearly trading off coverage for precision as $\tau$ rises. $\tau$ also determines the extent to which the group lists have overlapping surnames: at $\tau = 0.20$, $7.0\%$ of registrants have a surname that belongs to two groups, while at $\tau = 0.30$, that share falls to $0.04\%$, and above $\tau = 0.30$, it is zero. We therefore use $\tau = 0.30$, the smallest measured floor at which the lists are close to disjoint.

\subsection{Robustness of the LLM lists}\label{app-listgen}

Figure \ref{fig:listsize} shows the improvement in ROC-AUC and PR-AUC of using $\ell$BISG for two different list generation methods and various list sizes. This is relative to simply using the geographic baseline $\Pr(\targ=\targlow\mid\geo)$. For comparison, the plot also includes the standard BISG probabilities which one would estimate directly if frequency data were available. The figure shows diminishing returns as list size increases in most but not all cases and that performance for the two list generation methods are similar. The lift over geographic baselines is largest for the Hispanic and Asian groups, though in all cases, it is positive. 

\begin{figure}[h]\centering
\includegraphics[width=\textwidth]{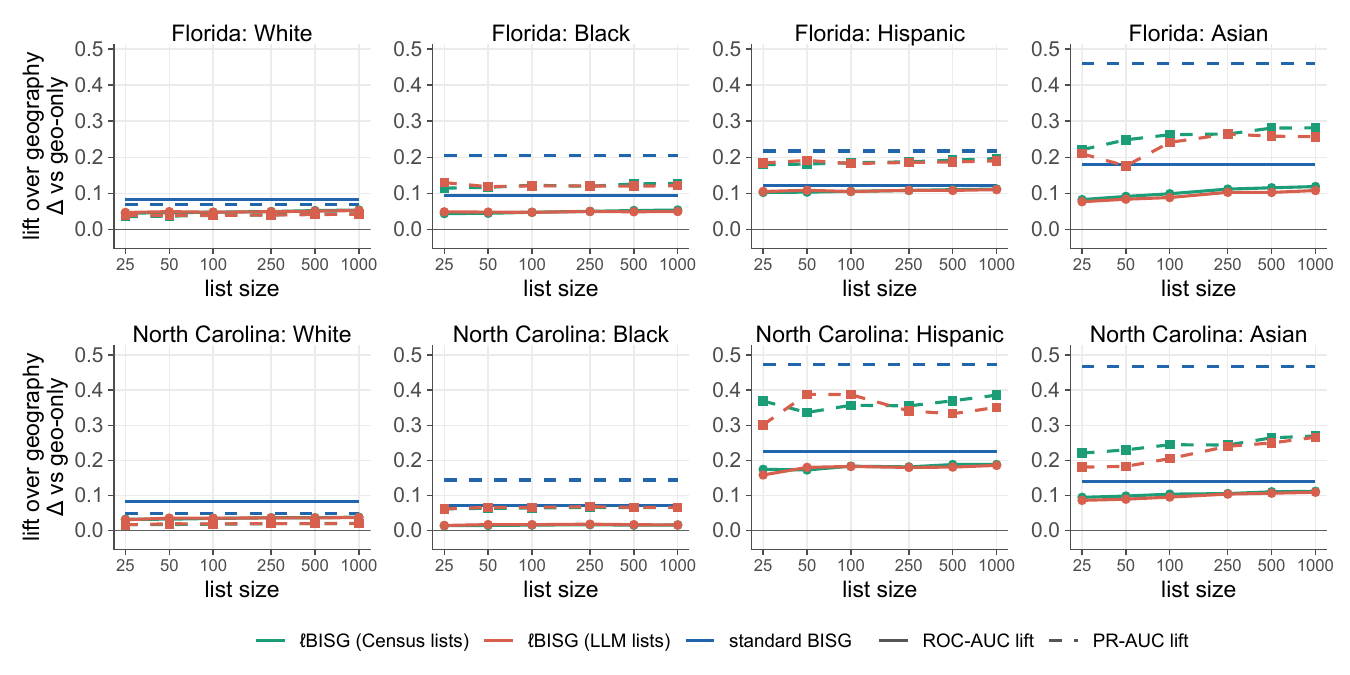}
\caption{Discrimination lift over geography alone, in ROC-AUC and PR-AUC, against the number of surnames per group. Calculated for the Florida (top row) and North Carolina (bottom row) voter files with published Census tract shares as the geographic baseline, for the Census and LLM lists of Table~\ref{tab:lists}. We use separate binary membership targets for each list. Standard BISG is included as a flat reference.}
\label{fig:listsize}
\end{figure}

\begin{figure}[h]\centering
\includegraphics[width=0.86\textwidth]{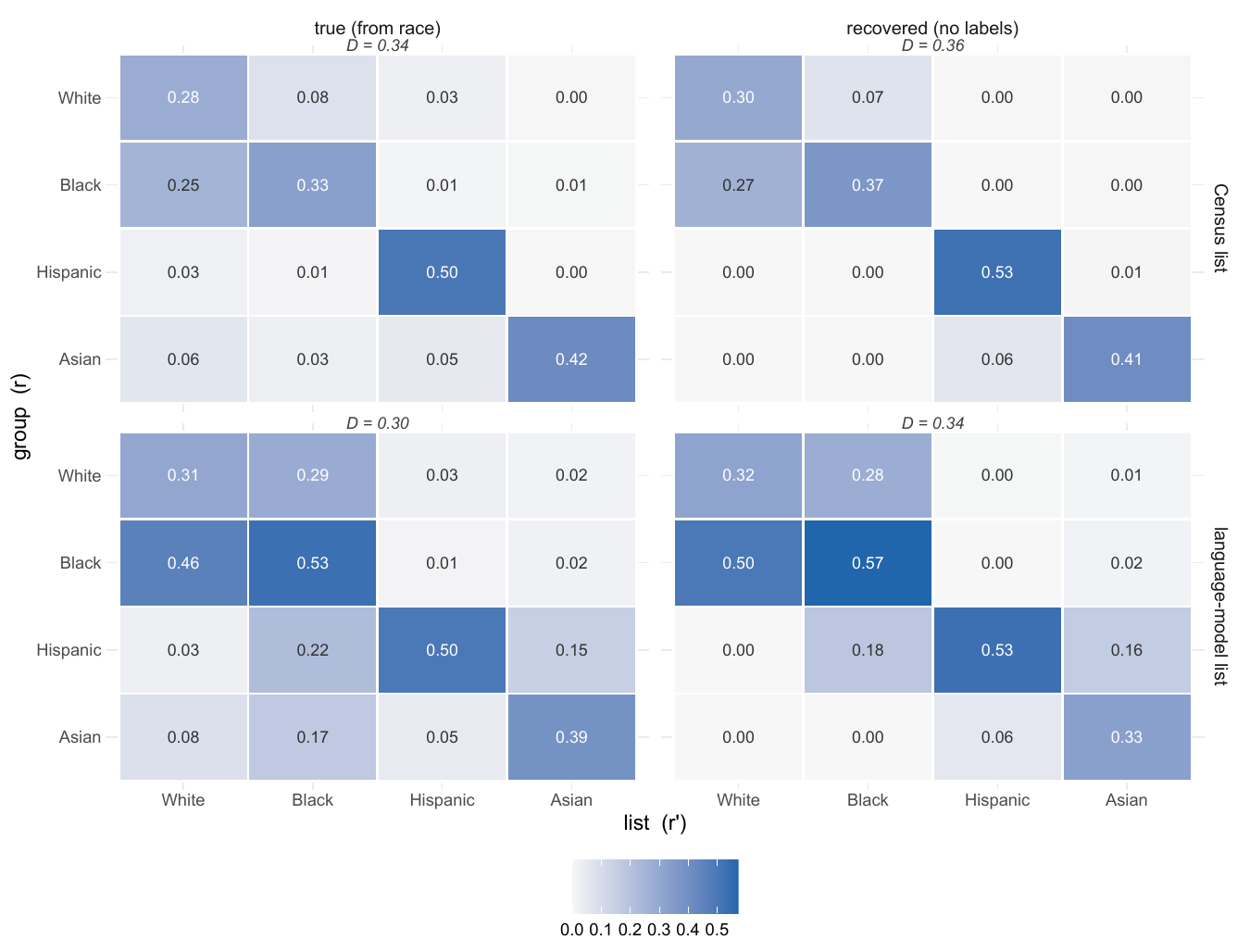}
\caption{List quality on the Florida voter file. Each cell is $\pi_{\targlow,\targlow'}$, the rate at
which members of group $\targlow$ (rows) has a name on list $\targlow'$ (columns), for the Census
lists (top) and the LLM lists (bottom). The left column computes $\Pimat$ from the self-reported race (i.e., ground-truth). The right column estimates $\Pimat$ from the tract-level list hit rates and the published tract race shares,
with no individual race labels. The number above each panel is $D(\Pimat)$.}
\label{fig:pi-quality}
\end{figure}

\begin{figure}[t]\centering
\includegraphics[width=\textwidth]{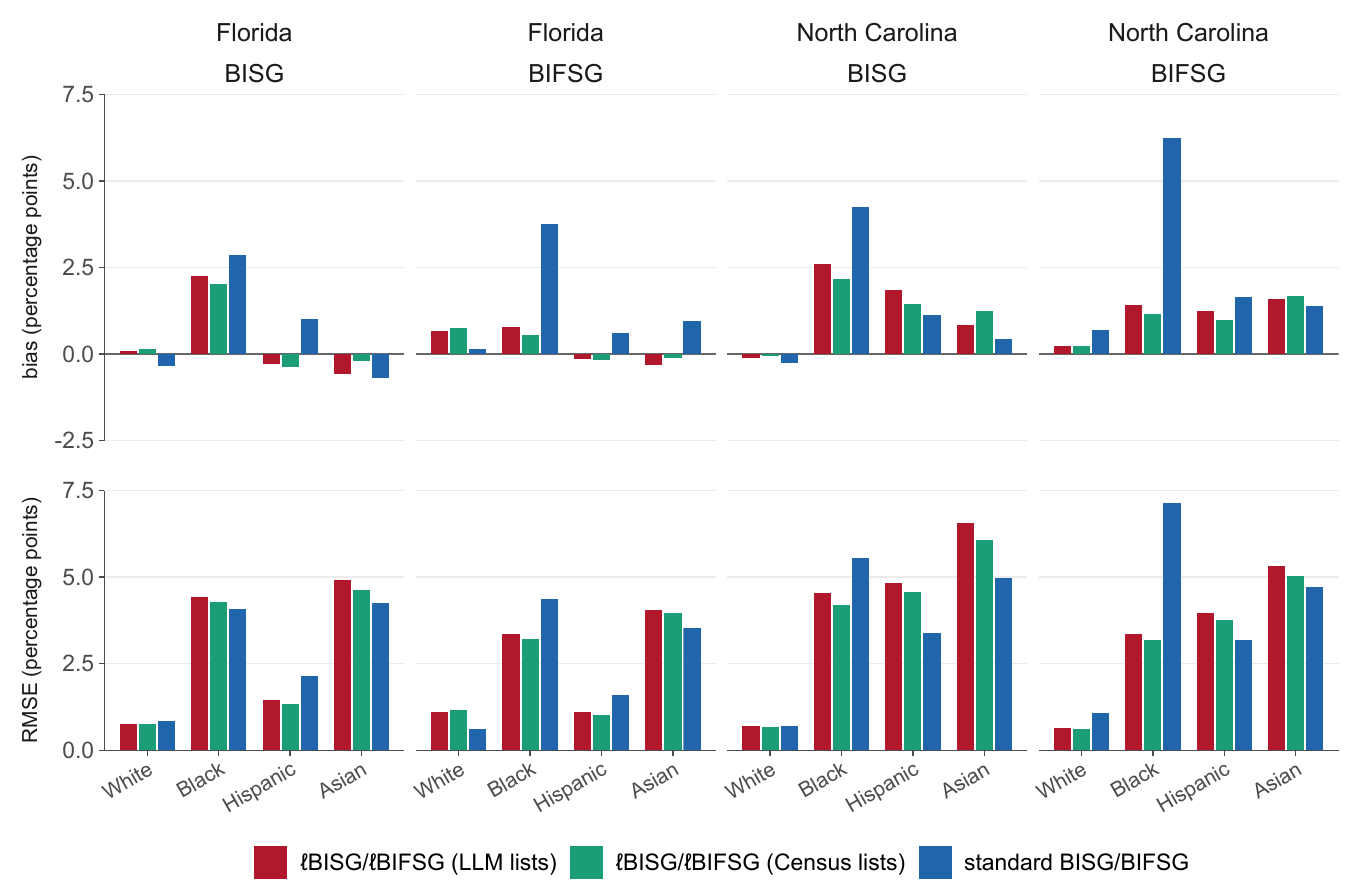}
\caption{County-level race-specific estimates of the Democratic share of registrants with the LLM lists (red), the Census lists (green), and standard BISG/BIFSG (blue), otherwise as Figure~\ref{fig:val-disparity}.}
\label{fig:val-disparity-census}
\end{figure}

\subsection{Combining $\ell$BISG with standard BISG}\label{app-combined}

Standard BISG assigns a generic geography prior to every surname that is absent from the Census surname table. Standard BIFSG similarly drops the first name updating for individuals whose first name is absent from the Census first name table.
In Table~\ref{tab:combined}, we compare standard BISG and BIFSG with BISG and BIFSG that use $\ell$BISG for unlisted names (BISG + $\ell$BISG and BIFSG + $\ell$BISG). Among registrants whose surname is missing from the table, BISG + $\ell$BISG improves average precision for every group in both states. The improvement is largest for Hispanic registrants. Probabilities for voters whose names are already listed on the tables are left unchanged. 

\begin{table}[!t]\centering
\begin{tabular}{lrrrrr}
\hline
 & \multicolumn{4}{c}{Average precision} & \\
Method & White & Black & Hispanic & Asian & Log-loss \\
\hline
\multicolumn{6}{l}{All registrants} \\
\multicolumn{6}{l}{\quad Florida} \\
\qquad Standard BISG & $0.92$ & $0.74$ & $0.81$ & $0.45$ & $0.95$ \\
\qquad $\ell$BISG & $0.90$ & $0.69$ & $0.80$ & $0.26$ & $1.17$ \\
\qquad BISG + $\ell$BISG & $0.92$ & $0.74$ & $0.84$ & $0.45$ & $0.95$ \\
\qquad Standard BIFSG & $0.94$ & $0.80$ & $0.86$ & $0.50$ & $0.86$ \\
\qquad BIFSG + $\ell$BISG & $0.94$ & $0.80$ & $0.86$ & $0.50$ & $0.87$ \\
\multicolumn{6}{l}{\quad North Carolina} \\
\qquad Standard BISG & $0.93$ & $0.75$ & $0.52$ & $0.51$ & $0.90$ \\
\qquad $\ell$BISG & $0.91$ & $0.70$ & $0.42$ & $0.35$ & $1.06$ \\
\qquad BISG + $\ell$BISG & $0.93$ & $0.75$ & $0.58$ & $0.52$ & $0.91$ \\
\qquad Standard BIFSG & $0.95$ & $0.82$ & $0.64$ & $0.54$ & $0.83$ \\
\qquad BIFSG + $\ell$BISG & $0.95$ & $0.82$ & $0.63$ & $0.54$ & $0.83$ \\
\multicolumn{6}{l}{Surname not in the Census table} \\
\multicolumn{6}{l}{\quad Florida} \\
\qquad Standard BISG & $0.78$ & $0.42$ & $0.73$ & $0.08$ & $1.27$ \\
\qquad BISG + $\ell$BISG & $0.83$ & $0.51$ & $0.84$ & $0.11$ & $1.29$ \\
\multicolumn{6}{l}{\quad North Carolina} \\
\qquad Standard BISG & $0.79$ & $0.41$ & $0.44$ & $0.23$ & $1.51$ \\
\qquad BISG + $\ell$BISG & $0.82$ & $0.43$ & $0.65$ & $0.27$ & $1.60$ \\
\multicolumn{6}{l}{First name not in the Census table} \\
\multicolumn{6}{l}{\quad Florida} \\
\qquad Standard BIFSG & $0.64$ & $0.88$ & $0.82$ & $0.53$ & $1.81$ \\
\qquad BIFSG + $\ell$BISG & $0.61$ & $0.87$ & $0.78$ & $0.49$ & $1.93$ \\
\multicolumn{6}{l}{\quad North Carolina} \\
\qquad Standard BIFSG & $0.66$ & $0.90$ & $0.49$ & $0.60$ & $1.78$ \\
\qquad BIFSG + $\ell$BISG & $0.62$ & $0.90$ & $0.40$ & $0.57$ & $1.86$ \\
\hline
\end{tabular}
\caption{Standard BISG and BIFSG combined with $\ell$BISG (LLM lists) for names missing from the Census tables. Log-loss is computed over the four list groups.}
\label{tab:combined}
\end{table}

For first names, BIFSG + $\ell$BISG performs about as well as standard BIFSG across all registrants. However, among registrants whose first name is missing from the Census first name table, BIFSG + $\ell$BISG has a lower average precision than standard BIFSG for every group in both states. For these registrants, the $\ell$BISG first name probability reduces the accuracy of standard BISG. Many of these first names are double names, initials, or spellings held only by a single registrant, and the $\ell$BISG first name probabilities for them are less informative than geography alone.

\subsection{Comparison with direct prompting of an LLM}\label{app-llm-direct}

One alternative approach that might be considered to $\ell$BISG is to simply prompt an LLM directly for the probability of each group given a surname \citep{Dasanaike2026zeroshot}. These probabilities can then be used in place of a surname table, as eBISG does for missing surnames in the race-name frequency tables \citep{DasanaikeImai2026}. We evaluate this alternative in all three settings. For each distinct surname, we prompt an open-weight LLM (gpt-oss-120b) for the probability of each group. We then combine these probabilities with the geographic prior, in the same way as standard BISG, and fit the same BIRDiE models as above. We provide the prompts below and the results in Table~\ref{tab:llm-direct}.

Specifically, we prompt gpt-oss-120b once for every distinct surname in each setting. We run the model locally on the Harvard FAS Research Computing cluster with a temperature of $0$. For the U.S. voter files, we use the following system prompt: 
\begin{quote}\small\emph{``You are an expert in the demography of names in the United States. Answer with a single JSON object and nothing else.''}\end{quote}

We then use the following prompt for each name, proposed by Claude Opus 5.5:
\begin{quote}\small\emph{Consider a registered voter in the United States whose surname is ``\{surname\}''. Estimate the probability that this person self-identifies as each of the following groups: white (non-Hispanic), black (non-Hispanic), hispanic, asian (non-Hispanic Asian or Pacific Islander), other (any other group, including multiracial and American Indian). The five probabilities must sum to 1. Respond only with JSON of the form \{"white": p, "black": p, "hispanic": p, "asian": p, "other": p\}.}\end{quote}

For the 1900 census, as with the respective name list generation, we define the prompt in reference to a person enumerated in the 1900 United States census, with the surname as written by the enumerator, and request the probability that the race of the given person was recorded as Chinese, Japanese, White, Black, American Indian, or any other race. For Lebanon, we define the prompt in reference to a registered voter in Lebanon with the family name written in Arabic, and request the probability of each of the seven sects and the residual class. We normalize the probabilities for each surname so that they sum to one. If the model does not return a probability for every group, we prompt the model again with the temperature at $0.7$.

Let $\Pr_{\textrm{LLM}}(\targ=\targlow\mid S)$ denote the probability for surname $S$. Following standard BISG, we obtain the direct prompting estimate as
\begin{equation*}
\Prhat(\targ=\targlow\mid S,\geo) \propto \frac{\Pr_{\textrm{LLM}}(\targ=\targlow\mid S)\Pr(\targ=\targlow\mid\geo)}{\Pr(\targ=\targlow)},
\end{equation*}
where $\Pr(\targ=\targlow\mid\geo)$ is the same geographic prior as for $\ell$BISG in each setting and $\Pr(\targ=\targlow)$ is the average of this prior over individuals. In addition to gpt-oss-120b, we also prompt a smaller LLM, Qwen3-30B-A3B, for robustness. Qwen3-30B-A3B produces a lower average precision than gpt-oss-120b for most groups in the U.S. settings and a similar average precision in Lebanon.

\begin{table}[!t]\centering
\begin{tabular}{lrrrrrr}
\hline
 & \multicolumn{2}{c}{Average precision} & \multicolumn{2}{c}{ECE} & \multicolumn{2}{c}{RMSE (ppts)} \\
Group & $\ell$BISG & LLM & $\ell$BISG & LLM & $\ell$BISG & LLM \\
\hline
\multicolumn{7}{l}{Florida} \\
\quad White & $0.90$ & $0.92$ & $0.059$ & $0.050$ & $0.8$ & $1.6$ \\
\quad Black & $0.69$ & $0.70$ & $0.020$ & $0.040$ & $4.4$ & $7.5$ \\
\quad Hispanic & $0.80$ & $0.84$ & $0.040$ & $0.025$ & $1.4$ & $1.5$ \\
\quad Asian & $0.26$ & $0.43$ & $0.016$ & $0.011$ & $4.9$ & $3.7$ \\
\multicolumn{7}{l}{North Carolina} \\
\quad White & $0.91$ & $0.92$ & $0.053$ & $0.046$ & $0.7$ & $3.2$ \\
\quad Black & $0.70$ & $0.69$ & $0.033$ & $0.074$ & $4.6$ & $7.0$ \\
\quad Hispanic & $0.42$ & $0.62$ & $0.041$ & $0.038$ & $4.8$ & $3.6$ \\
\quad Asian & $0.35$ & $0.50$ & $0.016$ & $0.014$ & $6.6$ & $5.1$ \\
\multicolumn{7}{l}{1900 census} \\
\quad Chinese & $0.52$ & $0.55$ & $0.001$ & $0.002$ & $8.8$ & $11.7$ \\
\quad Japanese & $0.72$ & $0.84$ & $0.000$ & $0.000$ & $5.0$ & $5.5$ \\
\quad White & $0.98$ & $0.98$ & $0.041$ & $0.082$ & -- & -- \\
\quad Black & $0.57$ & $0.61$ & $0.040$ & $0.085$ & -- & -- \\
\multicolumn{7}{l}{Lebanon} \\
\quad Maronite & $0.91$ & $0.85$ & $0.013$ & $0.032$ & $1.6$ & $3.3$ \\
\quad Roman Orthodox & $0.64$ & $0.58$ & $0.015$ & $0.011$ & $3.5$ & $5.2$ \\
\quad Roman Catholic & $0.55$ & $0.49$ & $0.015$ & $0.012$ & $4.6$ & $8.0$ \\
\quad Armenian Orthodox & $0.73$ & $0.76$ & $0.008$ & $0.008$ & $2.2$ & $2.2$ \\
\quad Sunni & $0.95$ & $0.92$ & $0.014$ & $0.028$ & $0.8$ & $1.7$ \\
\quad Shia & $0.98$ & $0.96$ & $0.007$ & $0.044$ & $1.1$ & $5.6$ \\
\quad Druze & $0.96$ & $0.91$ & $0.004$ & $0.006$ & $3.0$ & $4.8$ \\
\hline
\end{tabular}
\caption{$\ell$BISG (LLM lists) compared with direct prompting of gpt-oss-120b combined with the geographic prior (LLM). ``ECE'' is the expected calibration error over ten probability bins. ``RMSE'' is the root-mean-square error, computed as in Section~\ref{sec-validation}, of the BIRDiE estimate of: (1) the Democratic share by county in the voter files, for Florida and North Carolina; (2) the share in the top quartile of the socioeconomic index by county, for the 1900 census; and (3) the share born in $1985$ or later by district, for Lebanon.}
\label{tab:llm-direct}
\end{table}

Directly prompting an LLM for probabilities yields a higher average precision than $\ell$BISG for most groups in the voter files and the 1900 census, with the largest differences for Asian registrants in both states and Hispanic registrants in North Carolina. Given that our preferred specifications use LLM-generated lists that are then generalized to all voters using an embedding model, the difference in performance between $\ell$BISG and direct model prompting can be thought of as the residual improvement offered by LLMs over the embedding model (and $\ell$BISG algorithm) for unlisted names.

However, $\ell$BISG has better calibration (a lower ECE) than directly prompting an LLM for Black registrants in both states; for Chinese, White, and Black records in the 1900 census; and for four of the seven sects in Lebanon. The BIRDiE estimates based on $\ell$BISG have a lower RMSE than those based on the LLM for White and Black registrants in both states, Hispanic registrants in Florida, both subgroups in the 1900 census, and six of the seven sects in Lebanon. They have a higher RMSE for Asian registrants in both states and Hispanic registrants in North Carolina.

Because an LLM has only latent knowledge about the target population, the probabilities that the LLM generates need not be calibrated for that population. $\ell$BISG, conversely, uses the LLM only to generate the name lists, which shifts the burden from recall and inference of every name to the simpler task of generating short sets of example names that are distinctive and common.

\subsection{Violations of the conditional independence assumption}\label{app-violation}

We measure violations of the eBISG Assumption~\ref{a-ebisg} and the in theory weaker $\ell$BISG Assumption~\ref{a-lbisg} within each racial group. Violations occur if, within racial group, there is any remaining relationship between embedding (or embedding cluster) to geography. First, we define the normalized embedding vector, $\widetilde{E}_i := E_i/\norm{E_i}$.  To measure the dependence between an embedding vector and geography in any given group $r$, we compute the following proportion of variance explained by geography, 
$$
\xi_{\targlow} \ := \ \frac{\sum_{\geolow=1}^{\dimgeo} n_{\geolow\targlow} \norm{\overline{\widetilde{E}}_{\geolow\targlow} - \overline{\widetilde{E}}_{\targlow}}^2}{\sum_{i=1}^N \mathbf{1}\{\targ_i = r\}\norm{\widetilde{E}_i - \overline{\widetilde{E}}_{\targlow}}^2},
$$
where $N$ is the total number of registrants, $n_{\geolow\targlow}$ is the number of registrants in geographic unit $\geolow$ who belong to group $r$, $\overline{\widetilde{E}}_{\geolow\targlow} :=\sum_{i=1}^N \mathbf{1}\{\geo_i=\geolow, \targ_i = \targlow\} \widetilde{E}_i/n_{\geolow\targlow}$ is the average normalized embedding for group $r$ in geography $g$, and $\overline{\widetilde{E}}_\targlow:=\sum_{i=1}^N \mathbf{1}\{\targ_i = \targlow\}\widetilde{E}_i/N_\targlow$ is the overall average for group $r$ with $N_\targlow: = \sum_{i=1}^N \mathbf{1}\{\targ_i = \targlow\}$.  We measure the violation of Assumption~\ref{a-lbisg} by replacing the embeddings with indicators for the clusters $\bin$ in this expression. 

If registrants were allocated to geographic units at random within each racial group (i.e., Assumption~\ref{a-ebisg} holds), the expected value of this statistic is $\E[\xi_\targlow]=(d_{\mathcal{\geo}\targlow}-1)/(N_\targlow -1)$ with $d_{\mathcal{\geo}\targlow}:=\sum_{g=1}^{\dimgeo} \mathbf{1}\{n_{\geolow\targlow} > 0\}$ indicating the number of geographical units where at least one member of group $r$ resides \citep[e.g.,][]{kaziAoual1995refined}.  Since this expected value is larger when both the group size and geographical units are smaller, we use the following adjusted statistic, whose expectation is zero under the random geographic allocation and is bounded above by one,
$$
\xi^{\text{adj}}_{\targlow} \ := \ 1 - \frac{N_r-1}{N_r-d_{\mathcal{\geo}\targlow}} (1-\xi_{\targlow}).
$$

\begin{figure}[!t]\centering
\includegraphics[width=\textwidth]{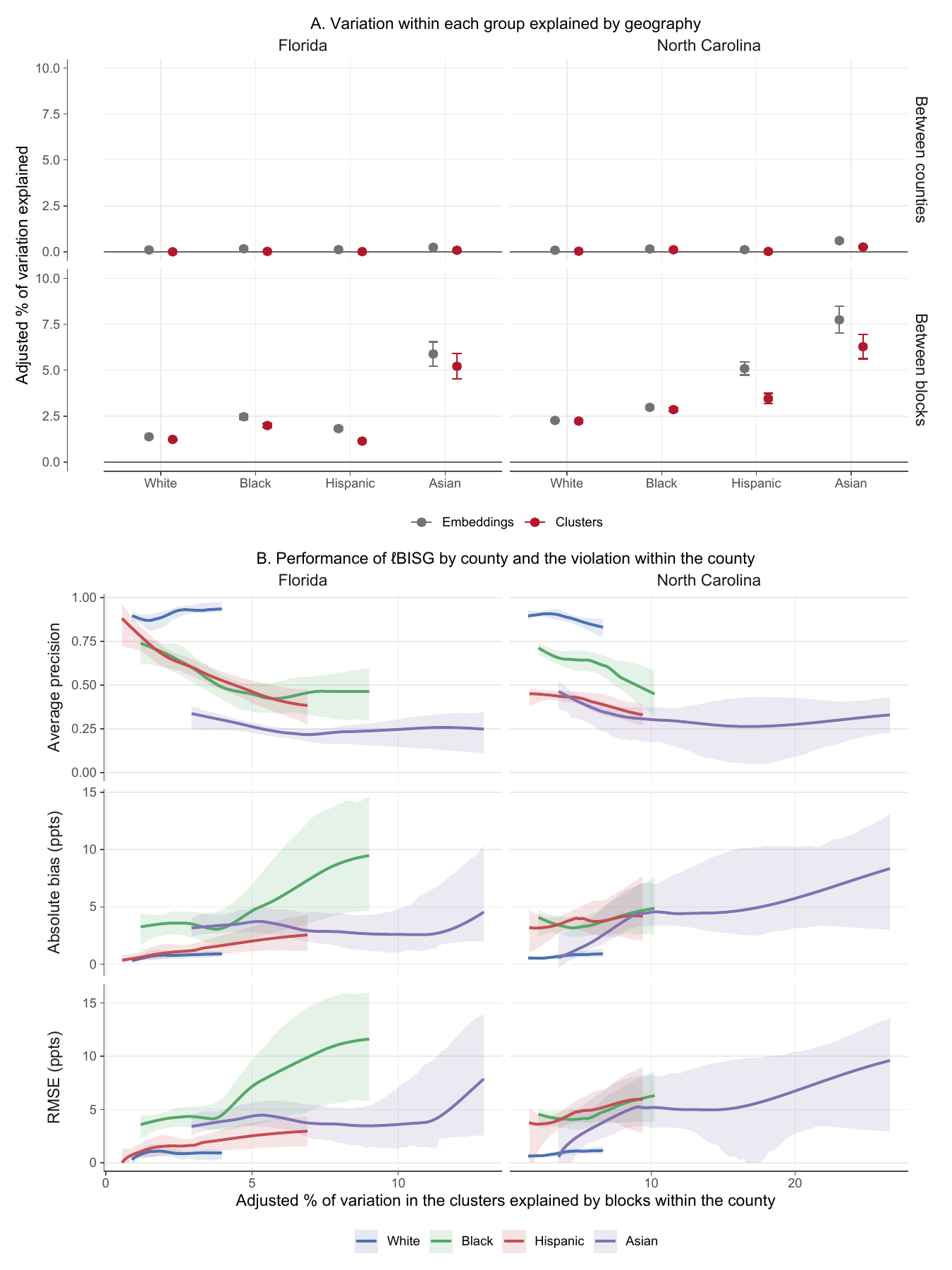}
\caption{Violation of the conditional independence of names and geography given race. Panel A plots the adjusted statistic $\xi^{\text{adj}}_{\targlow}$ (percent) between counties and between census blocks. Panel B plots weighted lowess estimates of the average precision of $\ell$BISG and of the absolute bias and root-mean-square error of the county BIRDiE estimate based on $\ell$BISG against $\xi^{\text{adj}}_{\targlow}$ computed within each county. Error bars in Panel A are $95\%$ confidence intervals based on repeated random half-samples of registrants, and shaded bands in Panel B are pointwise $95\%$ confidence intervals from a bootstrap over counties.}
\label{fig:violation}
\end{figure}

Between counties, the adjusted statistic is positive for every group in both states, suggesting some degree of violation (Figure~\ref{fig:violation}A). However, the magnitude of that violation is small, with differences between counties explaining at most $0.61\%$ of the variation in the embeddings within a group. Between census blocks, the adjusted statistic is larger, reaching $7.7\%$ for the embeddings of Asian registrants in North Carolina (Figure~\ref{fig:violation}A). The statistic between census blocks also depends on the size of the sample, as many blocks contain few registrants of a given group. In random half-samples of registrants, it is between $12\%$ and $19\%$ smaller than in the full sample, whereas the statistic between counties is nearly unchanged. Figure~\ref{fig:violation}A also shows that for every group and at both levels of geography, the statistic for the cluster variable $B$ is smaller than the statistic for the embeddings. This is consistent with Assumption~\ref{a-lbisg} being weaker than Assumption~\ref{a-ebisg}.

Next, we examine how the assumption violation is associated with the performance of $\ell$BISG and its downstream analysis.  Specifically, we compute $\xi^{\text{adj}}_{\targlow}$ within each county, using census blocks as the geographic units. We then compute the average precision of $\ell$BISG in each county and estimate the Democratic share in each county with BIRDiE based on $\ell$BISG. We fit a lowess regression of the average precision, the absolute bias of the BIRDiE estimate, and its squared error on the statistic, weighting each county by $n_{\geolow\targlow}$. We plot the fitted values for the average precision, and the absolute bias, and the square root of the fitted values for the squared error (Figure~\ref{fig:violation}B). We exclude counties where fewer than $100$ registrants belong to group $\targlow$.

We find that, for most groups, the average precision of $\ell$BISG decreases and the bias and root-mean-square error of the BIRDiE estimate increase with the extent of the assumption violation. The increases in bias and root-mean-square error are largest for Black registrants in Florida and Asian registrants in North Carolina, although the confidence intervals for these groups are wide. The association between the error and the violation remains positive after controlling for the number of the group's registrants in each county.

\end{document}